\documentclass[11pt]{article}
\usepackage{authblk}
\usepackage[english]{babel}
\usepackage[top=1in, bottom=1in, left=1in, right=1in]{geometry}
\usepackage[svgnames]{xcolor}
\usepackage{amsthm}
\usepackage{amsmath}
\usepackage{mathtools}
\usepackage{amsfonts}
\usepackage{graphicx}
\usepackage{float}
\usepackage{caption}
\usepackage{subcaption}
\usepackage{makecell}
\usepackage{verbatim}
\usepackage{multirow}
\usepackage{dsfont}
\usepackage{algorithm}
\usepackage{algpseudocode}
\usepackage{url}
\usepackage{natbib}
\usepackage{bbm}
\usepackage{mathrsfs}
\numberwithin{equation}{section}
\usepackage[colorlinks,pdfdisplaydoctitle,linkcolor=blue,citecolor=blue]{hyperref}
\usepackage{tikz}
\usetikzlibrary{decorations.pathreplacing}
\usetikzlibrary{fadings}
\usetikzlibrary{positioning, arrows.meta, fit, backgrounds}

\usepackage{booktabs}       

\newcommand\cN{\mathcal N}
\newcommand\cD{\mathcal D}
\newcommand\cW{\mathcal W}
\newcommand\cV{\mathcal V}
\newcommand\cP{\mathcal P}
\newcommand\cB{\mathcal B}
\newcommand\cL{\mathcal L}
\newcommand\cM{\mathcal M}
\newcommand\cO{\mathcal O}
\newcommand\bbZ{\mathbb Z}
\newcommand\bbR{\mathbb R}
\newcommand\bbP{\mathbb P}
\newcommand\bbD{\mathbb D}
\newcommand\bbN{\mathbb N}
\newcommand\bbE{\mathbb E}
\newcommand\dsone{\mathds 1}
\newcommand\sX{\mathsf X}
\newcommand\sS{\mathsf S}
\newcommand\sY{\mathsf Y}
\newcommand\sW{\mathsf W}
\newcommand\sV{\mathsf V}
\newcommand\norm[1]{\left\lVert#1\right\rVert}

\newcommand{\mysim}{\mathrel{\overset{\mathclap{\text{\tiny i.i.d.}}}{\sim}}}

\newtheorem{theorem}{Theorem}[section]
\newtheorem{lem}{Lemma}[section]
\newtheorem{rem}{Remark}[section]
\newtheorem{prop}{Proposition}[section]

\newtheorem{ass}{Assumption}[section]
\newtheorem{example}{Example}[section]

\title{ \vspace{-0cm} Theoretical and Practical Advances in Hierarchical \\ Bayesian Deep Learning Methodology}
\author{Neil K. Chada,\thanks{Department of Mathematics, City University of Hong Kong, Hong Kong SAR. \texttt{neilchada123@gmail.com}} \quad Simon Cotter,\thanks{School of Mathematics, University of Manchester, UK. \texttt{simon.cotter@manchester.ac.uk}} \quad Kody J. H. Law,\thanks{School of Mathematics, University of Manchester, UK. \texttt{kody.law@manchester.ac.uk}} \quad Xinzhu Liang\thanks{School of Mathematics, University of Manchester, UK. \texttt{xinzhu.liang@postgrad.manchester.ac.uk}}}

\begin{document}
\maketitle

\begin{abstract}
{
Reliable uncertainty quantification remains an important challenge in AI. 
Hierarchical Bayesian methods provide a principled framework for this
by representing unknown quantities through posterior distributions. 
Joint inference of 
neural-network parameters and prior hyperparameters
is challenging for conventional Markov chain Monte Carlo (MCMC) 
methods due to strong dependencies.
In this work, we introduce the multilevel particle marginal Metropolis-Hastings (MLPMMH) method, which combines
pseudo-marginal inference for the hyperparameters with coupled sequential Monte Carlo (SMC) across hierarchy of network-width truncations for conditional parameter inference. The hierarchy is induced by a trace-class neural network (TNN) prior with width-dependent decay, motivated empirically by the compressible structure observed in trained neural-network parameters. Under standard multilevel coupling assumptions, we establish that MLPMMH achieves mean-square error (MSE) $\mathcal{O}(\varepsilon^2)$ at expected computational cost $\mathcal{O}(\varepsilon^{-2})$, and verify the predicted rate on a Bayesian regression problem. To address larger-scale applications, we further introduce S-MLPMMH, a practical approximation that replaces the exact sampling components with a PCA-based hierarchy and scalable Bayesian Monte Carlo. Experiments on MNIST demonstrate competitive predictive performance and uncertainty quantification while exploiting the parallel structure of the approximation.
Software has been made publicly available at \url{https://github.com/submission-code-anon/mlpmmh-anon}.
}
\end{abstract}

\section{Introduction}
The focus of this work is on the design of Bayesian statistical algorithms for efficiently performing hierarchical 
Bayesian deep learning \citep{theo2,papamarkou2024position,bdl2026handbook}.
In this context, one is interested in estimating the posterior distribution of the joint parameter $\theta \times \mu \in \Theta \times \sS$, where we have access to data $\cD \in \bbD$. Then the joint posterior, via Bayes' Theorem, is expressed as
$$
\pi(\theta,\mu|\cD)
= \frac{\pi(\cD|\theta,\mu)\bar{\pi}(\theta|\mu)\pi_0(\mu)}{Z(\cD)} 
\propto \pi(\cD,\theta|\mu) \pi_0(\mu) \, ,
$$
where $Z(\cD)$ is a normalizing constant, $\pi(\cD|\theta,\mu)$ is the data-likelihood, $\bar{\pi}(\theta)$ is the prior distribution for parameter and $\pi_0(\mu)$ is the prior distribution for hyperparameter. Prior distributions, in general, allow us to represent an initial belief of the posterior, providing some level of flexibility. However, this hierarchical structure can hinder generic Markov chain Monte Carlo (MCMC) inference, even for powerful Bayesian methodologies such as HMC \citep{neal2011mcmc,betancourt2015hamiltonian}. Strong correlations between hyperparameters and parameters often make estimating the joint posterior difficult, leading to pathological behaviour. 
The pseudo-marginal approach
is a general class of MCMC algorithms that target the marginal posterior ($\int_{\Theta}\pi(\cD,d\theta|\mu)\pi_0(\mu)$) by replacing the intractable 
marginal likelihood ($\int_{\Theta}\pi(\cD,d\theta|\mu)$) of the hyperparameters with an unbiased stochastic estimator \citep{andrieu2009pseudo}. 
An important sub-class is known as particle MCMC
\citep{andrieu2010particle,golightley}, which uses particle filters 
to build the unbiased distribution estimators inside a pseudo-marginal scheme.
We will leverage the particle marginal Metropolis-Hasting (PMMH) method here.

Although Sequential Monte Carlo (SMC) methods are well suited for sampling in high dimensional settings, they still suffer from the curse of dimensionality, which limits their efficiency. To mitigate this, Multilevel Monte Carlo (MLMC) \citep{heinrich2001multilevel, MBG08, giles2015multilevel} has been introduced to 
achieve canonical computational complexity, $1/$MSE. 
Recent work has applied 
MLMC in the context of statistical inference
\cite{beskos2017multilevel, beskos2018multilevel, chada2021unbiased, cotter2022randomized} and 
machine learning 
\citep{chadaTNN, FS21, MLMClearn, LMM21, shi}.
In this paper, we 
leverage MLMC techniques in the context of 
trace-class neural networks (TNN) for prior modeling \citep{sell2020dimension}
following a similar approach as in \cite{chadaTNN}.
The seminal work by \cite{jasra2018bayesian} applied the particle marginal Metropolis–Hastings scheme
to parameter estimation for particle filter algorithms for parameter estimation in discretely observed continuous time state space models.
The present work applies a similar principle to the case of joint parameter and hyper-parameter estimation in 
static BNN models with TNN prior architectures by simulating from an approximate coupling of the 
parameter conditional distributions at coarse and fine scales and correcting the increments a posteriori 
with importance weights.

\subsection{Contributions}

The contributions of this work are summarized as follows:
\begin{itemize}
    \item \textbf{Algorithmic contribution.}
          We introduce a \emph{multilevel particle marginal Metropolis-Hastings} (MLPMMH) scheme 
          for BNN which builds increment estimators from approximate couplings of
          parameter conditional targets at successive resolutions simulated with SMC samplers,
          within the PMMH method for joint estimation of hyperparameters. 
          These increment estimators are combined to build an MLMC estimator.
    \item \textbf{Theoretical guarantees.}
          Under standard multilevel coupling assumptions and the trace-class prior, 
          we prove a convergence theorem (Theorem \ref{thm:mlpmmh}): 
          \[
            \text{MSE}= \mathcal{O}\!\bigl(\text{COST}^{-1}\bigr).
          \]
    \item {\bf Practical Approximation algorithm.} 
    A practically applicable approximation is introduced, where each increment is estimated from 
          an ensemble of SBMC estimators \cite{liang2025scalable}, 
          each conditional on a hyper-parameter simulated from a Laplace approximation of the 
          marginal. This is referred to as S-MLPMMH.
    \item \textbf{Empirical validation.} 
          (i) A toy Bayesian regression problem shows the expected MSE–COST convergence rate and 
          illustrates increment variance and mean decay across levels.  
          (ii) Strong performance of S-MLPMMH is demonstrated on the MNIST image classification problem. 
          (iii) Power-law decay with rank of weights and biases from MAP training of Bayesian neural networks 
          empirically supports the use of the TNN prior.
          (iv) Several MLSMC variants 
          for the inner parameter inference are compared in the appendix, including their computational cost-accuracy behavior.
\end{itemize}

\subsection{Outline}
The outline of this document is as follows. In Section \ref{sec:back}, we introduce the model formulation, trace-class neural network prior, and associated multilevel structure. We then move to Section \ref{sec:hyper_method}, introducing our MLPMMH methodology for joint parameter and hyperparameter estimation and the practical S-MLPMMH approximation. Related theory, including the MSE--computational-cost convergence result for MLPMMH, is provided in Section \ref{sec:theory}. In Section \ref{sec:num}, 
we numerically test MLPMMH to verify our theory, 
motivate and validate the TNN width-decay structure, and apply the S-MLPMMH approximation on MNIST classification. Finally, we summarize our findings in Section \ref{sec:conc}. Additional background on SMC and several existing and novel MLSMC methods for parameter estimation is provided in Appendix \ref{sec:param_method}, together with supporting theory, numerical experiments, and proofs.

\begin{figure*}[h!]
\centering
\includegraphics[width=0.8\columnwidth]{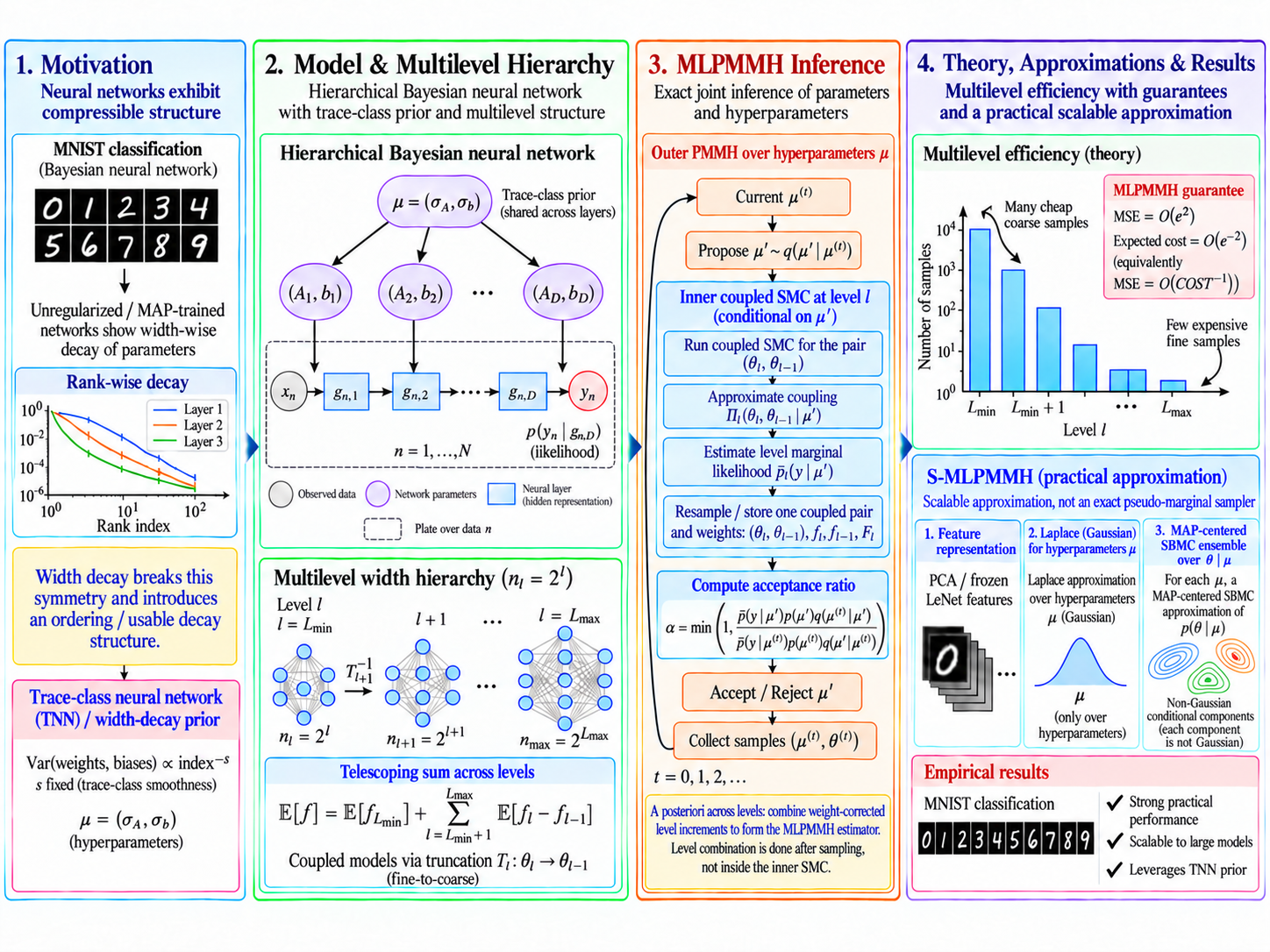}
\caption{
Diagram of the MLPMMH program. A technical diagram of the hierarchical TNN model itself is 
provided in Figure \ref{fig:tnn} in the Appendix.}
\label{fig:main_tnn}
\end{figure*}

\section{Model Formulation}
\label{sec:back}

Suppose we have data $\cD = ((x_1,y_1),...,(x_N,y_N))$, $N \in \bbN$, where $x_i \in \sX \subseteq \bbR^{n}$ are inputs associated to 
output labels $y_i \in \sY \subseteq \bbR^{m}$. We want to infer the joint state of parameter $\theta \in \Theta \subseteq \bbR^{d_{\theta}}$ and the hyperparameter $\mu \in \sS \subseteq \bbR^{d_{\mu}}$ based on the data $\cD$ within a deep neural network (DNN) model $g: \sX \times \Theta \times \sS \rightarrow \bbR^{m}$ \footnote{Although the hyperparameter $\mu$ does not enter the forward map
explicitly once $\theta$ is fixed, we retain the notation
$g(x,\theta,\mu)$ to emphasize that $\theta$ is distributed according to a prior parameterized by $\mu$. Thus, conditionally on $\theta$, $g(x,\theta,\mu)=g(x,\theta)$, while the distribution of the network output depends on $\mu$ through the prior on $\theta$.} where
$
g(x,\theta,\mu) = (g_1(x,\theta,\mu),...,g_m(x,\theta,\mu)).
$

\subsection{Trace-class neural network priors}

These TNN priors have also been introduced and well-established in the \cite{chadaTNN}, and here we followed a similar procedure: for $d=1,...,D$, the associated weights and bias ($\theta$) are defined as

\begin{equation}
\begin{aligned}
\label{eqn:prior_tnn}
A_{ij,d} &\sim \cN(0, \sigma_A^2(ij)^{-s}), 
\\
\quad A_{ij,1} &\sim \cN(0, \sigma_A^2i^{-s}), 
\\
\quad A_{ij,D} &\sim \cN(0, \sigma_A^2 j^{-s}), \\
b_{i,d} &\sim \cN(0,\sigma_b^2 i^{-s}).
\end{aligned}
\end{equation}

The model, presented in Figure \ref{fig:tnn}, is constructed based on the pushforward of a countably infinite Gaussian random variable that has a specified {\em weight decay} in the width direction (which can similarly be applied to depth). This instils a clear inductive bias—analogous to choosing a kernel to specify the smoothness of a Gaussian process (GP) while remaining fundamentally different.
These priors {\em converge strongly}, while GPNN typically converge only weakly, so that it allows us to apply the MLMC methodology. Since the hyperparameter $s$ is a tricky one, whose value determines the strong convergence rate in the MLMC methodology, 
we will set it to be a constant value here. We then can let the hyperparameter be $\mu = (\sigma_A, \sigma_b)$, and a suitable prior could be the 2-dimensional log-normal distribution.

\subsection{Multilevel Structure}

Given the level $l \in \bbN$, for simplicity, we assume $D$ is fixed and let the NN width is the same for all layer other than the input and the output, i.e.$n_d = n_l=2^l$ for $d \in \{1,...,D-1\}$, $n_0=n$ and $n_D=m$. The corresponding vector of parameters is denoted by $\theta_l:=((A_1^l,b_1^l),...,(A_D^l,b_D^l))$ which is in the parameter space $\Theta_l=\{\bbR^{n_l \times n} \times \bbR^{n_l}\} \otimes \{\bbR^{n_l \times n_l} \times \bbR^{n_l}\}^{\otimes D-2} \otimes \{\bbR^{m \times n_l} \times \bbR^{m}\}$ and the prior is $\bar{\pi}_l(\theta_l)$. 

Here, we introduce the canonical truncation map $T_l:\Theta_l\to\Theta_{l-1}$. Given a fine-level parameter $\theta_l \in \Theta_l$, we write $\theta_{l-1}=T_l(\theta_l)$ for the corresponding coarse-level parameter obtained by setting: $A^l_{ij,d}=A^{l-1}_{ij,d}$ and $b^l_{i,d}=b^{l-1}_{i,d}$ for $i,j\in{1,\ldots,n_{l-1}}$ and $d \in \{2, \dots, D\}$; a similar assignment is to be adopted for $(A^l_1, b^l_1)$ and $(A^l_D, b^l_D)$.

Then, we denote the DNN output by $g_l(x,\theta_l,\mu)$. Using the same output map $\Psi$ at every level, we define the quantity of interest by $\varphi_l(x,\theta_l,\mu) :=
(\Psi\circ g_l)(x,\theta_l,\mu) = \Psi(g_l(x,\theta_l,\mu))$. The  posterior distribution of the parameter $\theta_l$ conditioned on $\mu$ is
\begin{equation*}
    \pi_l(\theta_l|\mu) \propto \cL_{l}(\theta_l|\mu) \bar{\pi}_l(\theta_l)
\end{equation*}

$\pi_{l}(\theta|\mu) = \frac{f_l(\theta|\mu)}{Z_l}$, where $f_l(\theta_l|\mu) = \cL_l(\theta_l|\mu)\bar{\pi}_l(\theta_l)$ and $Z_l = \int_{\Theta_l} f_{l} d\theta$. Then we construct the approximate coupling, $\Pi_l$, as follows
\begin{equation*}
\begin{aligned}
     \Pi_l(\theta_l,\theta_{l-1}) 
     &= \frac{F_l(\theta_l,\theta_{l-1}|\mu)}{F_l(1|\mu)},\\
     F_l(\theta_l,\theta_{l-1}|\mu) &= {\bf L}_l(\theta_l,\theta_{l-1}|\mu) \bar{\Pi}_l(\theta_l,\theta_{l-1}),
\end{aligned}
\end{equation*}
where we choose the joint likelihood, ${\bf L}_l(\theta_l,\theta_{l-1}|\mu)$, as the maximum coupling
\begin{equation*}
    {\bf L}_l(\theta_l,\theta_{l-1}|\mu) = \max\{\cL_l(\theta_l|\mu),\cL_{l-1}(\theta_{l-1}|\mu) \},
\end{equation*}
and the coupling of the prior distribution $\bar{\Pi}_l$ is 
\begin{equation*}
    \bar{\Pi}_l(d\theta_l,d\theta_{l-1})
    =
    \bar{\pi}_l(d\theta_l)\,
    \delta_{T_l(\theta_l)}(d\theta_{l-1}).
\end{equation*}

\section{Hierarchical Bayesian Estimation of Neural Networks \ with Trace-Class Priors}
\label{sec:hyper_method}

In this section, we provide and present our primary contribution, which is our methodology for hierarchical Bayesian estimation of the hyperparameter. This will be based on particle Markov chain Monte Carlo, specifically the particle marginal Metropolis-Hastings method \citep{chada2021unbiased,jasra2018bayesian}. We review this method \citep{jasra2018bayesian} and design it in such a way related to our framework and trace-class neural network priors. Further details of related algorithms are provided in the supplementary material.

\subsection{MLPMMH}

Now, we embed the MLSMC sampler in Algorithm \ref{alg:mlsmc} inside the PMMH method to estimate the marginal likelihood, which gives the multilevel pseudo-marginal Metropolis-Hastings (MLPMMH) method. In the context of Bayesian neural networks, we infer the prior standard-deviation hyperparameters governing the weights and biases of the TNNs model with log-normal priors. Due to the multilevel setup, the joint posterior distribution at each level $l$ is given by
\begin{equation*}
    \pi_{l}(\theta_{l},\mu) \propto \underbrace{\cL_{l}(\theta_{l}|\mu) \bar{\pi}_{l}(\theta_{l})}_{f_{l}(\theta|\mu)}\pi_{0}(\mu) =: f_l(\theta_l,\mu).
\end{equation*}


One can run a MLSMC sampler in Algorithm \ref{alg:mlsmc} to approximate the marginal distribution
\begin{equation*}
\begin{aligned}
    \Pi_{l}(\theta_l, \theta_{l-1}|\mu) &\propto {\bf L}_{l}(\theta_{l},\theta_{l-1}|\mu) \bar{\Pi}_{l}(\theta_l,\theta_{l-1})\\ & =: F_{l}(\theta_l,\theta_{l-1}|\mu).
\end{aligned}
\end{equation*}

Then the joint coupling posterior distribution at each level can be approximate using Algorithm \ref{alg:mlpmmh_l} as
\begin{equation*}
    \Pi_{l}^{N,M_l} (d(\theta,\mu)) := \frac{1}{M_{l}} \sum_{k=1}^{M_{l}} \delta_{(\theta_{k,l},\mu_{k,l})} (d(\theta,\mu)),
\end{equation*}

Give a level parameter $L=\{L_{\text{min}},L_{\text{max}} \}$, the MLPMMH estimator is given by
\begin{equation}
\begin{gathered}
    \hat{\varphi}_{L}^{\text{MLPMMH}} 
    := \sum_{l=L_{\text{min}}}^{L_{\text{max}}} \bigg( 
     \frac{\frac{1}{M_l}\sum_{k=1}^{M_l} \varphi_l(\theta_{k,l},\mu_{k,l}) \frac{f_{l}(\theta_{k,l}|\mu_{k,l})}{F_{l}(\theta_{k,l},\theta_{k,l-1}|\mu_{k,l})}}{\frac{1}{M_l}\sum_{k=1}^{M_l}\frac{f_{l}(\theta_{k,l}|\mu_{k,l})}{F_{l}(\theta_{k,l},\theta_{k,l-1}|\mu_{k,l})}}
    \\- \frac{\frac{1}{M_l}\sum_{k=1}^{M_l} \varphi_{l-1}(\theta_{k,l-1},\mu_{k,l}) \frac{f_{l-1}(\theta_{k,l-1}|\mu_{k,l})}{F_{l}(\theta_{k,l},\theta_{k,l-1}|\mu_{k,l})}}{\frac{1}{M_l}\sum_{k=1}^{M_l}\frac{f_{l-1}(\theta_{k,l-1}|\mu_{k,l})}{F_{l}(\theta_{k,l},\theta_{k,l-1}|\mu_{k,l})}}\bigg),
\label{eqn:mlpmmh}
\end{gathered}
\end{equation}
where $\varphi_{L_{\text{min}}-1}:=0$.

\paragraph{Parallization.} The MLPMMH estimator also admits several levels of parallelization. First, the level-wise PMMH estimators can be run independently for different ($l$). Second, within each PMMH iteration, the SMC estimator used to estimate the level-dependent normalizing constant can exploit particle-level parallelism. Thus the main available parallelism is across levels and across SMC particles. The PMMH chain itself remains sequential in the iteration index (k), although independent chains or independent repetitions may be run in parallel.

\begin{algorithm*}[t]
\begin{algorithmic}[1]
\caption{MLPMMH for level $l$ \citep{jasra2018bayesian}}

\State Simulate $\mu_0 \sim \pi_0$ and run  Algorithm~\ref{alg:mlsmc} for level $l$ with $N$ and $\mu_0$, it outputs $(F_{0,l},\hat{\theta}_{0,l})$. 
\State Define $p_{0,l}^{N}(\mu_{0}) = \frac{1}{N} \sum_{i=1}^{N} F_{0,l}(\hat{\theta}_{0,l}^i,\hat{\theta}_{0,l-1}^{i}|\mu_{0})$ and $w_{0,l}^i = \frac{F_{0,l}(\hat{\theta}^i_{0,l},\hat{\theta}_{0,l-1}^{i}|\mu_{0})}{\sum_{i=1}^{N}F_{0,l}(\hat{\theta}^i_{0,l},\hat{\theta}_{0,l-1}^{i}|\mu_{0})}$.
\State Resample: Select $I_{0,l}^1 \sim \{ w_{0,l}^1,...,w_{0,l}^{N} \}$, and let $\theta_{0,l}=\hat{\theta}_{0,l}^{I_{0,l}^1}$.
\For{$k=1,...,M_l$}
\State Propose $\hat{\mu}_k \sim Q(\cdot,\mu_{k-1})$.
Run Algorithm~\ref{alg:mlsmc} for level $l$ with $N$ and $\hat{\mu}_k$, and output$\left\{\hat{\theta}_{k,l}^i, \hat{\theta}_{k,l-1}^i, f_{k,l}^i(\hat{\theta}_{k,l}^i), f_{k,l-1}^i(\hat{\theta}_{k,l-1}^i), F_{k,l}^i\right\}_{i=1}^{N}$. 
\State Define $p_{k,l}^{N}(\hat{\mu}_{k}) = \frac{1}{N} \sum_{i=1}^{N} F_{k,l}(\hat{\theta}_{k,l}^i,\hat{\theta}_{k,l-1}^{i}|\hat{\mu}_{k})$ and $$w_{k,l}^i = \frac{F_{k,l}(\hat{\theta}^i_{k,l},\hat{\theta}_{k,l-1}^{i}|\hat{\mu}_{k})}{\sum_{i=1}^{N}F_{k,l}(\hat{\theta}^i_{k,l},\hat{\theta}_{k,l-1}^{i}|\hat{\mu}_{k})}.$$
\State Resample: Select $I_{k,l}^1 \sim \{ w_{k,l}^1,...,w_{k,l}^{N} \}$, and let 
$\theta_{k,l}=\hat{\theta}_{k,l}^{I_{k,l}^1}$,
$\theta_{k,l-1} = \hat{\theta}_{k,l-1}^{I_{k,l}^1}$.
\State Simulate $u \sim U[(0,1)]$ and compute 
\begin{equation*}
    \alpha = \min \Bigg\{ \frac{\pi_0(\hat{\mu}_{k})Q(\mu_{k-1}|\hat{\mu}_{k})p_{k,l}^{N}(\hat{\mu}_{k})}{\pi_0(\mu_{k-1})Q(\hat{\mu}_{k}|\mu_{k-1})p_{k,l}^{N}(\mu_{k-1})}, \ 1 \Bigg\}.
\end{equation*}
\State If $u \leq \alpha$, then $(\mu_{k},\theta_{k,l},p_{k,l}^{N}(\mu_{k})) = (\hat{\mu}_{k},\hat{\theta}_{k,l},p_{k,l}^{N}(\hat{\mu}_{k}))$; otherwise, \\  $(\mu_{k},\theta_{k,l},p_{k,l}^{N}(\mu_{k})) = (\mu_{k-1},\theta_{k-1,l},p_{k,l}^{N}(\mu_{k-1}))$.
\EndFor
\State Define $\mu_{k,l} = \mu_{k}$. Store $\theta_{k,l}$, $\theta_{k,l-1}$, $f_{l-1}(\theta_{k, l-1}|\mu_{k,l}) = \bar{\pi}_{l-1}(\theta_{k,l-1}) \cL_{l-1}(\theta_{k,l-1}|\mu_{k,l})$, $f_{l}(\theta_{k, l}|\mu_{k,l}) = \bar{\pi}_{l}(\theta_{k,l}) \cL_{l}(\theta_{k,l}|\mu_{k,l})$ and $F_l(\theta_{k,l},\theta_{k,l-1}|\mu_{k,l})$.
\label{alg:mlpmmh_l}
\end{algorithmic}
\end{algorithm*}

\subsection{S-MLPMMH}
\label{sec:smlpmmh}

Inspired by the SBMC framework in \cite{liang2025scalable},
we introduce a practical SBMC-based approximation of MLPMMH
designed to reduce its serial computational cost, called S-MLPMMH. The original MLPMMH method uses an outer serial PMMH chain to sample the hyperparameter, and each outer proposal requires running an inner MLSMC estimator. To obtain a scalable practical version, we replace the serial outer PMMH chain with a Laplace-based SBMC approximation to the hyperparameter posterior, centered at an empirical-Bayes MAP estimate. We also replace the inner MLSMC sampler by a MAP-centered SBMC sampler. 

Specifically, at each level $l$, we consider the transformed
hyperparameter $\eta=\log\mu$ and construct a local Laplace
approximation centered at a level-dependent MAP estimate
$\eta_{l,\sf MAP}$. The outer hyperparameter samples are independently
drawn as
\begin{equation*}
\eta_{k,l} \sim \cN \left(\eta_{l,\sf MAP},V_{l,\eta}\right),
\qquad \mu_{k,l}=\exp(\eta_{k,l}),
\end{equation*}
for $k=1,\ldots,M_l$, where $V_{l,\eta}$ is obtained from the local curvature of the
empirical-Bayes objective.

For each sampled hyperparameter $\mu_{k,l}$, we then construct an
inner MAP-centered SBMC sampler. In particular, we compute the
corresponding MAP estimate $\theta_{l,\sf MAP}$ and construct a new
inner prior centered at this estimate,
\begin{equation*}
\hat{\bar{\pi}}_{l} = \cN \left(\theta_{l,\sf MAP},s_{\theta}\mu_{k,l}^{2}\sf I
\right).
\end{equation*}
The MAP estimate is recomputed for each hyperparameter sample
$\mu_{k,l}$.

This practical version preserves the multilevel correction structure
of MLPMMH while reducing both the serial outer cost and the cost of
inner sampling. Moreover, due to the parallelizability of the SBMC
framework, the outer hyperparameter samples and their corresponding
inner SBMC computations can be processed independently and therefore
in parallel. Note, however, that S-MLPMMH does not constitute an exact
pseudo-marginal sampler for the original joint posterior, and the
theoretical results for MLPMMH in Section~\ref{sec:theory} do not
directly apply to this practical approximation. Further details of
the empirical-Bayes, Laplace, and MAP-centered SBMC construction are
provided in Appendix~\ref{app:smlpmmh}.

\section{Theoretical results}
\label{sec:theory}

The short section is devoted to our main theoretical contribution, which is an MSE bound with our MLPMMH result. It is requires several important additional results, which are provided in Appendix \ref{app:theory_background} and proven in Appendix \ref{app:proofs}.

\subsection{Main results}

\begin{theorem}\label{thm:mlpmmh}
    Let us assume Assumptions \ref{ass:G1_hyper}, \ref{ass:K_hyper} and \ref{ass:rate_hyper} with $\beta >\gamma$. Then, for suitable $\varphi: \Theta \times \sS \rightarrow \bbR$, there exists a $C < +\infty$ such that for any $\varepsilon >0$ one can choose $(L=\{L_{\text{min}},L_{\text{max}}\},\{M_l\}_{l=L_{\text{min}}}^{L_{\text{max}}})$ such that
    \begin{equation*}
    \bbE[(\hat{\varphi}_{L}^{\mathrm{MLPMMH}} - \pi(\varphi))^2] \leq C \varepsilon^2,
    \end{equation*}
    and expected COST$(\hat{\varphi}_{L}^{\mathrm{MLPMMH}}) \leq C \varepsilon^{-2}$, i.e. the canonical rate. The estimator is defined in equation \eqref{eqn:mlpmmh}.
\end{theorem}

\begin{proof}
The proof sees Appendix \ref{app:main_thm_proof}.
\end{proof}

Whereas, the computational cost of PMMH to achieve the MSE$=O(\varepsilon^2)$ is $\cO(\varepsilon^{-2-\gamma/\alpha})$. It implies that PMMH will often suffer from the curse of dimensionality since the ratio $\frac{\gamma}{\alpha}$ will increase as the dimension becomes high.

\section{Numerical Experiments}
\label{sec:num}

In this section, we provide numerical experiments through our MLPMMH scheme, using TNN neural networks. We will not only verify our complexity rate (through Theorem \ref{thm:mlpmmh}) but also justify the choice of the TNN prior. That will be done by checking its weight decay in the width limit. Experiments will be conducted on a Bayesian regression and MNIST classification problem while providing a comparison with other methods.

\subsection{Numerical justification for MLPMMH}

\subsubsection{Synthetic regression dataset}

Generate $N_{tr}=100$ data points, $y_i$, followed \ref{eqn:bayes_reg} at level $l=7$ using $x_i = 2 + \sqrt{0.5} \cdot \nu_i \in \bbR^{10}$, where $\nu_i \sim \cN(0, I_{10})$ for $i = 1, \dots, 100$. 
Evaluation is then applied on one fixed test point $x = [2, 2, \dots, 2] \in \bbR^{10}$.


In the section, we numerically verified the rate decay in Assumption \ref{ass:rate_hyper} and convergence rate in Theorem \ref{thm:mlpmmh} on the synthetic regression example using the Bayesian TNN model. The number of neural layers is $D=3$, the input neural network width is then set to $n = 10$, and the output neural network width is set to $m=1$. The TNN prior uses a rate decay of $s = 2$, the MCMC process in MLPMMH takes $50$ burn-in samples and adaptive step size, and the HMC kernel 
is applied with $10$ mutation steps, one leapfrog step and adaptive step size in the MLSMC. The hyperparameter prior is the LogNormal distribution with the location parameter as $-0.125$ and scale parameter as $0.5$.

A regularity test, similar to MLSMC, is conducted to estimate the decay of bias and variance for normalized estimator differences between levels under the posterior distribution. The results are shown in Figure~\ref{fig:mlpmmh_regularity}, supporting the decay rate assumed in Assumption~\ref{ass:rate_hyper} for the convergence analysis. Again, starting from level $l = 4$, the decay of biases and weights aligns with the expected rate, where the bias decay (weak convergence) is bounded by $\alpha = \beta/2$ and the strong convergence rate is $\beta = 2s - 1 = 3$ with $s = 2$.

Taking $L_{\text{min}} = 4$ as the starting level, the convergence test for the MLPMMH is presented in Figure~\ref{fig:mlpmmh_complexity}, i.e. MSE$=\cO(\text{COST}^{-1})$. Given a target accuracy level $\varepsilon$, we can follow a similar procedure using the mean and variance decay to compute the maximum level $L_{\max}$ and the number of samples per level $M_l$ in the MLPMMH as in Section~\ref{sec:mlsmc_practical_set}, and set the $N=100$ samples in the MLSMC for estimating the marginal distribution. The reference solution is computed at accuracy $\varepsilon = 8\times10^{-4}$ using MLPMMH.

\begin{figure}[h!]
  \centering
  \includegraphics[width=.5\linewidth]{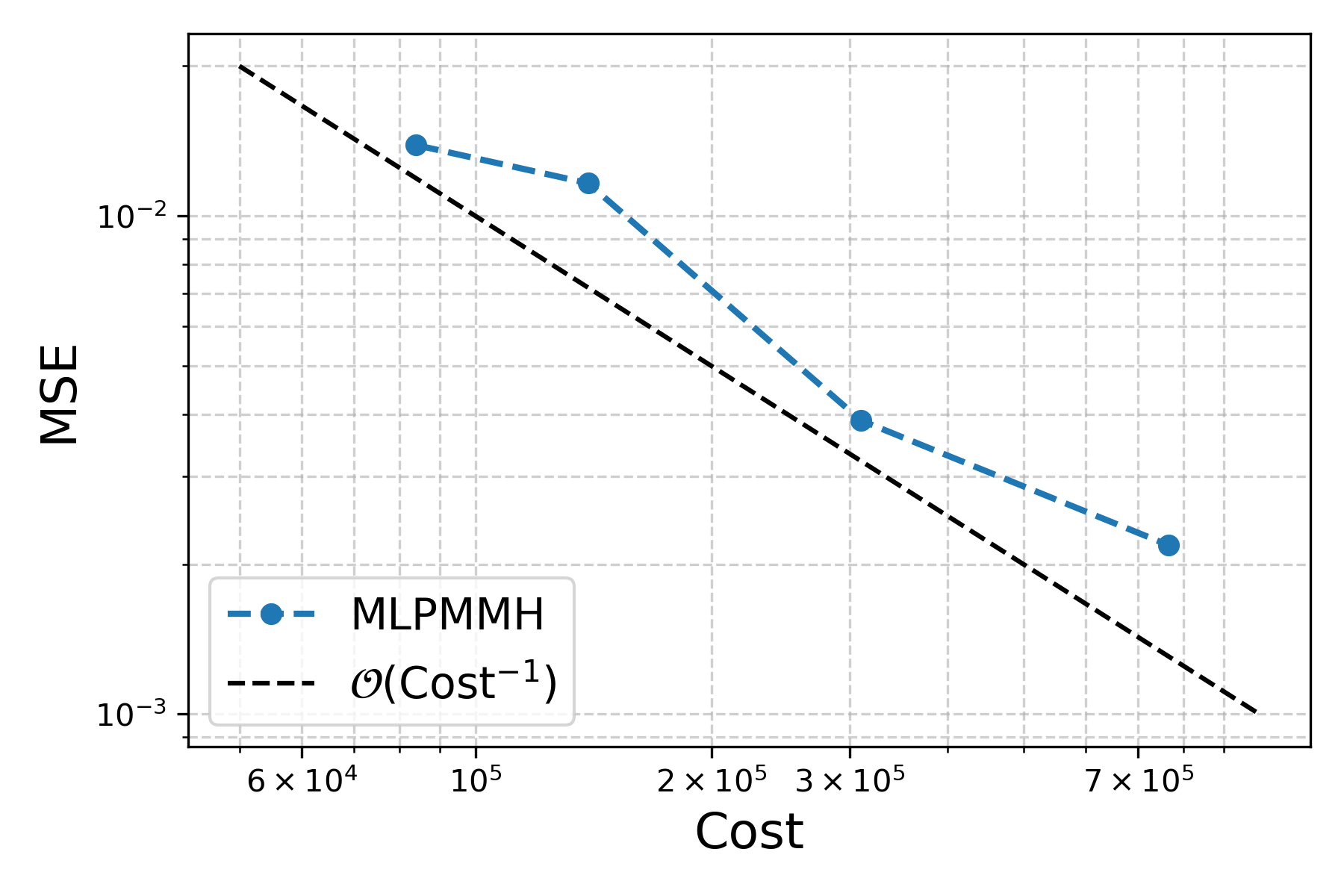}
  \caption{Complexity test for MLPMMH on Bayesian regression model: 
  {We use $10$ realizations to compute the MSE for each experiment 
  and plot the rate of convergence w.r.t computation costs. 
  The regression coefficient is $-0.9$.}}
\label{fig:mlpmmh_complexity}
\end{figure}

\subsection{Empirical verification of TNN prior structure}


To numerically validate the TNN prior, we conduct a series of controlled MAP training experiments on the MNIST dataset. 

A key issue in interpreting width-dependent decay is the exchangeability of hidden units.
For NoReg and weight-decay, the ordering of hidden units is arbitrary: the same network function can often be represented after permuting the hidden-unit indices.
Therefore, a trained network may contain a low-effective-width or compressible structure even if this structure is not aligned with the original width ordering.
For this reason we use a rank-wise diagnostic, which sorts learned weight magnitudes and tests whether compressibility is present at all. 
This informs the width-decay that can be non-intrusively imposed when we break the exchangeability.

We sort the absolute weights in each layer in descending order and plot $|w|$ against rank on log--log axes, as shown in Figure~\ref{fig:tnn_weight_decay_summary}.
Since this sorting removes the original width indices, the figure should be interpreted as a compressibility diagnostic rather than as direct evidence of index-wise TNN structure.
The NoReg and weight-decay models already show clear decay in sorted weight magnitudes, suggesting that trained neural networks naturally concentrate much of their mass in a relatively small number of effective directions.
The width-decay model preserves a comparable leading-rank structure while inducing a stronger decay in the tail, especially in the hidden layers.
This suggests that the TNN prior is not imposing an unrelated structure from scratch; rather, it makes an already present compressibility structure explicit and ordered.

\begin{figure*}[h!]
  \centering
  \includegraphics[width=\linewidth]{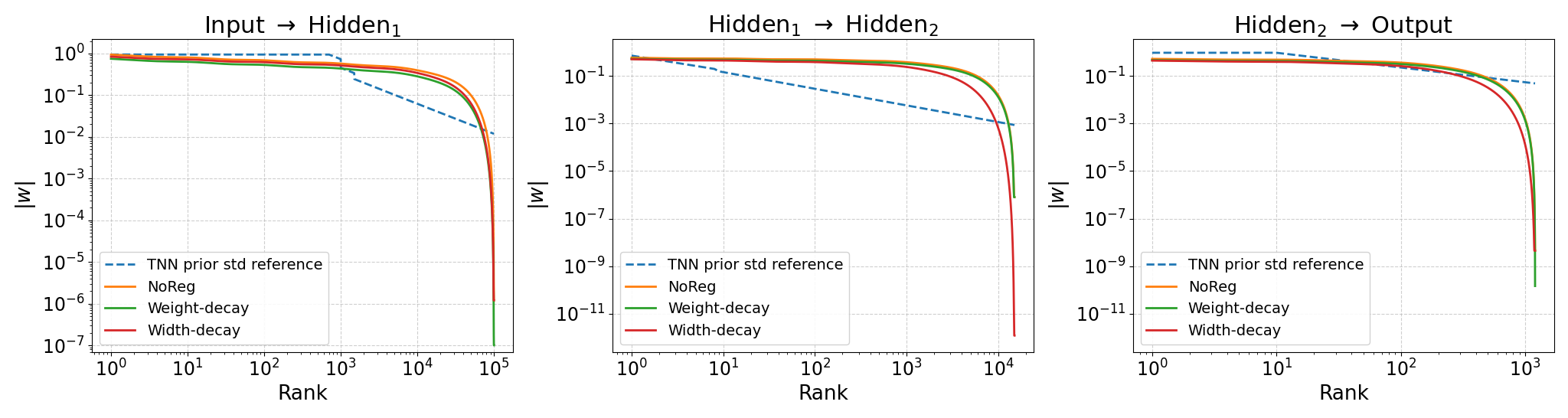}  
\caption{Rank-wise decay of learned weight magnitudes for $s=2.0$ and $\sigma_A=\sigma_b=1.0$.
  {For each layer, the absolute values of the learned weights are sorted in descending order and plotted against rank on log-log axes.
  }}
\label{fig:tnn_weight_decay_summary}
\end{figure*}

Finally, Table~\ref{tab:tnn_weight_decay} shows that width-decay achieves performance comparable to NoReg and weight-decay for moderate choices of the decay parameter $s$.
In particular, for $s \leq 4$, width-decay maintains similar accuracy, NLL, Brier score, and ECE relative to the two baselines.
This indicates that introducing the TNN ordering does not substantially compromise predictive performance within a reasonable range of decay strengths.
However, very large values such as $s=10.0$ lead to numerical instability and poor performance, suggesting that overly aggressive width decay should be avoided.
Since the appropriate value of $s$ is problem-dependent, we discuss its practical choice further in the MLPMMH experiments below.



\begin{table*}[h!]
  \centering
  \scriptsize
  \begin{tabular}{l|c|c|ccccc}
    \toprule
    Method & NoReg & Weight-decay & \multicolumn{5}{c}{Width-decay} \\
    \cmidrule(lr){4-8}
    $s$ &  &  & 0.5 & 1.0 & 2.0 & 4.0 & 10.0 \\ 
    \midrule
    Acc (\%) & 97.92 ± 0.23 & 97.73 ± 0.06 & 97.82 ± 0.14 & 97.82 ± 0.28 & 97.65 ± 0.15 & 97.07 ± 0.13 & 89.28 ± 1.05 \\
    NLL      & 0.106 ± 0.015 & 0.110 ± 0.005 & 0.108 ± 0.005 & 0.112 ± 0.012 & 0.105 ± 0.011 & 0.105 ± 0.005 & 0.374 ± 0.042 \\
    Brier    & 0.035 ± 0.004 & 0.038 ± 0.001 & 0.036 ± 0.002 & 0.036 ± 0.004 & 0.038 ± 0.002 & 0.045 ± 0.002 & 0.161 ± 0.016 \\
    ECE      & 0.014 ± 0.002 & 0.015 ± 0.000 & 0.014 ± 0.001 & 0.015 ± 0.002 & 0.015 ± 0.001 & 0.008 ± 0.001 & 0.027 ± 0.004 \\
    \bottomrule
  \end{tabular}
    \caption{Metrics for three methods evaluated on the MNIST dataset: 
  {MLE, standard weight decay regularization, and TNN width-decay regularization are used for training on the MNIST training dataset with PyTorch default initialization and testing on the MNIST test dataset. 
  TNN width-decay regularization with fixed $\sigma=1$ and varying decay parameter $s$. Results are averaged over 5 replicates, with error bar given by one standard deviation.}}
  \label{tab:tnn_weight_decay}
\end{table*}


\begin{table*}[h!]
  \centering
  \footnotesize
  \resizebox{\textwidth}{!}{%
    \begin{tabular}{l|c|c|c|c|c|c|c|c}
      \toprule
      $L_{\max}$ & $M_l$ & $T_{\mathrm{total}}$ (s)
      & $T_{\mathrm{parallel}}$ (s) & FLOPs (TFLOP) & Acc. (\%) & NLL & Brier & ECE \\
      \midrule
      $2$ & $[7]$ & 635.08 & 9.07 & 0.32 & 83.51±0.05 & 0.470±0.000 & 0.239±0.000 & 0.023±0.001 \\
      \midrule
      $3$ & $[13, 5]$ & 1633.67 & 10.26 & 1.29 & 91.71±0.07 & 0.568±0.005 & 0.126±0.000 & 0.020±0.000 \\
      \midrule
      $4$ & $[30, 13, 5]$ & 4240.72 & 10.47 & 4.53 & 95.34±0.10 & 0.376±0.002 & 0.079±0.001 & 0.024±0.001 \\
      \midrule
      $5$ & $[77, 33, 14, 6]$ & 11259.92 & 9.41 & 14.65 & 97.06±0.05 & 0.241±0.001 & 0.057±0.000 & 0.035±0.001 \\
      \midrule
      $6$ & $[200, 87, 37, 16, 7]$ & 29716.96 & 9.43 & 44.32 & 97.33±0.04 & 0.204±0.003 & 0.051±0.000 & 0.035±0.000 \\
      \midrule
      $7$ & $[527, 229, 99, 43, 18, 8]$ & 79113.25 & 9.46 & 129.12 & 97.54±0.02 & 0.204±0.003 & 0.049±0.000 & 0.036±0.000 \\
      \bottomrule
    \end{tabular}%
  }
  \caption{Metrics for S-MLPMMH on the MNIST dataset: 
  {The starting level is $K=2$. The reported metrics include accuracy, NLL, Brier score, ECE, total wall-clock runtime $T_{\mathrm{total}}$, pseudo-parallel runtime $T_{\mathrm{parallel}}$, and estimated FLOPs for different choices of $L_{\max}$. The QOI is the predictive probability. Results are averaged over 3 replicates, with error bar gives by one standard deviation.}}
  \label{tab:practical_mlpmmh}
\end{table*}

Overall, these experiments suggest that the TNN width-decay structure is empirically compatible with effective neural network training.
The rank-wise diagnostic shows that trained networks already exhibit substantial compressibility, while the index-wise diagnostic shows that width-decay turns this compressibility into an ordered, truncation-friendly structure.
This provides empirical support for using the TNN prior as a structural prior in the MLPMMH construction.

\subsection{Practical application of MLPMMH} 



The original MLPMMH algorithm targets the joint posterior of the hyperparameters and model parameters by combining an outer pseudo-marginal Markov chain with an inner multilevel SMC estimator.
Although this provides a principled approach to inference under strong parameter-hyperparameter dependence, it is computationally expensive: each outer proposal requires a complete inner MLSMC run, while the outer Markov chain remains sequential. These features make the method difficult to scale to large datasets and high-dimensional feature
spaces. We therefore introduce 
practical modifications and evaluate the resulting approximation on the full MNIST dataset, using
$60000$ training and $10000$ test images.

\subsubsection{Feature matrix from the pre-trained LeNet NN}

We first pre-train the full LeNet architecture on the first \(N_{\sf tr}\) MNIST images using plain supervised learning. The network comprises two \(\text{Conv} \to \text{ReLU} \to \text{MaxPool}\) blocks followed by three fully–connected layers, the latter regularized with \(50\%\) dropout. All parameters are optimized for \(100\) epochs with cross–entropy loss and the Adam optimiser (initial learning rate \(10^{-3}\), halved every \(10\) epochs). Once training is complete, the parameters in the LeNet NN are fixed, i.e. a frozen LeNet NN. The output of the second MaxPool layer is then flattened into a \(256\)-dimensional feature vector, and these vectors are collected for all \(N_{\sf tr}\) training images. The resulting \(N_{\sf tr} \times 256\) matrix of features is 
the feature matrix used for the subsequent Bayesian analysis.

\subsubsection{PCA hierarchy and multilevel coupling.}

To obtain the TNN-style (isotropic) weight–decay and the nested sequence of models required for the multilevel scheme, we apply Principal Component Analysis (PCA) to the \(N_{\sf tr}\times 256\) training-feature matrix and express both training and test features in this basis. At level \(l\), the input layer width is set to $n_0 = 2^{l}$, i.e. we keep only the first \(2^{l}\) principal components (PCs). Followed the coupling construction introduced above, level \(l-1\) uses the first \(2^{l-1}\) PCs from the level \(l\), i.e. the fine and coarse likelihoods share their leading coordinates, yielding the strong correlation that drives multilevel variance reduction. The prior, $\bar{\pi}_{l}$, for the parameter in single Bayesian logistic-regression layer is the Gaussian prior now, i.e. $N(0,\sigma_A^2 I_{m\times n_0})$  and $N(0,\sigma_b^2 I_{m \times 1})$ where $\mu=(\sigma_A,\sigma_b)=(\sigma,\sigma)$. The hyperparameter prior, $\pi_0$, is the Log-Normal distribution. 
In our implementation, we set $\bbE[\mu] = \frac{1}{\sqrt{2^8/2}}$
\footnote{Here $2^8=256$ is the full feature dimension. The
choice $\bbE[\mu]=\sqrt{2/256}$ uses the He/Kaiming
standard-deviation scale as a convenient dimension-dependent
reference value.}
, $\operatorname{Var}(\mu) = 0.25\,\mathbb{E}[\mu]^2$. So, this gives the location parameter as $-2.54$ and scale parameter as $0.47$.

\subsubsection{Results}

We then apply the practical S-MLPMMH approximation introduced in Section~\ref{sec:smlpmmh}, using the PCA hierarchy described above to construct the multilevel estimators.

Table \ref{tab:practical_mlpmmh} reports the results for the practical
MLPMMH method as the computational budget increases. In these experiments, $L_{\max}$ is increased gradually, together with the corresponding number of outer SBMC samples $M_l$ at each level $l$.
The runtime is reported in two forms: the total wall-clock runtime
$T_{\mathrm{total}}$ and the pseudo-parallel runtime $T_{\mathrm{parallel}}$.
In our implementation, the independent outer SBMC samples could be run in parallel, and the inner SBMC samples within each outer SBMC sample could also be parallelized. Therefore, we estimate the parallel runtime as $T_{\mathrm{parallel}}=\max_l \frac{T_l}{N M_l}$, where $T_l$ denotes the total runtime at level $l$, $N$ is the number of inner SBMC samples, and $M_l$ is the number of outer SBMC samples at level $l$. We also report the estimated floating-point cost in TFLOPs.



With $d_l=2^l$ for $l=2,\ldots,7$, the level-dependent Bayesian
models are relatively small, and the idealized parallel runtime is therefore nearly constant across levels. The estimated FLOPs, however, increase with $L_{\max}$, as expected from the theoretical sample allocation. Consequently, this experiment is not well suited to demonstrating level-dependent wall-clock savings, although it illustrates the predictive behaviour of the practical multilevel approximation as the maximum level increases.
The ECE remains low for $L_{\max}\leq4$ but increases from $L_{\max}=5$ onward. To investigate this behaviour, we also compute the signed calibration gap, defined as
$\operatorname{Gap}
=
\frac{1}{N_{\mathrm{test}}}
\sum_{i=1}^{N_{\mathrm{test}}}
\left(\operatorname{acc}_i-\operatorname{conf}_i\right).
$
The gap is consistently positive across the considered values of
$L_{\max}$, indicating that the model is, on average, under-confident.
This suggests conservative uncertainty estimates,
which is a positive sign for reliability and robustness.




\section{Conclusion}
\label{sec:conc}

This paper develops and applies a
Bayesian method for
joint parameter and hyperparameter estimation for neural networks.
The methodology builds on multilevel Monte Carlo particle Metropolis methods.
and a particular neural network prior.
The prior structure is based on a trace-class neural network, which imposes 
decay in the width direction and enjoys strong convergence.
We provide a theoretical complexity result relating MSE to computational cost, 
along with numerical verification.
Numerical experiments were conducted on both a Bayesian linear regression problem and on the MNIST data set for classification. 
Our results demonstrate the effectiveness of our proposed method, 
both in terms the model prior and the algorithm used for posterior inference. 
 For the former, we compared our network architecture to that of standard training, and training where there is no regularization.
 For moderate values of the decay parameter $s$, width-decay
regularization achieves predictive accuracy and NLL comparable to the unregularized and standard weight-decay baselines, whereas excessively strong decay degrades predictive performance. Separately, the synthetic regression experiments support the predicted multilevel regularity and the canonical MSE-cost rate for the proposed MLPMMH method.
We also
presented an SBMC-based multilevel approximation using frozen LeNet
features and a PCA hierarchy. Since this practical variant replaces
both the outer PMMH chain and the inner MLSMC sampler, it should be
viewed as a scalable approximation rather than a direct 
verification of the exact MLPMMH.


\section*{Acknowledgements}
NKC is supported by an EPSRC-UKRI AI for Net Zero Grant: ``{Enabling CO2 Capture And Storage Projects Using AI}", (grant EP/Y006143/1). NKC is also supported by a City University of Hong Kong Start-up Grant, project number 7200809, and a RGC-ECS Grant, project number 21300826. KJHL and XL gratefully acknowledge the support of IBM and EPSRC in the form of an Industrial Case Doctoral Studentship Award.

\bibliography{refs}
\bibliographystyle{plainnat}

\appendix
\section{Model Formulation}
\label{app:model_background}

In this section, we present and review our model setup for the paper. This includes a review of our Bayesian posterior distribution of interest, while discussing this in the context of deep neural networks. Related to this, we also discuss our prior, which is based on trace-class neural networks. 

\subsection{Model Formulation}

We now define the deep neural network model as follows. Let $D \in \bbN$, $(n_0,...,n_D) \in \bbN^{D+1}$ with $n_0=n$ for the input layer and $n_D=m$ with output layer. Let weights be $A_{d} \in \bbR^{n_d \times n_{d-1}}$ and biases be $b_d \in \bbR^{n_d}$ for $d \in \{1,...,D\}$, we denote $\theta := ((A_1,b_1),...,(A_D,b_D)) \in \Theta = \otimes_{d=1}^{D} \{\bbR^{n_d \times n_{d-1}} \times \bbR^{n_d} \}$ and $\mu$ is the hyperparameter related to $\theta$. The layer is defined by
\begin{equation*}
\begin{aligned}
    g_1(x,\theta,\mu) & := A_1 x + b_1, \\
    g_d(x,\theta,\mu) & := A_d \sigma_{n_{d-1}}(g_{d-1}(x)) + b_{d}, \ \ d \in \{2,...,D-1\}, \\
    g(x,\theta,\mu) & := A_{D}\sigma_{n_{D-1}}(g_{D-1}(x)) + b_{D},
\end{aligned}
\end{equation*}
where $\sigma_{k}(z) := (\nu(z_1),...,\nu(z_k))^{T}$ with ReLU activation $\nu(z) = \max \{0,z\}$.

In order to infer the parameter $\theta$ and hyperparameter $\mu$, we establish the Bayesian model of the DNNs, which is called Bayesian DNNs. Given a proper prior $\bar{\pi}$ on $\Theta$ and a prior $\pi_0$ on $\mu$, the joint posterior distribution of the Bayesian DNNs \citep{chadaTNN,neal} can be computed up to a constant of proportionality
\begin{equation}\label{eqn:posterior_single}
    \pi(\theta) \propto \cL(\theta,\mu) \bar{\pi}(\theta|\mu) \pi_0(\mu),
\end{equation}
where $\cL(\theta,\mu)$ is the joint likelihood on $\Theta \times \sS$.

Define the quantity of interest through a map \(\Psi\) applied to the output. That is,
$$
\varphi(x,\theta,\mu)
:=
(\Psi\circ g)(x,\theta,\mu)
=
\Psi(g(x,\theta,\mu)).
$$
We then aim to compute the posterior expectation of the quantity of interest,
\begin{equation}\label{eqn:exp_qoi}
    \pi(\varphi)
    :=
    \bbE_{\pi}[\varphi].
\end{equation}

\subsection{Trace-class neural network priors}

The neural network priors on $\theta$, the weights and bias, we consider are the trace class neural networks, first proposed by \citet{sell2020dimension}. These priors were introduced to mimic Gaussian priors $\overline{\pi}_0 \sim \cN(0,C)$, for function-space inverse problems, where a common way to simulate Gaussian random fields is to use the Karhunen-Lo\`{e}ve expansion
\begin{equation}
\label{eq:kle}
\theta = \sum_{j \in \mathbb{Z}^+} \sqrt{\lambda_j} \xi_j \Phi_j, \qquad \xi_j \sim \cN(0,I),
\end{equation}

where $(\lambda_j,\Phi_j)_{j \in \bbZ^+}$ are the associated eigenbasis of the covariance function $C$, and $\{\xi_j\}_{j \in \bbZ^+}$ is Gaussian white noise. For a detail description of the derivation and its application to stochastic numerical problems, we refer the reader to \cite{LPS14}. An issue that can arise with using \eqref{eq:kle} is that using such priors do not scale well with high-dimensional problems. This acts as the initial motivation for trace-class neural network (TNN) priors, where the work of \citet{sell2020dimension} provided a justification for such priors. 

The hyperparameters $\sigma_A,\sigma_b,s$ significantly influence the inductive bias of the prior, and should be tuned for optimal performance. TNN priors are {\em not} based on the Gaussian processes, which have been popular in a variety of Bayesian-related tasks. Recent literature has understood GP Neural Networks (GPNNs) in the context of width and depth limits of neural networks \citep{cho2009kernel, agrawal2020wide,lee2018deep}.


\begin{figure}[h!]
\centering
\includegraphics[scale=0.25]{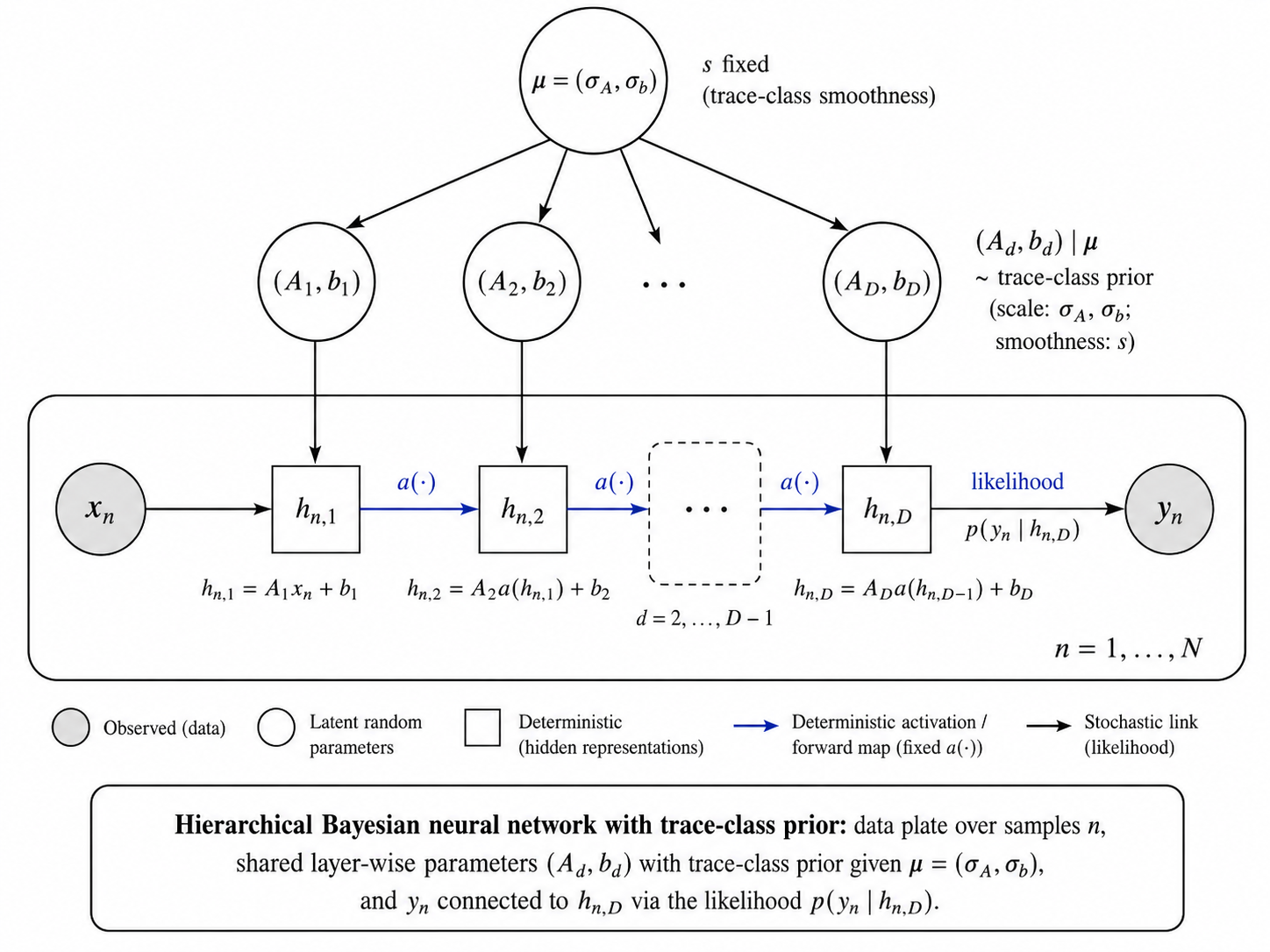}
\centering
\caption{Trace class nerual network architecture.}
\label{fig:tnn}
\end{figure}

\section{State-of-the-Art Estimation Methods for Bayesian Neural Networks}
\label{sec:param_method}

In the context of the Bayesian neural network with fixed hyperparameters $\mu$ and $s$, the parameter estimation methods introduced here can be utilized for computing the marginal distribution of $\theta$ conditional on the hyperparameter $\mu$.

\subsection{Sequential Monte Carlo estimator}

First, we look at the sequential Monte Carlo (SMC) estimator \citep{beskos2018multilevel,del2006sequential}. Consider the conditioned posterior distribution, $\pi(\theta|\mu)$, in \eqref{eqn:posterior_single} with the fixed hyperparameters, we define the target distribution as 
$\pi(\theta|\mu) = f(\theta|\mu)/Z$, where $Z = \int_{\Theta} f(\theta|\mu) d\theta$
and $f(\theta|\mu):= \cL(\theta|\mu) \bar{\pi}(\theta)$. Given a quantity of interest $\varphi: \Theta \rightarrow \bbR$, for simplicity, we define 
\begin{equation*}
    f(\varphi|\mu) := \int_{\Theta} \varphi(\theta|\mu) f(\theta|\mu) d\theta = f(1) \pi(\varphi|\mu),
\end{equation*}
where $f(1) = \int_{\Theta} f(\theta|\mu) d\theta = Z$. So that the conditional target distribution is $\pi(\varphi) = \frac{f(\varphi)}{f(1)}$.

Define $h_1,...,h_{J-1}$ by $h_j = f_{j+1}/f_j$, where $f_1 = \bar{\pi}, f_{J} = f=\cL\bar{\pi}$, and define $\pi_j = f_j/Z_j$ where $Z_j = \int_{\Theta} f_j d\theta$. For $j=2,...,J$, we let $f_j$ define in a annealing scheme:
\begin{equation*}
    f_j = f_1^{1-\lambda_j}f_J^{\lambda_j}=\cL(\theta)^{\lambda_j}\bar{\pi}(\theta),
\end{equation*}
where $0=\lambda_1<\lambda_2<\cdots<\lambda_J=1$ and $\lambda_j$ will be chosen adaptively according to the effective sample size (ESS), see Example~\ref{ex:adaptive_temp}.

\begin{example}[Adaptive tempering]
\label{ex:adaptive_temp}
    In order to keep a sufficient diversity of sample population, we let the effective sample size be at least $ESS_{\min} = N/2$ at each tempering $\lambda_{j-1}$ and use it to compute the next tempering $\lambda_j$. For $j$th tempering, we have weight samples $\{ w^{k}_{j-1}, \theta^{k}_{j-1} \}_{k=1}^{N}$, then the ESS is computed by
    \begin{equation*}
        ESS = \frac{1}{\sum_{k=1}^{N} (w^{k}_{j-1})^2},
    \end{equation*}
    where $w^{k}_{j-1} = \cL(\theta^{k}_{j-1})^{\lambda_{j} - \lambda_{j-1}}/\sum_{k=1}^N \cL(\theta_{j-1}^k)^{\lambda_j-\lambda_{j-1}}$. Let $\Delta \lambda = \lambda_j - \lambda_{j-1}$, the effective sample size can be presented as a function of $\Delta \lambda$, ESS$(\Delta \lambda)$. Using suitable root finding method, one can find $\Delta \lambda^{*}$ such that $ESS(\Delta \lambda^{*})=ESS_{\min}$, then set the next tempering $\lambda_j = \lambda_{j-1}+\Delta \lambda^{*}$.
\end{example}

\begin{algorithm}[H]
   \caption{SMC sampler}
   \label{alg:smc}
\begin{algorithmic}[1]
   \State Generate $\theta^i_1 \sim \bar{\pi}$ for $i=1,...,N$ and $Z^N_1=1$.
   \For{$j=2$ {\bfseries to} $J$}
    \State Store $Z^N_j = Z^N_{j-1} \frac{1}{N} \sum_{k=1}^N \cL(\theta_{j-1}^k)^{\lambda_j-\lambda_{j-1}}$.
    \For{$i=1$ {\bfseries to} $N$} 
    \State Define $w^i_j = \cL(\theta_{j-1}^i)^{\lambda_j-\lambda_{j-1}}/\sum_{k=1}^N \cL(\theta_{j-1}^k)^{\lambda_j-\lambda_{j-1}}$.
    \State Resample: Select $I_j^i \sim \{ w_j^1,...,w_j^N\}$ and let $\hat{\theta}_j^i = \theta_{j-1}^{I_j^i}$.
    \State Mutate: Draw $\theta_j^i \sim \cM_j(\hat{\theta}_j^i, \cdot )$.
    \EndFor
    \EndFor
\end{algorithmic}
\end{algorithm}

Now let $\cM_{j}$ for $j=2,\dots,J$ be Markov transition kernels such that $(\pi_{j} \cM_{j})(d\theta) = \pi_{j}(d\theta)$, analogous to any suitable MCMC kernel \citep{geyer1992practical,robert1999monte,cotter2013mcmc,law2014proposals}. This operation must sufficiently decorrelate the samples, and as such, we typically define the MCMC kernels ${\cM_{j}}$ by a suitable number of steps of some basic MCMC kernel, where it may or may not be 1.

One can observe that
\begin{equation*}
    \pi_{j}(h_j) = \frac{1}{Z_j} \int \frac{f_{j+1}(\theta)}{f_j(\theta)} f_j(\theta) d\theta = \frac{Z_{j+1}}{Z_j}.
\end{equation*}
Since we let $\pi_J=\pi$ be the target distribution and $Z_1 = 1$, the normalization constant of $\pi$ is given by
\begin{equation*}
    Z  = Z_1\prod_{j=1}^{J-1} \pi_{j}(h_j)
       = \prod_{j=1}^{J-1} \pi_{j}(h_j).
\end{equation*}

From Algorithm~\ref{alg:smc}, we can define the following estimators:
\begin{equation}
\begin{aligned}
\label{eqn:est_smc_unnorm}
    \pi_j^N(h_j) & :=  \frac{1}{N} \sum_{i=1}^{N} h_j(\theta_{j}^{i}), \ j = 1,...,J-1, \\
    Z^N & := \prod_{j=1}^{J-1} \pi_j^N(h_j), \\
    f^N(\varphi) & := \prod_{j=1}^{J-1} \pi_j^N(h_j) \pi^N(\varphi) = Z^N \pi^N(\varphi).
\end{aligned}
\end{equation}
Hence, the sequential Monte Carlo estimator for estimating the expectation of the quantity of interest is given by
\begin{equation}
\label{eqn:est_smc}
    \hat{\varphi}_{\text{SMC}} := \pi^N(\varphi) =\frac{f^N(\varphi)}{f^N(1)}.
\end{equation}

\subsection{Novel Multilevel sequential Monte Carlo estimators}

Integrating a multilevel scheme into the single SMC estimator further decreases the variance in the estimator, yielding particularly more gains in high-dimensional settings. In this section, we restate three multilevel Monte Carlo methods: MLSMC \citep{law2022multi}, randomized MLSMC \citep{cotter2022randomized} ratio estimators, and the original serial MLSMC \citep{chadaTNN}. 


\subsubsection{Multilevel Structure}

There are two fundamentally different ways used here to couple across levels: the “top‐down” construction and the “bottom‐up”. 
\begin{itemize}

    \item "top-down": one starts with a fine‐level sample at level $l$ and then constructs its coarse analogue at level $l-1$, typically by truncating parameters in the data, inputs, or network width (as in (randomized) MLSMC estimators). 

    \item "bottom-up": one first samples a coarse-level parameter $\theta_{l-1}\in\Theta_{l-1}$, and then extends or propagates it to a fine-level parameter $\theta_l\in\Theta_l$. This direction is used later for the serial MLSMC estimator. 
\end{itemize}
In what follows, unless otherwise stated, the approximate coupling $\bar{\Pi}_l$ refers to the top-down truncation coupling.

\paragraph{Paralleization.} This distinction also affects the available parallelization. In the "top-down" construction in (randomized) MLSMC, the estimators at different levels can be run independently, and within each level the SMC particles can be propagated and evaluated in parallel between resampling steps. 
By contrast, the "bottom-up" construction in serial MLSMC has a stronger dependence across levels, since the fine-level sample is obtained by propagating information from the previous coarse level. Hence this construction is generally less parallel across levels, although particle-wise operations and independent replications can still be parallelized.

\paragraph{Approximate Coupling.} At each level $l$, we intend to simulate a pair of samples $(\theta_l^i,\theta_{l-1}^i) $ follows $(\theta_l^i,\theta_{l-1}^i) \sim \Pi_l(\theta_l,\theta_{l-1}|\mu)$, where $\Pi_l(\theta_l,\theta_{l-1}|\mu)$ is the exact coupling distribution such that
\begin{equation*}
     \int_{\Theta_l} \Pi_{l}(\theta_l,\theta_{l-1}|\mu)d\theta_l = \pi_{l-1}(\theta_{l-1}|\mu), \quad
      \int_{\Theta_{l-1}} \Pi_{l}(\theta_l,\theta_{l-1}|\mu)d\theta_{l-1} = \pi_{l}(\theta_{l}|\mu).
\end{equation*}
However, it is not a trivial approach to compute the exact coupling; instead, we introduce the approximate coupling technique. Following the same procedure in rewriting the posterior distribution at the single level in SMC, we have

There are two fundamentally different ways used here to couple across levels. The first runs “top‐down”: one starts with a fine‐level sample at level $l$ and then constructs its coarse analogue at level $l-1$, typically by truncating parameters in the data, inputs, or network width (as in (randomized) MLSMC estimators). The second runs “bottom‐up”: one begins with a coarse‐level sample at level $l-1$ and serially propagates it to level $l$ (as in the serial MLSMC estimator).

Since the data $x$ is the same and the hyperparameters are fixed, we omit the $x$ and $\mu$ in the notations below.

\subsubsection{Multilevel sequential Monte Carlo estimators}

Given a starting level $L_{\text{min}} \geq 1$ and a maximum level $L_{\text{max}} \geq L_{\text{min}} $, the approximation of the expectation \eqref{eqn:exp_qoi} is given by
\begin{equation*}
\begin{aligned}
    \pi(\varphi) \approx \sum_{l=L_{\text{min}}}^{L_{\text{max}}} \pi_l(\Delta \varphi) & := \sum_{l=L_{\text{min}}}^{L_{\text{max}}} \pi_l(\varphi_l(\theta_l)-\varphi_{l-1}(\theta_{l-1})) \approx \sum_{l=L_{\text{min}}}^{L_{\text{max}}} \frac{1}{N_l} \sum_{i=1}^{N_l} \Big\{ \varphi_l(\theta_l^i)-\varphi_{l-1}(\theta_{l-1}^i) \Big\}, 
\end{aligned}
\end{equation*}
where $\theta^i_{l}$ are i.i.d. samples at each level $l$, $N_l$ is the number of samples at level $l$ and $\varphi_{L_{\text{min}}-1} := 0$.

Here, we use the sequential Monte Carlo sampler, stated in Algorithm \ref{alg:smc}, to sample from the approximate coupling distribution at each level, shown in Algorithm~\ref{alg:mlsmc}. We now define intermediate distributions based on the coupling distribution, $F_{l}$, ${\bf L}_{l}$ and $\bar{\Pi}_l$. 
Define $H_{l,1},...,H_{l,J-1}$ by $H_{l,j} = F_{l,j+1}/F_{l,j}$
where $F_{l,j} = {\bf L}_{l}^{\lambda_j} \bar{\Pi}_l$ with $0=\lambda_1<\lambda_2<\cdots<\lambda_J=1$\footnote{Use the same adaptive tempering technique as in the SMC estimator shown in Example \ref{ex:adaptive_temp}}. And let $\cM_{l,j}$ for $j=1,...,J$ be such that $(\Pi_{l,j} \cM_{l,j}) (d\theta) = \Pi_{l,j}(d\theta)$.

For $j=1,...,J$,
\begin{equation*}
    \Pi_{l,j}^{N_{l}} (d\theta) := \frac{1}{N_{l}} \sum_{i=1}^{N_{l}} \delta_{\theta_{j}^{i}} (d\theta),
\end{equation*}
and then, we define
\begin{equation*}
    Z_{l}^{N_{l}} := \prod_{j=1}^{J-1} \Pi_{l,j}^{N_{l}} (H_{l,j}),
    \ \ F_{l}^{N_{l}}(d\theta) := Z_{l}^{N_{l}} \Pi_{l,J}^{N_{l}} (d\theta).
\end{equation*}

\begin{algorithm}[h!]
\begin{algorithmic}[1]
\caption{Sequential Monte Carlo \citep{MBG08,giles2015multilevel} sampler for coupled estimation at level $l$}
\State Simulate $\theta_{l,1}^i \sim \bar{\pi}_{l}$ for $i=1,\ldots,N_l$ and set $\theta_{l-1,1}^i = T_l(\theta_{l,1}^i)$.
\State Set $Z_{l,1}^{N_{l}}=1$
\State Compute the likelihoods $\cL_{l,1}(\theta_{l,1}^{i}),\cL_{l-1,1}(\theta_{l-1,1}^{i})$
\State Store the approximate coupling
\begin{equation*}
    F_{l,1}(\theta_{l,1}^{i},\theta_{l-1,1}^{i}) = \bar{\Pi}_{l}(\theta_{l,1}^{i},\theta_{l-1,1}^{i})
\end{equation*}
\For{$j=2,...,J$}
\State Store $Z_{l,j}^{N_{l}} = Z_{l,j-1}^{N_{l}} \frac{1}{N_{l}} \sum_{p=1}^{N_{l}} H_{l,j-1}(\theta_{l,j-1}^{p})$
\For{$i=1,...,N_{l}$}
\State Define $w_{l,j}^i = \frac{H_{l,j-1}(\theta^i_{l,j-1})}{\sum_{p=1}^{N_l}H_{l,j-1}(\theta^p_{l,j-1})}$
\State Resample: Select $I_{l,j}^i \sim \{ w_{l,j}^1,\ldots,w_{l,j}^{N_l} \}$, and set
$\hat{\theta}_{l,j}^i=\theta_{l,j-1}^{I_{l,j}^i}$,$\hat{\theta}_{l-1,j}^i=\theta_{l-1,j-1}^{I_{l,j}^i}$.
\State Mutate: Draw $\theta_{l,j}^i \sim \cM_{l,j}(\hat{\theta}_{l,j}^i,\cdot)$ and set
$\theta_{l-1,j}^i=T_l(\theta_{l,j}^i)$.
\State Compute the likelihoods $\cL_{l,j}(\theta_{l,j}^{i})$ and $\cL_{l-1,j}(\theta_{l-1,j}^{i})$.
\State Store the approximate coupling
\begin{equation*}
    F_{l,j}(\theta_{l,j}^{i},\theta_{l-1,j}^{i}) = \bar{\Pi}_{l}(\theta_{l,j}^{i},\theta_{l-1,j}^{i}) {\bf L}_{l}(\theta_{l,j}^{i},\theta_{l-1,j}^{i})^{\lambda_j} 
\end{equation*}
\State where ${\bf L}_{l}(\theta_{l,j}^{i},\theta_{l-1,j}^{i})  = \max \{ \cL_{l,j}(\theta_{l,j}^{i}),\cL_{l-1,j}(\theta_{l-1,j}^{i}) \}$
\EndFor
\EndFor
\State Define $\theta_l^i=\theta_{l,J}^i$, $\theta_{l-1}^i=\theta_{l-1,J}^i$. Store
$f_l(\theta_l^i)=\bar{\pi}_l(\theta_l^i)\cL_{l,J}(\theta_l^i)$, $f_{l-1}(\theta_{l-1}^i)=\bar{\pi}_{l-1}(\theta_{l-1}^i)\cL_{l-1,J}(\theta_{l-1}^i)$, and $F_l(\theta_l^i,\theta_{l-1}^i)=F_{l,J}(\theta_{l,J}^i,\theta_{l-1,J}^i)$.

\label{alg:mlsmc}
\end{algorithmic}
\end{algorithm}

\paragraph{Self-normalized estimator}
The trivial strategy for constructing the MLSMC estimator is to first normalized itself at each level and then compute the telescope sum.
Then, the MLSMC self-normalized estimator is given as 
\begin{equation}
\begin{aligned}
    \hat{\varphi}^{\text{MLSN}}_{L} 
    & = \sum_{l=L_{\text{min}}}^{L_{\text{max}}} \Bigg( \frac{F^{N_l}_{l}(\varphi_{l})}{F^{N_l}_{l}(1)} - \frac{F^{N_{l}}_{l-1}(\varphi_{l-1})}{F^{N_{l}}_{l-1}(1)} \Bigg)\\
    & = \sum_{l=L_{\text{min}}}^{L_{\text{max}}} \Bigg( \frac{\sum_{i=1}^{N_l}\varphi_l(\theta_l^i) \frac{f_l(\theta_l^i)}{F_l(\theta_l^i,\theta_{l-1}^i)}}{\sum_{i=1}^{N_l}\frac{f_l(\theta_l^i)}{F_l(\theta_l^i,\theta_{l-1}^i)}}-\frac{\sum_{i=1}^{N_l}\varphi_{l-1}(\theta_{l-1}^i)\frac{f_{l-1}(\theta_{l-1}^i)}{F_{l}(\theta_l^i,\theta_{l-1}^i)}}{\sum_{i=1}^{N_l}\frac{f_{l-1}(\theta_{l-1}^i)}{F_{l}(\theta_l^i,\theta_{l-1}^i)}} \Bigg). \label{eqn:mlsmcsn}
\end{aligned}
\end{equation}



\paragraph{Ratio estimator}
Besides the self-normalized estimator, a ratio estimator is given below. Let $\zeta=\{1,\varphi\}$, Define the difference function $\Delta\zeta: \Theta \rightarrow \bbR^m$ as
\begin{equation*}
    \Delta \zeta = \zeta_l(\theta_l)\frac{f_l(\theta_l)}{F_l(\theta_l,\theta_{l-1})} - \zeta_{l-1}(\theta_{l-1})\frac{f_{l-1}(\theta_{l-1})}{F_l(\theta_l,\theta_{l-1})},
\end{equation*}
and define the MLSMC ratio estimator as
\begin{equation}
\begin{aligned}
    \hat{\varphi}^{\text{MLRE}}_{L} 
    & = \frac{\sum_{l=L_{\text{min}}}^{L_{\text{max}}}F^{N_l}_{l}(\Delta \varphi)}{\max \{ \sum_{l=L_{\text{min}}}^{L_{\text{max}}} F^{N_l}_{l}(\Delta 1), \ Z_{\min} \}} \\
    & = \frac{\sum_{l=L_{\text{min}}}^{L_{\text{max}}} Z_l^{N_l} \Big(\frac{1}{N_l}\sum_{i=1}^{N_l}\varphi_l(\theta_l^i)\frac{f_l(\theta_l^i)}{F_l(\theta_l^i,\theta_{l-1}^i)}-\varphi_{l-1}(\theta_{l-1}^i)\frac{f_{l-1}(\theta_{l-1}^i)}{F_{l}(\theta_l^i,\theta_{l-1}^i)}\Big)}{\max \{\sum_{l=L_{\text{min}}}^{L_{\text{max}}} Z_l^{N_l} \Big(\frac{1}{N_l}\sum_{i=1}^{N_l}\frac{f_l(\theta_l^i)}{F_l(\theta_l^i,\theta_{l-1}^i)}-\frac{f_{l-1}(\theta_{l-1}^i)}{F_{l}(\theta_l^i,\theta_{l-1}^i)}\Big) , Z_{\min} \}}. \label{eqn:mlsmcrat}
\end{aligned}
\end{equation}

It is worth noting that when computing the likelihoods in practice, it may occur computationally small number that the computer will recognize as 0. However, the computer can store a much smaller number if one takes the logarithm of it. So, here we store and run with the log-likelihood instead of the likelihood and take the exponential of it when needed. For the small number of IS weights, we work with the log-likelihood for unnormalized weights and modify it with a constant.


The practical settings for the MLSMC estimators are shown in Appendix \ref{sec:mlsmc_practical_set}.

\subsubsection{Randomized multilevel sequential Monte Carlo ratio estimator}

In randomized method, give a starting level $L_{\text{min}}$ as before, and instead of giving a maximum level $L_{\text{max}}$ in advance, we choose the sample of levels randomly from a distribution from numbers larger or equal than $L_{\text{min}}$. Let ${\bf p}$ be a probability measure on $\bbP=\{l \in \bbN: l \geq L_{\text{min}} \}$, where $p_{l}:={\bf p}(l)$. We simulate $N$ i.i.d samples $l_i \sim {\bf p}$ and define the number of samples at level $l$ by $N_{l}=\#\{i:l_i=l\}$ for any $l \in \bbP$. Then, the posterior is 
\begin{equation*}
\begin{aligned}
    \pi(\varphi) = \sum_{l\in \bbP} \pi_{l}(\Delta \varphi)  
    \approx \sum_{l\in \bbP} \frac{1}{p_l N} \sum_{i=1}^{N_l} \Big\{ \varphi_l(\theta_l^i)-\varphi_{l-1}(\theta_{l-1}^i) \Big\}, 
\end{aligned}
\end{equation*}
where $\theta^i_{l}$ are i.i.d. samples at each level $l$ and $\varphi_{L_{\text{min}}-1} := 0$.

Following the same approximate technique and SMC sampler on coupling distributions in Algorithm \ref{alg:mlsmc},
The randomized MLSMC ratio estimator is defined as
\begin{equation}
\begin{aligned}
    \hat{\varphi}^{\text{rML}} & = \frac{\sum_{l\in \bbP} \frac{N_l}{p_l N} F^{N_l}_{l}(\Delta \varphi)}{\max\{\sum_{l\in \bbP} \frac{N_l}{p_l N} F^{N_l}_{l}(\Delta 1), Z_{\min}\}} \\
    & = \frac{\sum_{l\in \bbP} Z_l^{N_l} \Big(\frac{1}{p_l N}\sum_{i=1}^{N_l}\varphi_l(\theta_l^i)\frac{f_l(\theta_l^i)}{F_l(\theta_l^i,\theta_{l-1}^i)}-\varphi_{l-1}(\theta_{l-1}^i)\frac{f_{l-1}(\theta_{l-1}^i)}{F_{l}(\theta_l^i,\theta_{l-1}^i)}\Big)}{\max \{\sum_{l\in \bbP} Z_l^{N_l} \Big(\frac{1}{p_l N}\sum_{i=1}^{N_l}\frac{f_l(\theta_l^i)}{F_l(\theta_l^i,\theta_{l-1}^i)}-\frac{f_{l-1}(\theta_{l-1}^i)}{F_{l}(\theta_l^i,\theta_{l-1}^i)}\Big) , Z_{\min} \}},
\label{eqn:rmlsmcrat}
\end{aligned}
\end{equation}
where $Z_{\min}$ is added here to prevent an unexpected computational small number.

The practical settings for the rMLSMC estimator are shown in Appendix \ref{sec:rmlsmc_practical_set}.

\subsubsection{Serial multilevel sequential Monte Carlo estimator}

\begin{algorithm}[H]
\begin{algorithmic}[1]
\caption{Serial Sequential Monte Carlo sampler}
\State Simulate $\theta_0^{i} \sim \bar{\pi}_{0}$ for $i=1,...,N_0$
\For{$l=L_{\text{min}},...,L_{\text{max}}$}
\State Give $\theta_{l-1}^{i}$ for $i=1,...,N_{l-1}$ from previous level $l-1$ 
\State Simulate only the new/missing dofs for the next level ($l$) from the prior sample, with $\theta_{l}^i|\theta_{l-1}^i \sim \bar{\pi}_{l}$ for $i=1,...,N_{l}$
\State Compute and store the likelihood $\cL_{l}(\theta_{l}^{i})$
\For{$i=1,...,N_{l}$}
\State Define $w_{l,j}^i = \frac{\cL_{l}(\theta_{l}^i)}{\cL_{l-1}(\theta^i_{l-1})}$
\State Resample: Select $I_{l}^i \sim \{ w_{l}^1,...,w_{l}^{N_{l}} \}$, and let $\hat{\theta}_{l}^i=\theta_{l-1}^{I_{l}^i}$
\State Mutate: Draw $\theta_{l}^i \sim \cM_{l}(\hat{\theta}_{l}^i, \cdot)$
\State Compute and store the likelihood $\cL_{l}(\theta_{l}^{i})$
\EndFor
\EndFor
\label{alg:smlsmc}
\end{algorithmic}
\end{algorithm}

The algorithm of the serial multilevel sequential Monte Carlo sampler is shown in Algorithm~\ref{alg:smlsmc}, which is taken from \cite{chadaTNN,beskos2017multilevel, beskos2018multilevel}. This is a serial method in which the sampling in the current level is based on the previous level. Here, give the level parameter $L=\{L_{\text{min}}, L_{\text{max}}\}$, we follow the same path as in MLSMC to generate the intermediate distributions and mutation kernel here.

The serial MLSMC estimator is defined as
\begin{equation}
\begin{aligned}
    \hat{\varphi}^{\text{sML}} & = \sum_{l=L_{\text{min}}}^{L_{\text{max}}} \Bigg( \frac{ \sum_{i=1}^{N_l}\varphi_{l}(\theta_{l}^{i}) \cL_{l}(\theta_{l}^{i})}{ \sum_{i=1}^{N_l}\cL_{l}(\theta_{l}^{i})} - \frac{1}{N_{l}}\sum_{i=1}^{N_l}\varphi_{l-1}(\theta_{l-1}^{i})  \Bigg),
\label{eqn:smlsmc}
\end{aligned}
\end{equation}
where $\varphi_{L_{\text{min}}-1}:=0$. The practical settings follow the same procedure as the MLSMC in Appendix \ref{sec:mlsmc_practical_set}. However, the serial MLSMC is not parallelizable in level-level since its procedure is serial.

\section{PMMH}

\begin{algorithm}[h!]
\begin{algorithmic}[1]
\caption{PMMH at a given state \citep{andrieu2010particle, andrieu2009pseudo}}
\State Given $\mu,\theta$ and $p^{N}(\mu)$ 
\State Propose $\mu^{*} \sim Q(\cdot,\mu)$ and sample $\theta^{*} \sim P(\cdot|\mu^{*})$
\State Simulate $u \sim U[(0,1)]$ and compute 
\begin{equation*}
    \alpha = \min \Bigg\{ \frac{\pi_0(\mu^{*})Q(\mu|\mu^{*})p^{N}(\mu^{*})}{\pi_0(\mu)Q(\mu^{*}|\mu)p^{N}(\mu)}, \ 1 \Bigg\}.
\end{equation*}
\State If $u \leq \alpha$, then $(\mu,\theta,p^{N}(\mu)) = (\mu^{*},\theta^{*},p^{N}(\mu^{*}))$; otherwise,  $(\mu,\theta,p^{N}(\mu)) = (\mu,\theta,p^{N}(\mu))$.
\label{alg:pmmh}
\end{algorithmic}
\end{algorithm}

We first describe the pseudo-marginal Metropolis-Hastings (PMMH) method for inferring the hyperparameter $\mu$.
The target posterior distribution of this case is, given data $\cD$,
\begin{equation*}
    \pi(\mu|\cD) \propto \pi(\cD|\mu) \pi_0(\mu),
\end{equation*}
and one often found the posterior $\pi(\mu|\cD)$ might be analytically intractable or too complex to evaluate. The pseudo-marginal method then gives us an approach to computing it on an 
joint state space $\Theta \times \sS$, in which case we could compute $\pi(\mu|\cD)$ by computing the marginal of the joint posterior distribution
\begin{equation*}
    \pi(\theta,\mu|\cD) \propto \pi(\cD,d\theta|\mu) \pi_0(\mu) \propto f(\theta|\mu) \pi_0(\mu).
\end{equation*}
to get the target distribution as follows
\begin{equation*}
    \pi(\mu|\cD) = \int_{\Theta} \pi(d\theta,\mu|\cD) \propto \int_{\Theta} \pi(\cD,d\theta|\mu) \pi_0(\mu).
\end{equation*}
Then, the likelihood $p(\mu):=\pi(\cD|\mu)$ can be computed by
\begin{equation*}
    p(\mu) = \int_{\Theta} \pi(\cD,d\theta|\mu) = \int_{\Theta} f(d\theta|\mu).
\end{equation*}
We often call this likelihood $p(\mu)$ the marginal distribution since it computes the marginal of the joint likelihood.



One can then run a SMC sampler with the selection functions $h_j$ in each tempering step $j=1,...J-1$ to give an unbiased estimator of the marginal distribution
\begin{equation*}
    p^{N}(\mu) = \prod_{j=1}^{J-1} \frac{1}{N} \sum_{i=1}^{N} h_j(\theta_j^i|\mu), 
\end{equation*}
i.e. $\bbE[p^{N}(\mu)] = p(\mu)$, by sampling from the following auxiliary joint, on the space $\{1,...,N\}^{J} \times \Theta^{NJ}$,
\begin{equation*}
\begin{aligned}
    P(I_{1:J}^{1:N},\theta_{1:J}^{1:N}|\mu) & = \prod_{i=1}^{N} \bar{\pi}(\theta_1^{I_{1}^{i}}|\mu) \prod_{j=1}^{J-1} \cM(\theta_{j+1}^{I_{j+1}^{i}},\theta_{j}^{I_{j}^{i}}|\mu) \frac{h_{j}(\theta_{j}^{I_{j}^{i}}|\mu)}{\sum_{k=1}^{N} h_{j}(\theta_{j}^{k}|\mu)}.
\end{aligned}
\end{equation*}
Moreover, $p^{N}(\mu) \pi_0(\mu)$ is an unbiased estimator of the joint distribution $\pi(\cD,\mu)$ which is proportional to our target distribution $\pi(\mu|\cD)$. A general approach of the PMMH is given in Algorithm \ref{alg:pmmh}.

In the previous discussions, we fix the hyperparameter's value $\mu$ and omit it in the computation. We now expand the hyperparameter $\mu$ term in these expressions by writing joint variable $\varphi(x,\theta,\mu)$ (or conditional variable $\varphi(x,\theta|\mu)$) and joint variable $\pi(\theta,\mu)$ (or conditional variable $\pi(\theta|\mu)$ ). Also, sometimes we omit the term $x$ for symbol simplicity since the data are the same.

\section{Details of the S-MLPMMH approximation}
\label{app:smlpmmh}

We provide further details of the empirical-Bayes, Laplace, and
MAP-centered SBMC construction used in S-MLPMMH.

\paragraph{Hyperparameter transformation and empirical-Bayes estimate.}
Consider the transformed hyperparameter variable
$\eta=\log\mu$, where $\mu>0$ denotes the standard deviation
hyperparameter. We place a Gaussian hyper-prior on the transformed
variable $\eta$,
\begin{equation*}
\eta \sim \cN(m_{\eta},\sigma_{\eta}^{2}),
\end{equation*}
which is equivalent to assigning a log-normal prior to $\mu$.
Given the mean and variance of $\mu$, the corresponding Gaussian
parameters for $\eta=\log\mu$ are
\begin{equation*}
\sigma_{\eta}^{2} = \log\left(1+\frac{\operatorname{Var}(\mu)}{\mathbb{E}[\mu]^{2}}\right), \qquad
m_{\eta} = \log\left(\mathbb{E}[\mu]\right) - \frac{1}{2}\sigma_{\eta}^{2}.
\end{equation*}

At each level $l$, we first compute an empirical-Bayes MAP estimate
of the hyperparameter. For each candidate value of $\eta$, with
$\mu=\exp(\eta)$, we compute a MAP estimate of the model parameters
$\theta$ conditional on $\mu$. Around this conditional MAP estimate,
we approximate the marginal likelihood using a diagonal Laplace
approximation. The level-dependent empirical-Bayes objective is then
\begin{equation*}
J_l(\eta) = \log \hat{\pi}_{l}\left(\cD\mid\mu=\exp(\eta)\right) + \log \pi_{0}(\mu),
\end{equation*}
where
$\hat{\pi}_{l}(\cD\mid\mu=\exp(\eta))$
is the diagonal Laplace approximation of the marginal likelihood at level $l$, and $\pi_{0}(\mu)$ is the hyper-prior.

The resulting hyperparameter MAP center is
\begin{equation}
\label{eq:smlpmmh_eta_map}
\eta_{l,\sf MAP} = \arg\max_{\eta} J_l(\eta), \qquad \mu_{l,\sf MAP} = \exp(\eta_{l,\sf MAP}).
\end{equation}

\paragraph{Outer Laplace approximation.}
After obtaining $\eta_{l,\sf MAP}$, we construct a local Laplace
approximation to the hyperparameter posterior centered at
$\eta_{l,\sf MAP}$. Specifically, we estimate the local curvature of
the empirical-Bayes objective by
\begin{equation*}
H_{l,\eta} = - \left. \frac{\partial^{2}J_l(\eta)} {\partial\eta^{2}}
\right|_{\eta=\eta_{l,\sf MAP}},
\end{equation*}
and use the corresponding Laplace variance $V_{l,\eta}=H_{l,\eta}^{-1}$.
The outer SBMC hyperparameter samples are then independently drawn from
\begin{equation*}
\eta_{k,l} \sim \mathcal{N} \left(\eta_{l,\sf MAP},V_{l,\eta}\right),
\qquad \mu_{k,l} = \exp(\eta_{k,l}),\qquad k=1,\ldots,M_l.
\end{equation*}
In practice, $H_{l,\eta}$ is estimated using a finite-difference
approximation of $J_l$ around $\eta_{l,\sf MAP}$.

\paragraph{Inner MAP-centered SBMC approximation.}
For each sampled hyperparameter $\mu_{k,l}$, we then construct an inner MAP-centered SBMC sampler. The original inner prior is
\begin{equation*}
\bar{\pi}_{l} = \mathcal{N} \left(0, \mu_{k,l}^{2}\sf I\right).
\end{equation*}
Conditional on $\mu_{k,l}$, we compute the MAP estimator
\begin{equation*}
\theta_{l,\sf MAP} = \arg\max_{\theta} {\bf L}_{l} \left(\theta_l,\theta*{l-1} \right) \bar{\pi}_{l}\left(\theta_l,\theta_{l-1}\right).
\end{equation*}
The MAP estimate is recomputed for each hyperparameter sample
$\mu_{k,l}$.

Following the SBMC framework, we then establish a new inner prior
centered at the MAP estimate in \eqref{eqn:inner_map} and shrink its
covariance by a tuning parameter $s_{\theta}$. The resulting prior is
\begin{equation*}
\hat{\bar{\pi}}_{l}=\cN\left(\theta_{l,\sf MAP},s_{\theta}\mu_{k,l}^{2}\sf I\right).
\end{equation*}
In our numerical experiments, we use $s_{\theta}=0.1$.
This MAP-centered construction reduces the cost of the inner sampling stage while retaining the multilevel correction structure across
successive levels.

\section{Theoretical results}
\label{app:theory_background}

In this section, we present our main results related to our proposed methodology, discussed in Section \ref{sec:hyper_method}. Specifically, we first restate the theoretical results for MLSMC methods, provide a variance bound on the trace-class neural network and a convergence theorem related to our joint posterior distribution of interest is followed by.

We write $\sV = \Theta \times \sS$ with associated $\sigma-$algebra as $\cV$. Let $(\sW, \cW)$ be a measurable space such that $\sV \in \sW$. Let $K: \sW \times \cW \rightarrow [0,1]$ be any Markov kernel of invariant measure $\pi$ such that, for any $B \subset \cV$,
\begin{equation*}
    \int_{B \times (\sW \backslash \sV)} \pi(dw) = \Pi(B).
\end{equation*}

\subsection{Assumptions}

\paragraph{MLMC bound assumptions.} The assumption for boundedness of the function $f_{i,j}(\theta|\mu)$ is given in Assumption \ref{ass:G1_hyper} and for rate decay in parameters required for the application of the MLMC is given in Assumption \ref{ass:rate_hyper}.
\begin{ass} \label{ass:G1_hyper}
Let $J\in\bbN$ be given. For each $j\in\{0,\dots,J\}$ there exists some $C>0$ such that for all $(l,\theta,\mu)\in\bbN \times \Theta \times \sS$
$$
C^{-1} < Z_{l}, \ f_{l,j}(\theta|\mu) \leq C \, .
$$
\end{ass}

\begin{ass}\label{ass:rate_hyper}
For any $\varphi \in \cB_{b}(\Theta_{l} \times \Theta_{l-1} \times S) \cap \text{Lip} (\Theta_{l} \times \Theta_{l-1} \times \sS)$ there exist $c_1, c_2, c_3, \beta, \alpha, \gamma >0$ such that the following holds, under the joint posterior $\Pi_{l}$,
\begin{itemize}
\item[(B)] $\bbE\Big[\big|\varphi_{l}(\theta_l,\mu) \frac{f_{l}(\theta_l|\mu)}{F_{l}(\theta_l,\theta_{l-1}|\mu)} - f(\varphi(\theta,\mu)) \big| \Big] =: B_{l} \leq c_1  2^{-\alpha l}$;
\item[(V)] $ \bbE \Big[ \big|\varphi_{l} (\theta_l,\mu) \frac{f_{l}(\theta_l|\mu)}{F_{l}(\theta_l,\theta_{l-1}|\mu)} - \varphi_{l-1}(\theta_{l-1},\mu) \frac{f_{l-1}(\theta_{l-1}|\mu)}{F_{l}(\theta_l,\theta_{l-1}|\mu)} \big|^2 \Big] =: V_{l} \leq c_2 2^{-\beta l}$; 
\item[(C)] ${\rm COST}(\Pi_{l}(\varphi_l), \Pi_{l-1}(\varphi_{l-1})) =: C_{l} \leq c_3  2^{\gamma l}$.
\end{itemize}
\end{ass}
Assumption \ref{ass:rate_hyper} is related to the standard assumptions required in MLMC, where one has assumptions on the ``weak" error where (B) denotes the bias, the ``strong" error where (V) denotes the variance and the costs hence (C), where the cost is related to the cost of functional evaluations, and not time taken to complete the experiment.
Fix the value of the hyperparameter $\mu$, Assumption \ref{ass:G1_hyper} and \ref{ass:rate_hyper} falls into the original assumptions needed in MLMC.

\paragraph{PMMH assumption.} The assumption for transition in the PMMH is given in Assumption \ref{ass:K_hyper}.

\begin{ass}\label{ass:K_hyper}
    There exists a $\xi \in (0,1)$ and $\nu \in \cP(\sW)$ such that, for each $w \in \sW$, $\varphi \in \cB_{b}(\sW) \cap \text{Lip}(\sW)$,
    \begin{equation*}
        \int_{\sW} \varphi(w^{\prime}) K(w,dw^{\prime}) \geq \xi \int_{W} \varphi(v) \nu(dv).
    \end{equation*}
\end{ass}

The verification of Assumption \ref{ass:rate_hyper} and \ref{ass:K_hyper} above are given in Appendix \ref{app:verify_ass}.

\subsection{MLSMC results}

We restated theoretical results for MLSMC with self-normalized estimator\footnote{It is worth noting that only the ratio estimator has a rigorous theoretical result in the MIMC case (an extension from the MLMC), where the self-normalized does not.}, MLSMC with ratio estimator, randomized MLSMC estimator and serial MLSMC estimator, which are shown in Theorem \ref{thm:main_mlsmc_sn}, \ref{thm:main_mlsmc}, \ref{thm:main_rmlsmc} and \ref{thm:main_smlsmc}, respectively.

\begin{theorem}\label{thm:main_mlsmc_sn}
Assume Assumptions \ref{ass:G1_hyper} and \ref{ass:rate_hyper} with fixed $\mu$ and $\beta>\gamma$. 
Then for any $\varepsilon>0$ and suitable $\varphi: \Theta \rightarrow \bbR$, 
it is possible to choose a level paramemter $L=\{L_{\text{min}}, L_{\text{max}}\}$, where $L_{\text{min}}$ be the starting level and $L_{\text{max}}$ be the maximum level, such that for some $C>0$ 
$$
\bbE [ ( \widehat \varphi^{\rm MLSN}_L - \pi(\varphi) )^2 ] \leq C \varepsilon^2 \, ,
$$
and ${\rm COST}(\widehat \varphi^{\rm MLSN}_L) \leq C \varepsilon^{-2}$, the canonical rate.
The estimator $\widehat \varphi^{\rm MLSN}_L$ is defined in equation \eqref{eqn:mlsmcsn}.
\end{theorem}
\begin{proof}
    The proof is given in \cite{moral2017multilevel}.
\end{proof}

\begin{theorem}\label{thm:main_mlsmc}
Assume Assumptions \ref{ass:G1_hyper} and \ref{ass:rate_hyper} with fixed $\mu$ and $\beta>\gamma$. 
Then for any $\varepsilon>0$ and suitable $\varphi: \Theta \rightarrow \bbR$, 
it is possible to choose a level paramemter $L=\{L_{\text{min}}, L_{\text{max}}\}$, where $L_{\text{min}}$ be the starting level and $L_{\text{max}}$ be the maximum level, such that for some $C>0$ 
$$
\bbE [ ( \widehat \varphi^{\rm MLRE}_L - \pi(\varphi) )^2 ] \leq C \varepsilon^2 \, ,
$$
and ${\rm COST}(\widehat \varphi^{\rm MLRE}_L) \leq C \varepsilon^{-2}$, the canonical rate.
The estimator $\widehat \varphi^{\rm MLRE}_L$ is defined in equation \eqref{eqn:mlsmcrat}.
\end{theorem}
\begin{proof}
    The proof is given in \cite{law2022multi}.
\end{proof}

\begin{theorem}\label{thm:main_rmlsmc}
Assume \ref{ass:G1_hyper} and \ref{ass:rate_hyper} with fixed $\mu$ and
$\beta>\gamma$. 
Then, for 
suitable $\varphi: \Theta \rightarrow \bbR$, give a starting level $L_{\text{min}}$,  
it is possible to choose a probability distribution ${\bf p}$ on 
$\bbP=\{l \in \bbN: l \geq L_{\text{min}} \}$ such that for some $C>0$, 
$$
\bbE [ ( \widehat \varphi^{\rm rML} - \pi(\varphi) )^2 ] \leq C\varepsilon^2 \, ,
$$
and expected ${\rm COST}(\widehat \varphi^{\rm rMI}) \leq C \varepsilon^{-2}$, i.e. the canonical rate.
The estimator $\widehat \varphi^{\rm rMI}$ is defined in equation \eqref{eqn:rmlsmcrat}.
\end{theorem}
\begin{proof}
    The proof is given in \cite{cotter2022randomized}.
\end{proof}

\begin{theorem}\label{thm:main_smlsmc}
Assume Assumptions \ref{ass:G1_hyper} and \ref{ass:rate_hyper} with fixed $\mu$ and $\beta>\gamma$. 
Then for any $\varepsilon>0$ and suitable $\varphi: \Theta \rightarrow \bbR$, 
it is possible to choose a level paramemter $L=\{L_{\text{min}}, L_{\text{max}}\}$, where $L_{\text{min}}$ be the starting level and $L_{\text{max}}$ be the maximum level, such that for some $C>0$ 
$$
\bbE [ ( \widehat \varphi^{\rm sML}_L - \pi(\varphi) )^2 ] \leq C \varepsilon^2 \, ,
$$
and ${\rm COST}(\widehat \varphi^{\rm sML}_L) \leq C \varepsilon^{-2}$, the canonical rate.
The estimator $\widehat \varphi^{\rm sML}_L$ is defined in equation \eqref{eqn:smlsmc}.
\end{theorem}
\begin{proof}
    The proof is given in \cite{beskos2017multilevel,beskos2018multilevel}.
\end{proof}

\subsection{Trace-class neural network priors}

We begin this section by providing a result on the strong convergence, i.e. a bound on the variance related to the TNN prior. In particular, as shown in \cite{chadaTNN}, for a trace‐class neural network prior with smoothness parameter $s\geq 1/2$, 
the truncation error of the network output satisfies the decay rate in \eqref{eq:tnn_decay}. If the quantity of interest is obtained by applying a Lipschitz map $\Psi$ to the network output, the same rate transfers directly to $\varphi_l=\Psi\circ g_l$, see the Proposition \ref{prop:tnn}.

\begin{prop}\label{prop:tnn}
Assume that for all $z\in \bbR$, $\sigma(z) \leq |z|$. Then for $s \geq 1/2$, the priors defined as in~\eqref{eqn:prior_tnn} are trace-class. Furthermore, let $x \in \sX$ and $\mu \in \sS$, and consider the TNN $g_l(x,\theta_l,\mu)$ truncated at $n_l=2^l$ terms, for $l\in \bbN$, with the limit denoted $g(x,\theta,\mu) = \lim_{l \rightarrow \infty} g_l(x,\theta_l,\mu)$, as described above. Then there are $C^{\prime}(x), \ C(x)< +\infty$ such that 
\begin{equation}\label{eq:tnn_decay}
    \int_{\Theta \times \sS} |g_l(x,\theta_l,\mu)-g(x,\theta,\mu)|^2 \bar{\pi}(d\theta) \pi_0(d\mu) \leq C^{\prime} 2^{-(2s-1)l}.
\end{equation}
if $\Psi$ is Lipschitz continuous, then naturally we have
\begin{equation*}
    \int_{\Theta \times \sS} |\varphi_l(x,\theta_l,\mu)-\varphi(x,\theta,\mu)|^2 \bar{\pi}(d\theta) \pi_0(d\mu) \leq C 2^{-(2s-1)l}.
\end{equation*}
\end{prop}

\begin{proof}
The proof of the first part of Proposition \ref{prop:tnn} can be found in \cite{chadaTNN}. The second part can be derived instantly since $$|\varphi_l-\varphi| =|\Psi(g_l)-\Psi(g)| \leq L  |g_l-g|, \quad L>0,$$ where $L$ is the Lipschtiz constant.
\end{proof}


\subsection{Supporting results}

\begin{theorem}\label{thm:var_term}
    Assume Assumptions \ref{ass:G1_hyper}, \ref{ass:rate_hyper}. For any $\varphi: \Theta \times \sS \rightarrow \bbR^{d}$, there exists a $C < +\infty$ such that
    \begin{equation}
        \bbE\Bigg[ \Big( \Pi_{l}^{N,M_l}(\Delta \varphi_l) - \Pi_l(\Delta  \varphi_l) \Big)^2 \Bigg] \leq \frac{C 2^{-\beta l}}{M_l}.
    \end{equation}
\end{theorem}

\begin{proof}
    Proof sees Appendix \ref{app:proof_var_term}.
\end{proof}

\begin{lem}\label{lem:separate_terms}
    Let $a,b,c,d,A,B,C,D \in \bbR$ with $A,B,C,D$ non-zero then;
    \begin{equation}
    \begin{aligned}
        \frac{a}{A} - & \frac{b}{B} - \Bigg( \frac{c}{C} - \frac{d}{D} \Bigg) = \frac{[a-b-(c-d)]}{A} - \frac{b[A-B-(C-D)]}{AB} + \frac{1}{AC}[C-A][c-d] \\
        &- \frac{1}{AB}(b-d)(C-D) + \frac{d}{CBD}(B-D)(C-D) + \frac{d}{ACB}(A-C)(C-D).
    \end{aligned}
    \end{equation}
\end{lem}
\begin{proof}
    The Lemma is from Lemma C.3. in \citet{jasra2017multilevel} and its proof is given in Appendix \ref{lem:proof_separate_terms} for completeness.
\end{proof}

Define $v^k = (\theta_{k,l}, \mu_{k,l})$ as the relevant variables of $w^k$ from the MCMC kernel, and defining 
$$
\tilde{\varphi}_{l}(v^k) := \Big(\varphi_l(\theta_{k,l},\mu_{k,l}) \frac{f_{l}(\theta_{k,l}|\mu_{k,l})}{F_{l}(\theta_{k,l},\theta_{k,l-1}|\mu_{k,l})} - \varphi_{l-1}(\theta_{k,l-1},\mu_{k,l}) \frac{f_{l-1}(\theta_{k,l-1}|\mu_{k,l})}{F_{l}(\theta_{k,l},\theta_{k,l-1}|\mu_{k,l})} \Big),
$$
and the map $\psi: w^k \rightarrow v^k$.

\begin{prop}\label{prop:bound_term}
    Assume Assumptions \ref{ass:G1_hyper}, \ref{ass:K_hyper}, \ref{ass:rate_hyper}. Suppose that $\{ W^i \}_i$ is a Markov chain with the Markov kernel $K$, and $W^1 \sim \pi$. Then there exists a $\beta > 0$ such that for any $\varphi \in \cB_b(\Theta \times \sS) \cap \text{Lip}(\Theta \times \sS)$ there exists a $C>0$ such that
    \begin{equation*}
        \bbE \Bigg[ \Big( \frac{1}{M_l} \sum_{k=1}^{M_l} \tilde{\varphi}_{l}(V^{k}) - \Pi_{l}(\tilde{\varphi}_{l}) \Big)^2 \Bigg] \leq \frac{C 2^{-\beta l}}{M_l},
    \end{equation*}
    where $V^k = (\theta_{k,l}, \mu)$ is the relevant variables of $W^k$.
\end{prop}

\begin{proof}
Proof sees the Appendix \ref{app:proof_bound_term}.
\end{proof}

\section{Additional Numerical Experiments}
\label{app:numerical_details}

We begin by presenting a comparative study of several MLSMC methods for approximating the target conditional posterior distribution. Next, we conduct a numerical verification on a synthetic regression benchmark to confirm both the regularity and complexity rate of the MLPMMH method. Furthermore, we justify the TNN architecture on the MNIST subset through controlled MAP training experiments. 
The COST is measured as the time (arithmetic) complexity, i.e. the number of basic floating-point operations needed in the coupled level.

\subsection{Problem setting}

In the case that data is defined in continuous space, such as the regression problem, we have $Y=\bbR^{m}$, then we can assume the data follows
\begin{equation}\label{eqn:bayes_reg}
    y_i = g(x_i,\theta,\mu) + \varepsilon_i, \ \varepsilon \mysim \cN_{m}(0,\Sigma_i),
\end{equation}
where $i = \{1,...,N_{\text{tr}}\}$ and $\cN_{m}(0,\Sigma_i)$ is the $m$-dimensional Gaussian distribution with mean zero and covariance matrix $\Sigma_i$.

Then the likelihood for the regression problem \eqref{eqn:bayes_reg} can be computed as
\begin{equation*}
    \cL(\theta,\mu) = \prod_{i=1}^{N_{\text{tr}}} \phi_m(y_i;g(x_i,\theta,\mu),\Sigma_i), 
\end{equation*}
where $\psi_{m}(y_i;g(x_i,\theta,\mu),\Sigma_i)$ is the $m$-dimensional Gaussian density of mean $g(x_i,\theta,\mu)$, covariance matrix $\Sigma_i$ evaluated at $y_i$.

In the discrete data set, like the classification problem, we have $\sY = \{1,...,m\}$, then we instead define the so-called \textit{softmax} function as
\begin{equation}\label{eqn:bayes_class}
    h_k(x,\theta,\mu) = \frac{\exp(g_{k}(x,\theta,\mu))}{\sum_{j=1}^{m}\exp(g_{j}(x,\theta,\mu))}, \ k \in \sY,
\end{equation}
and define $h(x,\theta,\mu)=(h_1(x,\theta,\mu),...,h_m(x,\theta,\mu))$ as a categorical distribution on $m$ outcomes based on data $x$. Then we assume that $y_i \mysim h(x_i)$ for $i=\{1,...,N_{\text{tr}}\}$.

The likelihood for the classification problem \eqref{eqn:bayes_class} can be computed as
\begin{equation*}
    \cL(\theta,\mu) = \prod_{i=1}^{N_{\text{tr}}} \prod_{k=1}^{m} h_{k}(x_i,\theta,\mu)^{\dsone_{[y_i=k]}},
\end{equation*}
where $\dsone_{[y_i=k]} =  1$ if $y_i=k$ and $0$ otherwise.


\subsubsection{MNIST dataset}

Consider the MNIST handwritten–digit dataset, which contains a total of $60\,000$ training and $10\,000$ test images of digits $0$–$9$. Denote the first $N_{tr}$ images from the training be the train dataset as $\{(x_i,\,y_i)\bigr\}_{i=1}^{N_{\text{tr}}}\subset\bbR^{784}\times\{0,\dots,9\}$ and first $N_{\text{test}}$ images from test be the test dataset as $\{(x_j^{\text{test}},\,y_j^{\text{test}})\bigr\}_{j=1}^{N_{\text{test}}}\subset\bbR^{784}\times\{0,\dots,9\}$, with every image vectorised and standardized to lie in $[0,1]^{784}$. 

\subsection{Comparison among MLSMC methods}


We compare three MLSMC methods on the synthetic regression example using the Bayesian TNN model with fixed hyperparameters $\mu = (\sigma_A, \sigma_b) = (1,1)$. The number of neural layers is $D=3$, the input neural network width is then set to $n = 10$, and the output neural network width is set to $m=1$. The TNN prior uses a rate decay of $s = 2$, and the pCN kernel 
is applied with $10$ mutation steps in three MLSMC methods.

A regularity test is conducted to estimate the decay of bias and variance for normalized estimator differences between levels under the posterior distribution. The results are shown in Figure~\ref{fig:mlsmc_regularity}, supporting the decay rate assumed in Assumption~\ref{ass:rate_hyper} with fixed $\mu$ for the convergence analysis. Starting from level $l = 4$, the decay of biases and weights aligns with the expected rate. In the MLSMC setting for TNN models, the bias decay is uniformly bounded due to Jensen's inequality, ensuring at least $\alpha = \beta/2$, where the strong convergence rate is $\beta = 2s - 1 = 3$ with $s = 2$. It is important to note that the first level does not need to follow the rate, as it is treated as a single level. 

Taking $L_{\text{min}} = 3$ as the starting level, the convergence test for the three MLSMC methods, (r)MLSMC and serial MLSMC, is presented in Figure~\ref{fig:mlsmc_complexity}. Given a target accuracy level $\varepsilon$, the maximum level $L_{\max}$ and the number of samples per level $N_l$ in both MLSMC and serial MLSMC are followed the computation in Section~\ref{sec:mlsmc_practical_set}. For rMLSMC, level indices $l$ and corresponding $N_l$ are selected based on the index probability distribution, as described in Section~\ref{sec:rmlsmc_practical_set}. The reference solution is computed at accuracy $\varepsilon = 10^{-4}$ using MLSMC.

\begin{figure}[h!]
\begin{subfigure}{.49\textwidth}
  \centering
  \includegraphics[width=.9\linewidth]{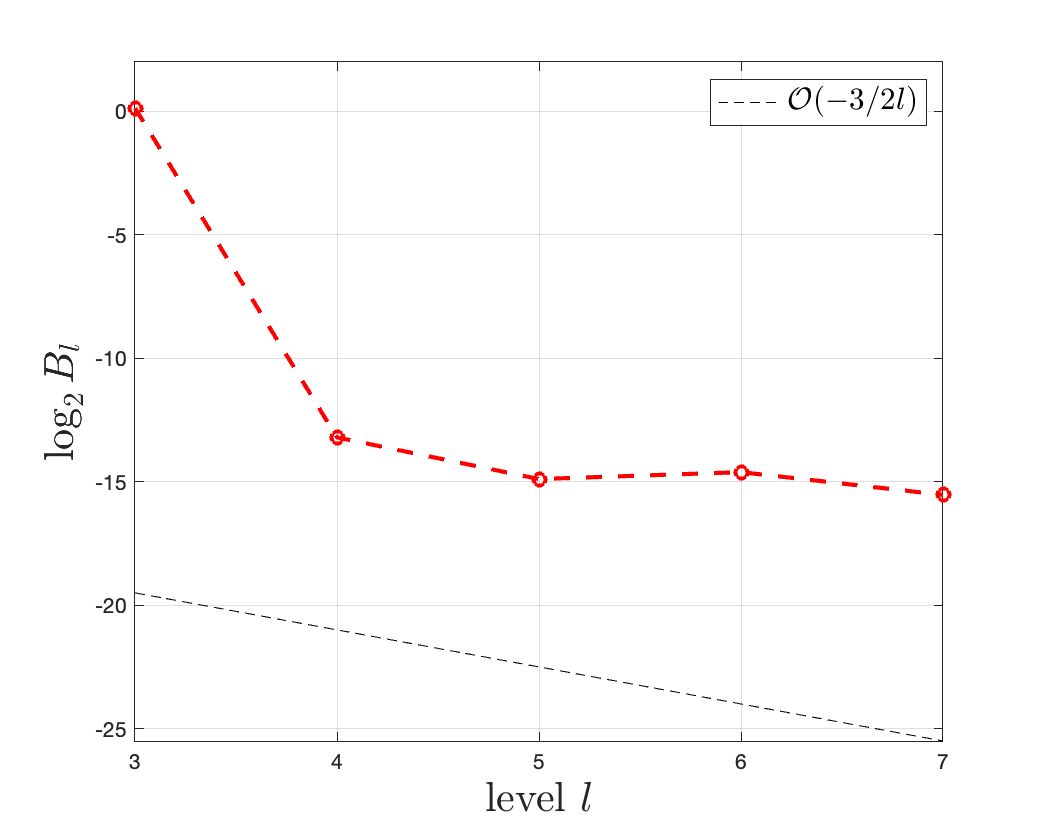}  
  \caption{}
\end{subfigure}
\begin{subfigure}{.49\textwidth}
  \centering
  \includegraphics[width=.9\linewidth]{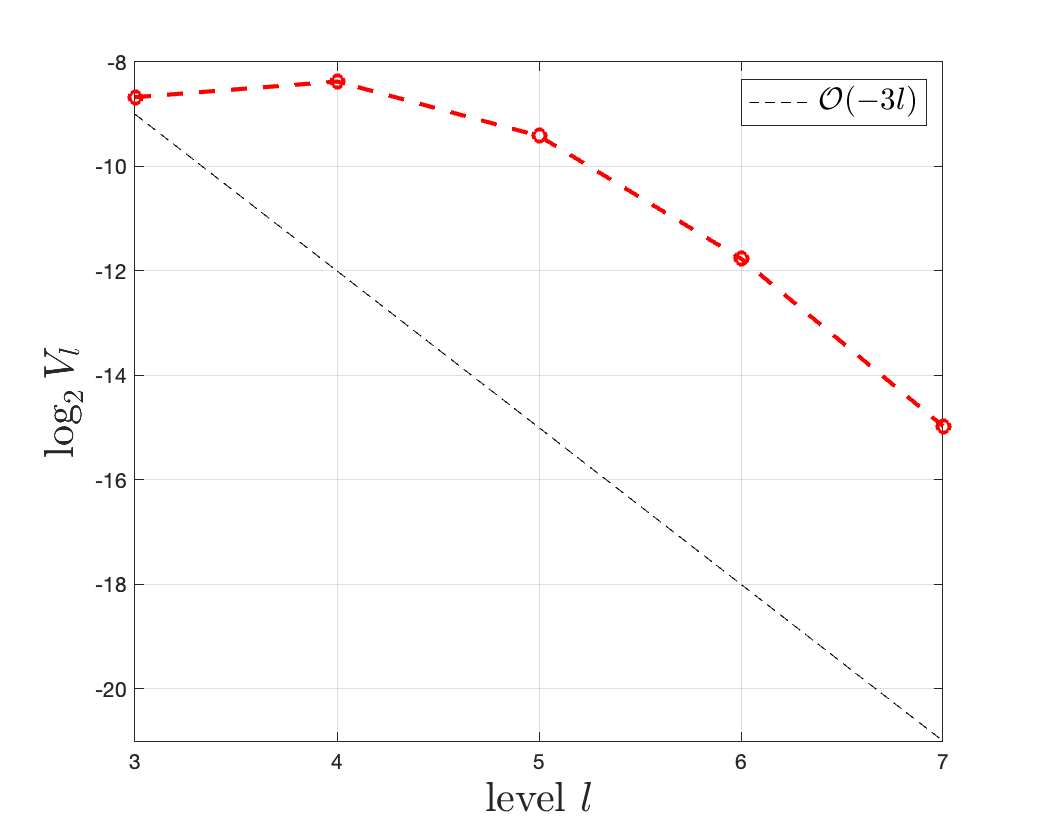}  
  \caption{}
\end{subfigure}
\caption{Regularity test for different MLSMC on Bayesian regression model with fixed $\mu$: \small\textit{Empirical values for strong and weak convergence parameters in normalized estimator. All plots use $30$ realizations with $10000$ samples in each index, as the maximum level $L_{\max}=7$ and starting level as $L_{\text{min}}=3$ at each experiment. (a) Logarithm plot with base $2$ for mean of difference and reference line with slope $-3/2$; (b) Logarithm plot with base $2$ for variance of difference and reference line with slope $-3$.}}
\label{fig:mlsmc_regularity}
\end{figure}

\begin{figure}[h!]
  \centering
  \includegraphics[width=.6\linewidth]{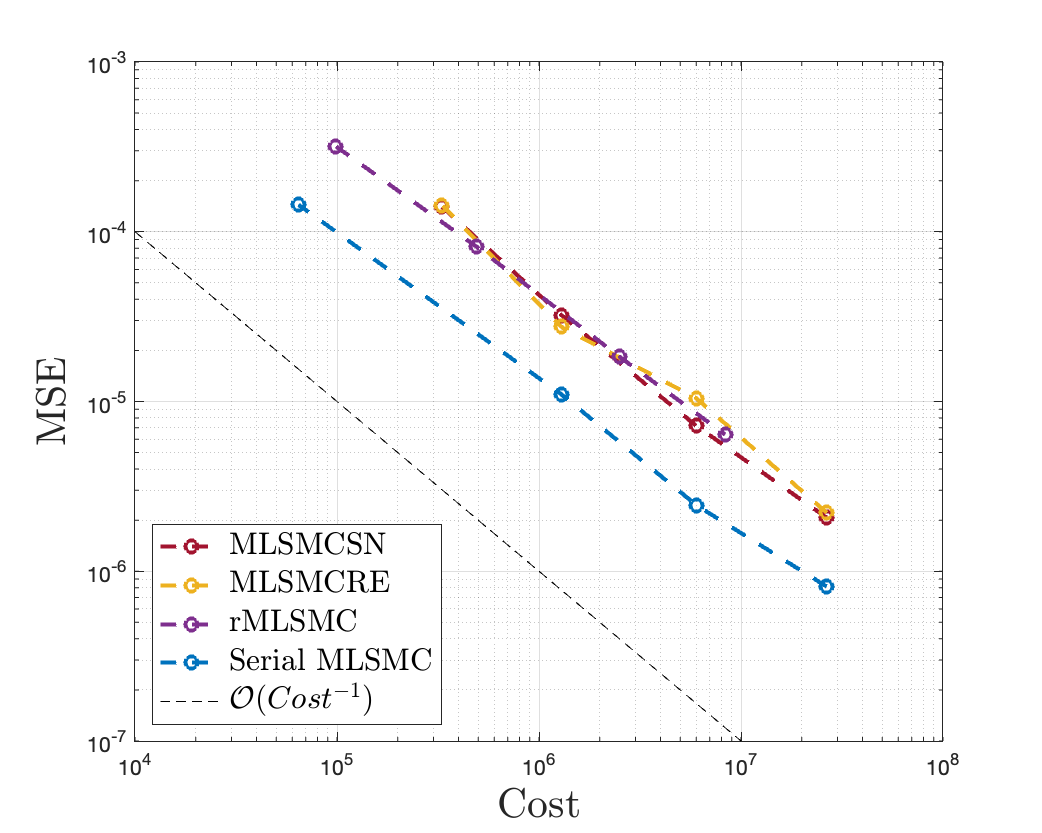}
  \caption{Complexity test for different MLSMC methods on Bayesian regression model: \small\textit{We use $30$ realisations to compute the MSE for each experiment and use the rate of convergence to compare the different methods. Rate of regression: MLSMCSN(red line) is $-0.962$, MLSMCRE (yellow line) is $-0.916$, rMLSMC (purple line) is $-0.885$ and Serial MLSMC (blue line) is $-0.870$.}}
\label{fig:mlsmc_complexity}
\end{figure}

\subsection{Numerical justification for MLPMMH}

\begin{figure}[h!]
\begin{subfigure}{.49\textwidth}
  \centering
  \includegraphics[width=.9\linewidth]{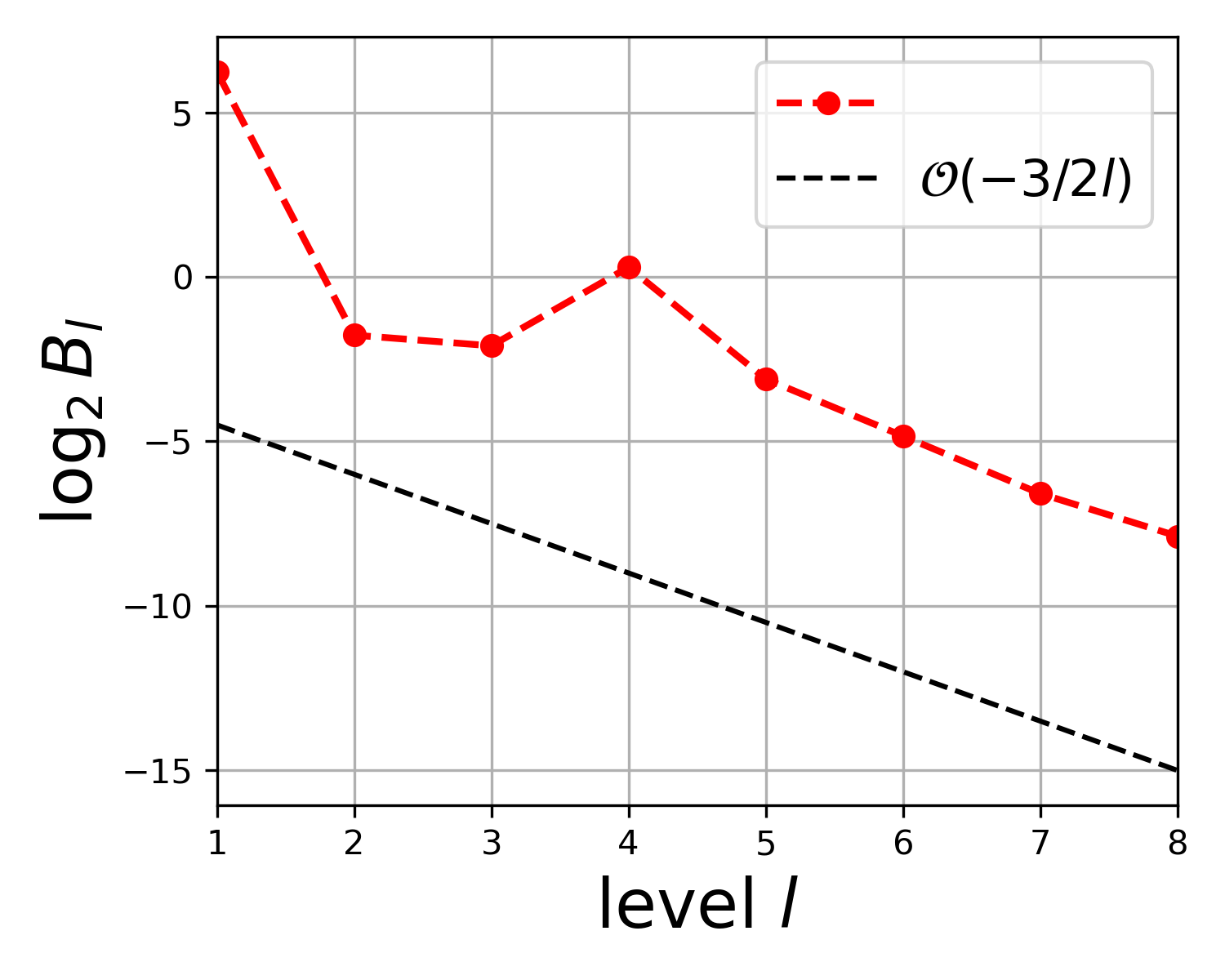}  
  \caption{}
\end{subfigure}
\begin{subfigure}{.49\textwidth}
  \centering
  \includegraphics[width=.9\linewidth]{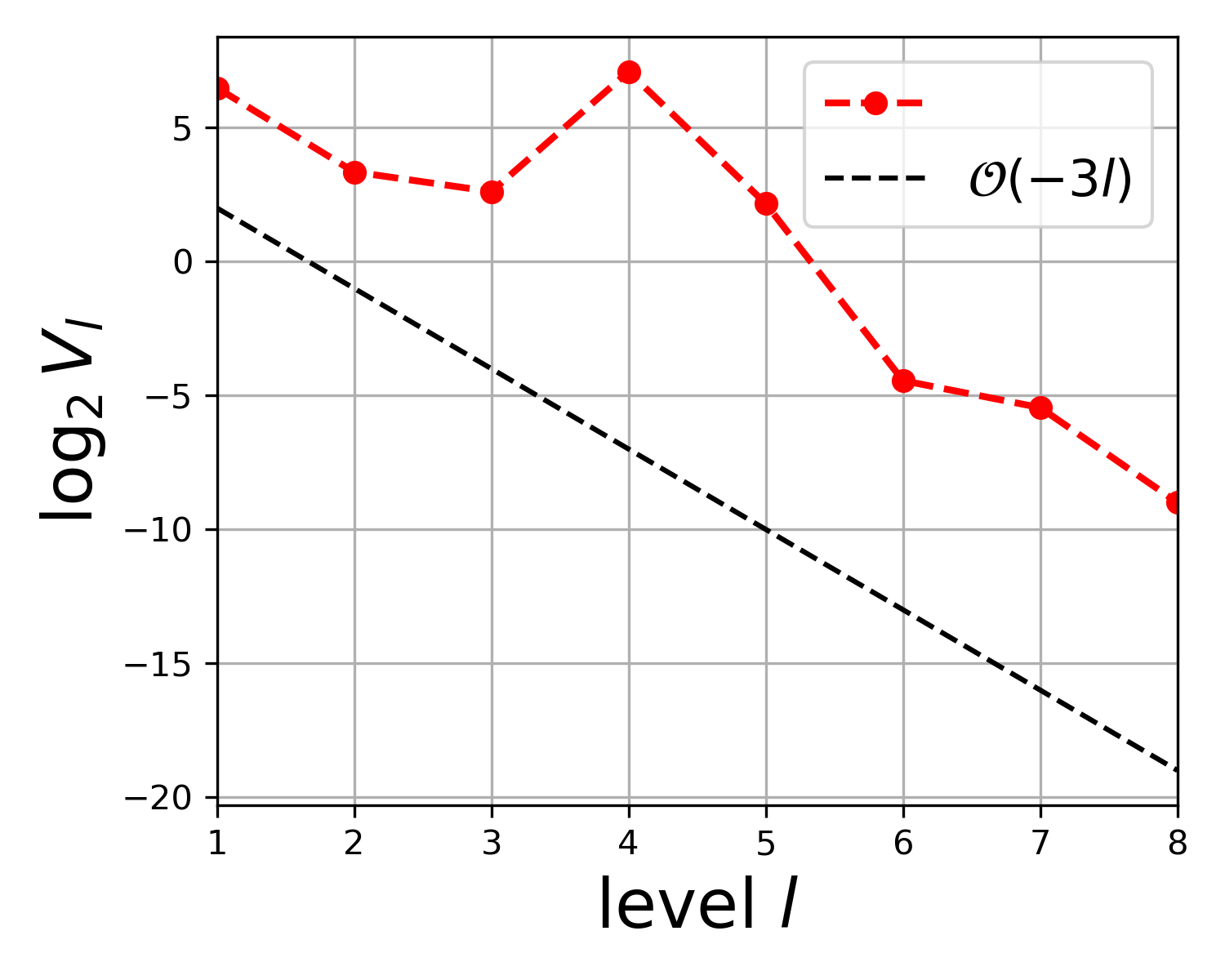}  
  \caption{}
\end{subfigure}
\caption{Regularity test for MLPMMH on Bayesian regression model: \small\textit{Empirical values for strong and weak convergence parameters in normalized estimator. All plots use $50$ realizations with $M_{l}=100$ samples in each index, as the maximum level $L_{\max}=8$ and starting level as $L_{\text{min}}=1$ at each experiment, with $50$ burn-in samples. For the MLSMC sampler, $N=100$ samples are used with an HMC kernel of one leapfrog step and $10$ mutation steps. (a) Logarithm plot with base $2$ for mean of difference and reference line with slope $-3/2$; (b) Logarithm plot with base $2$ for variance of difference and reference line with slope $-3$.}}
\label{fig:mlpmmh_regularity}
\end{figure}

The plots for autocorrelation functions of $\sigma_{A}$ and $\sigma_b$ are given in Figure \ref{fig:mlpmmh_acf}, suggesting that a reasonable number of burn-in samples in the MCMC in MLPMMH is $50$.

\begin{figure}[h!]
\begin{subfigure}{.48\textwidth}
  \centering
  \includegraphics[width=.9\linewidth]{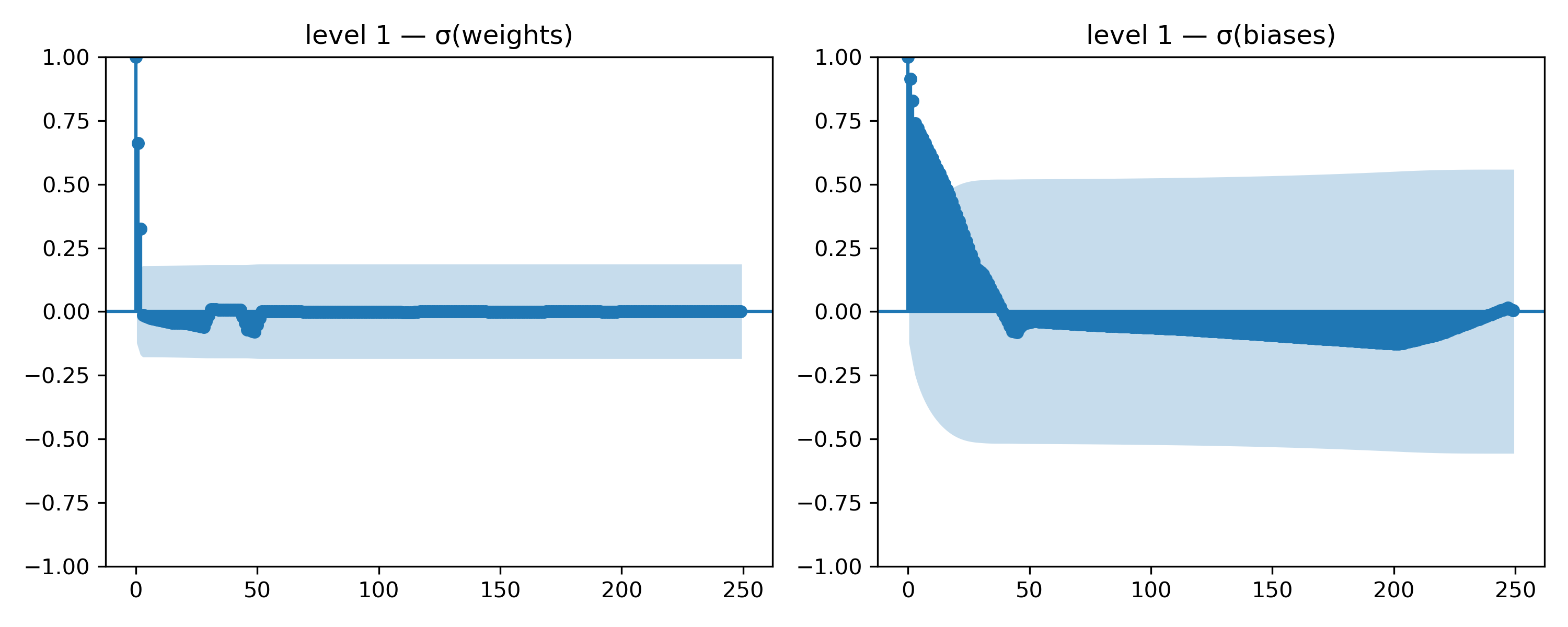}  
  \caption{$l=1$: $\sigma_{A}$ (left) and $\sigma_{b}$ (right)}
\end{subfigure}
\hfill
\begin{subfigure}{.48\textwidth}
  \centering
  \includegraphics[width=.9\linewidth]{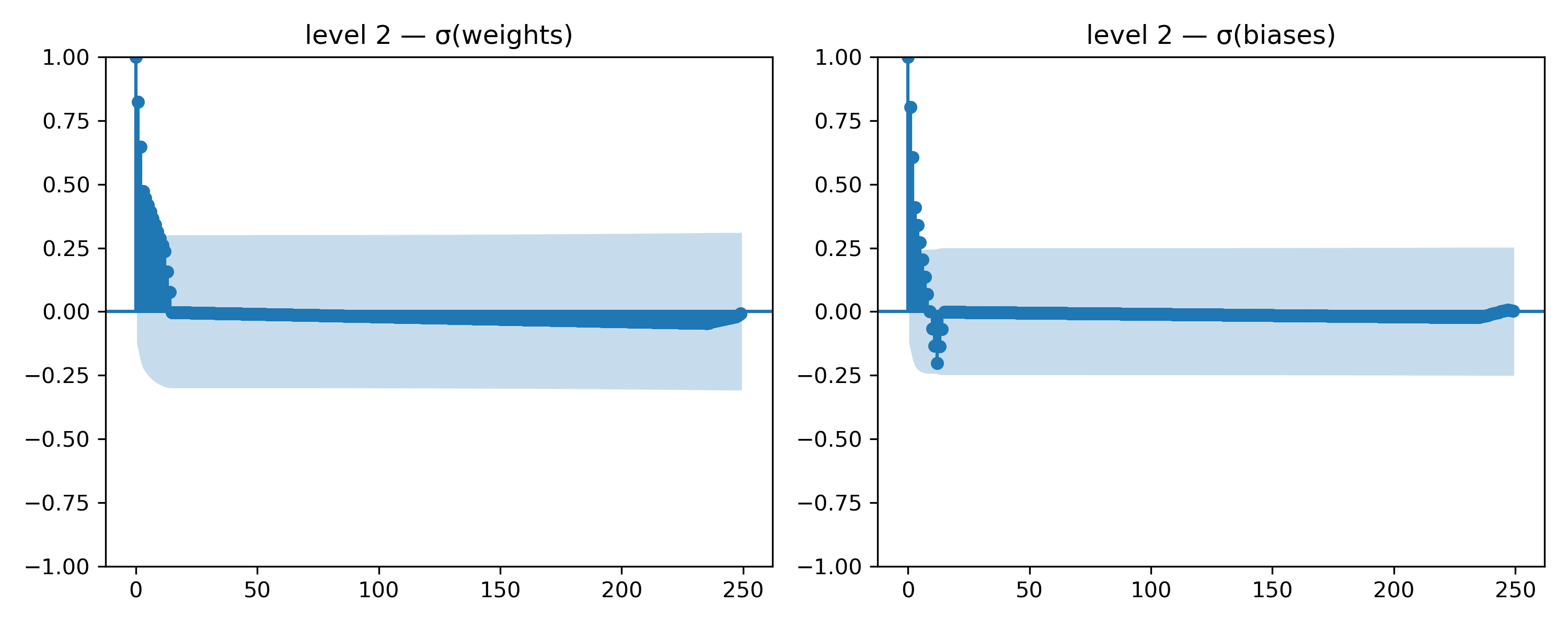}  
  \caption{$l=2$: $\sigma_{A}$ (left) and $\sigma_{b}$ (right)}
\end{subfigure}
\begin{subfigure}{.48\textwidth}
  \centering
  \includegraphics[width=.9\linewidth]{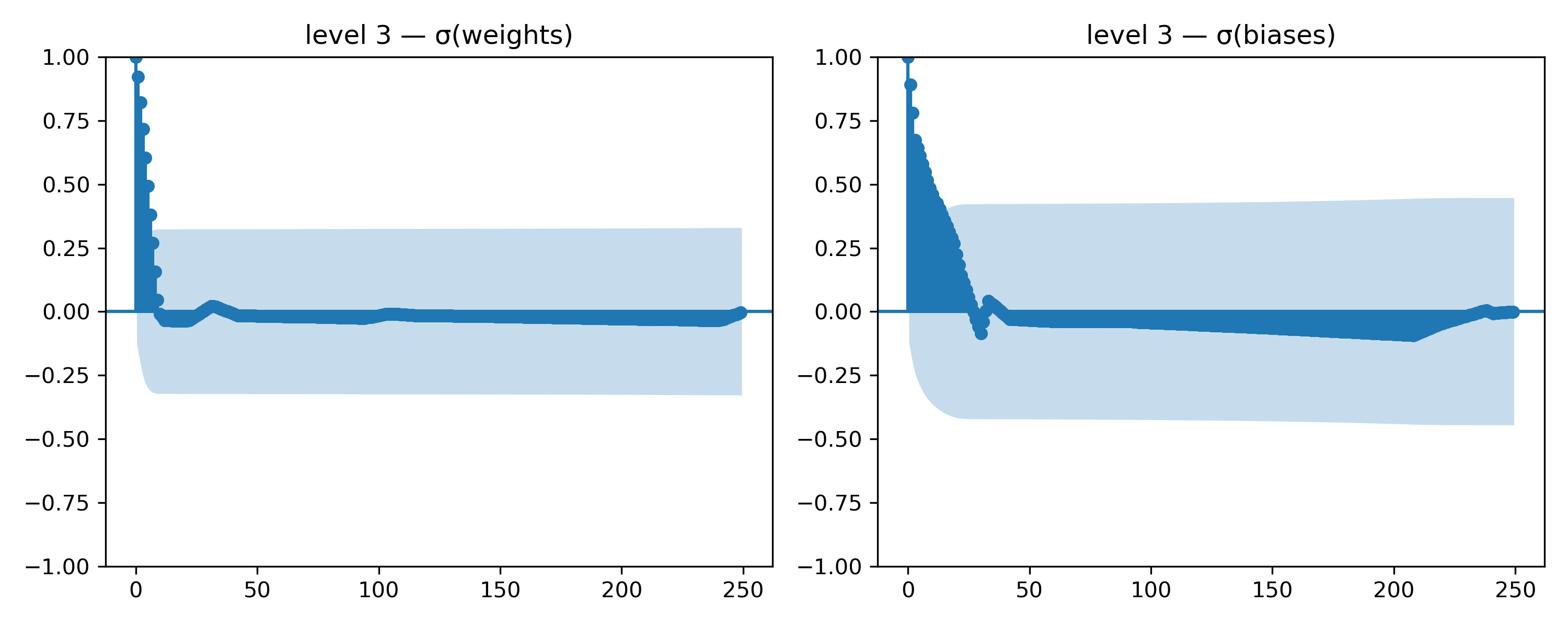}  
  \caption{$l=3$: $\sigma_{A}$ (left) and $\sigma_{b}$ (right)}
\end{subfigure}
\hfill
\begin{subfigure}{.48\textwidth}
  \centering
  \includegraphics[width=.9\linewidth]{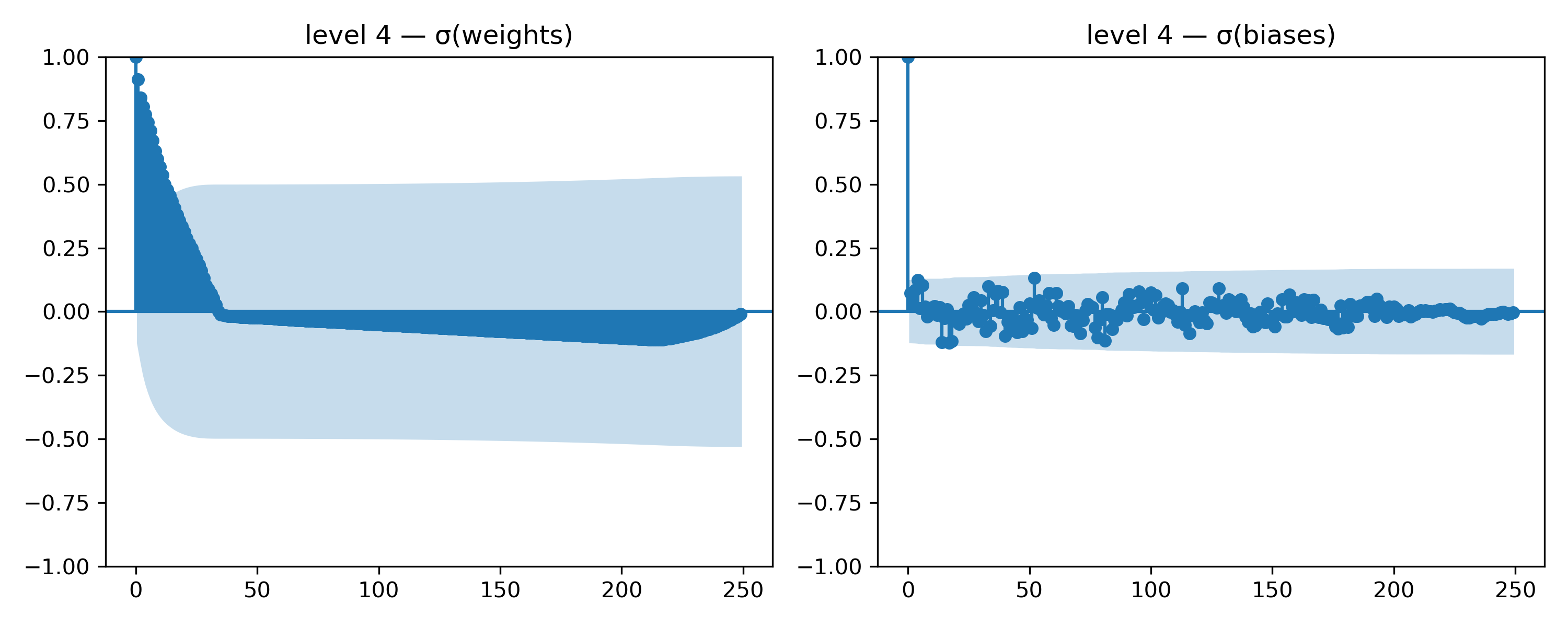}  
  \caption{$l=4$: $\sigma_{A}$ (left) and $\sigma_{b}$ (right)}
\end{subfigure}
\begin{subfigure}{.48\textwidth}
  \centering
  \includegraphics[width=.9\linewidth]{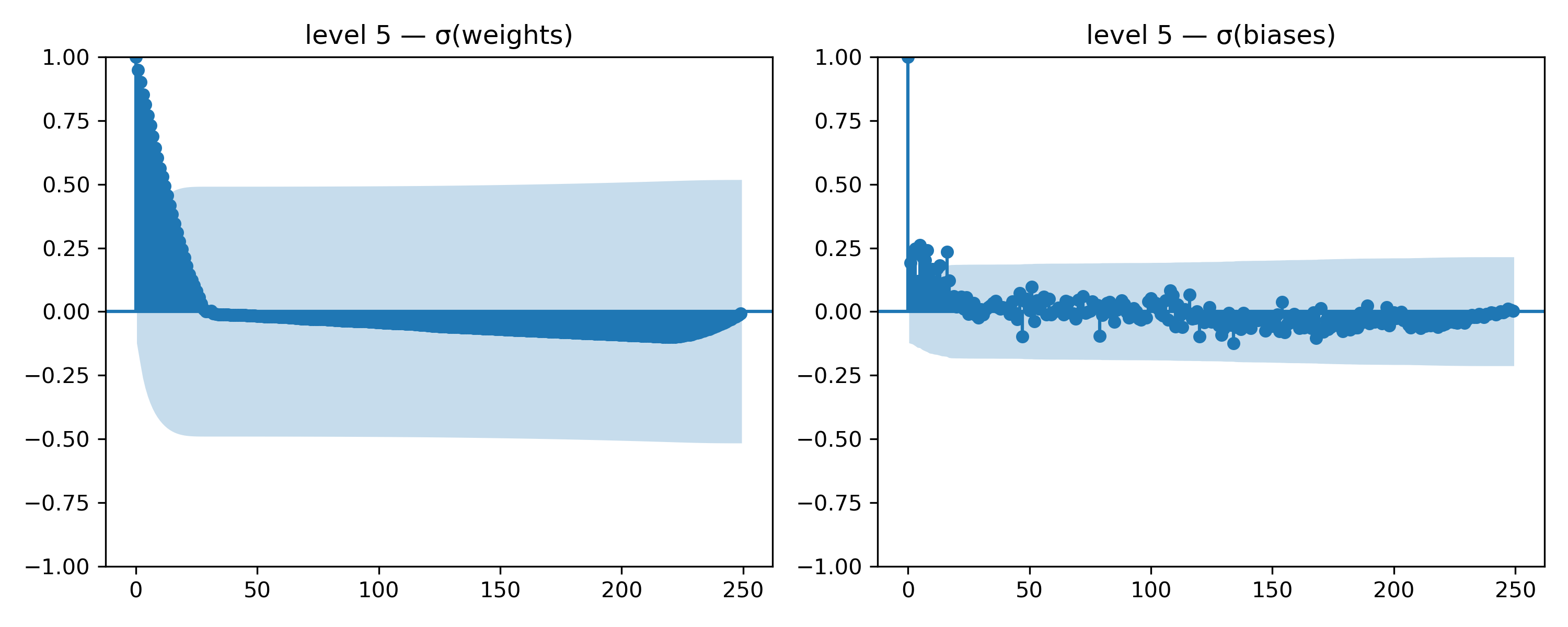}  
  \caption{$l=5$: $\sigma_{A}$ (left) and $\sigma_{b}$ (right)}
\end{subfigure}
\hfill
\begin{subfigure}{.48\textwidth}
  \centering
  \includegraphics[width=.9\linewidth]{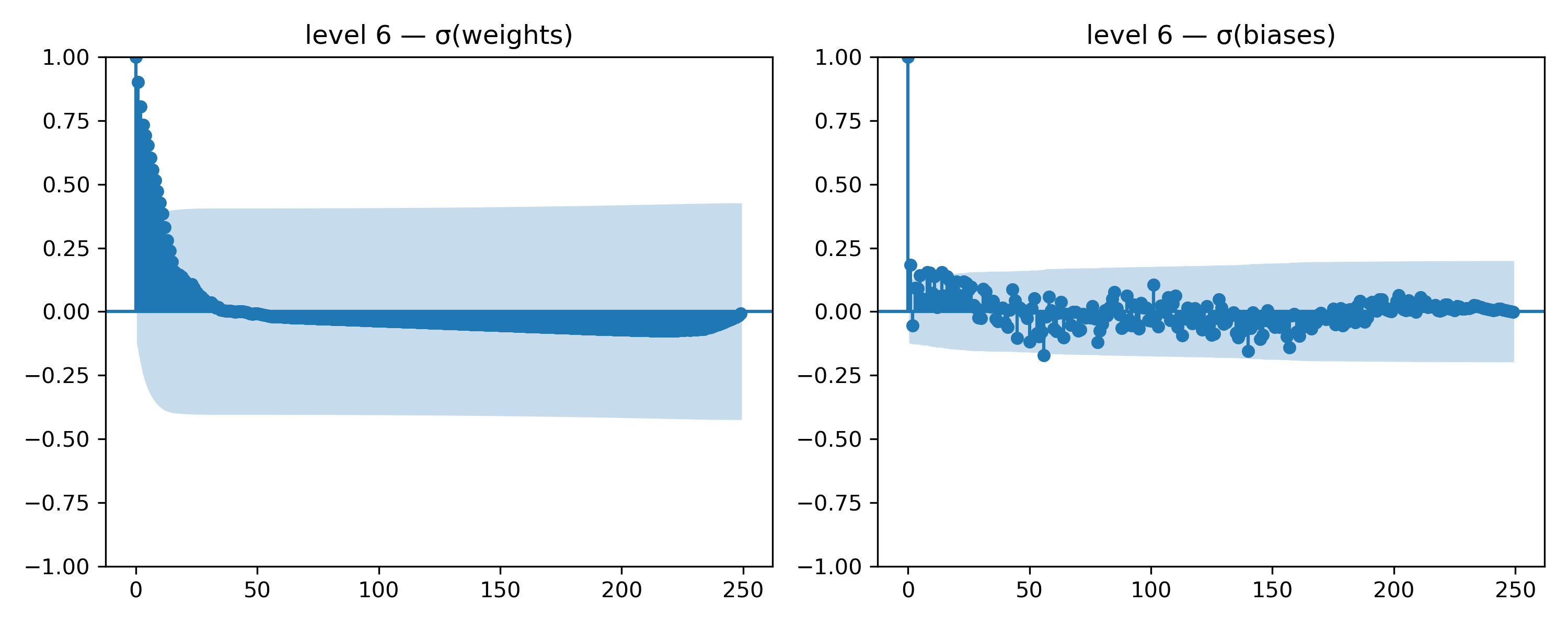}  
  \caption{$l=6$: $\sigma_{A}$ (left) and $\sigma_{b}$ (right)}
\end{subfigure}
\begin{subfigure}{.48\textwidth}
  \centering
  \includegraphics[width=.9\linewidth]{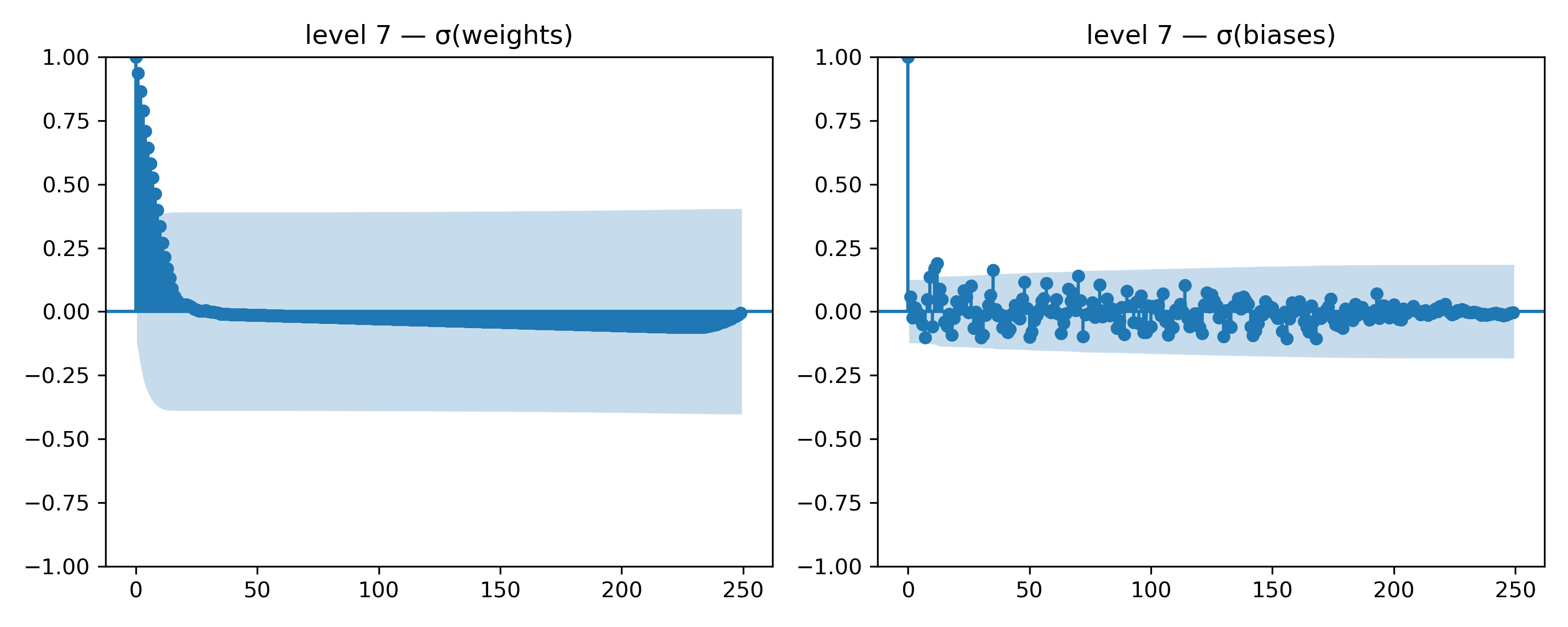}  
  \caption{$l=7$: $\sigma_{A}$ (left) and $\sigma_{b}$ (right)}
\end{subfigure}
\hfill
\begin{subfigure}{.48\textwidth}
  \centering
  \includegraphics[width=.9\linewidth]{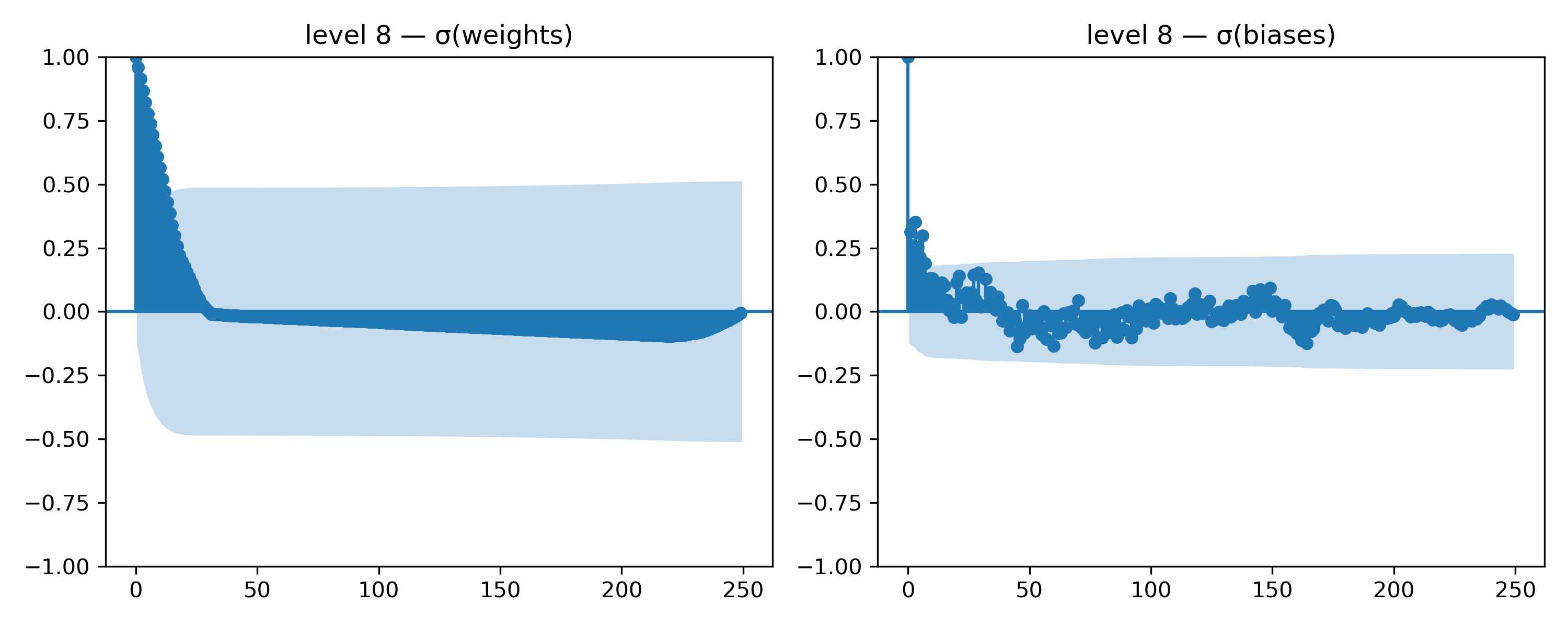}  
  \caption{$l=8$: $\sigma_{A}$ (left) and $\sigma_{b}$ (right)}
\end{subfigure}
\caption{Autocorrelation functions for $\sigma_A$ and $\sigma_B$ on Bayesian regression model: \small\textit{All plots use $50$ realizations with $M_{l}=200$ samples in each index $l$ as the maximum level $L_{\max}=8$ and starting level as $L_{\text{min}}=1$ at each experiment, with $50$ burn-in samples. For the MLSMC sampler, $N=100$ samples are used with an HMC kernel of one leapfrog step and $10$ mutation steps.}}
\label{fig:mlpmmh_acf}
\end{figure}

\subsection{Empirical verification of TNN prior structure}

We use a fully connected neural network ($D=3$) with input width $n = 784$, the hidden neural network width $n_1=128$ (equivalent to $l=7$), and the output neural network width $m=10$. We compare three training regimes:
\begin{itemize}
    \item \textbf{NoReg}: standard training with no regularization, corresponding to pure data fitting;
    \item \textbf{Weight-decay}: MAP training with an isotropic Gaussian prior, corresponding to standard $L_2$ regularization;
    \item \textbf{Width-decay}: MAP training with the TNN-style prior, where weights are regularized according to their expected variance under the TNN prior. In particular, 
    $$Var[w_{ij}] \propto 
    \begin{cases}
        \sigma^2 i^{-s}, & \text{for the input-to-hidden layer},\\
        \sigma^2 (ij)^{-s}, & \text{for hidden-to-hidden layer},\\
        \sigma^2 j^{-s}, & \text{for the hidden-to-output layer},
    \end{cases}
    $$
    where $\sigma=\sigma_A$ for weights and $\sigma=\sigma_b$ for biases.
\end{itemize}

All models were initialized using the 
PyTorch default initialization and trained with the same Adam optimiser (lr=$10^{-3}$) for $20$ epochs\footnote{This is the sufficient number of epochs to ensure convergence and avoid overfit for the model.} for NoReg, Weight-decay, Width-decay respectively.

\subsection{Practical application of MLPMMH}

We first verified that the PCA decay in the featured matrix of the MINST10 training set follows a similar decay as in the TNN-style weight-decay as in Figure \ref{fig:mlpmmh_pca_decay}. 
The eigenvalues decay approximately linearly on the log-log scale, with fitted slope $-0.615$. Interpreting this as
a decay of the form $j^{-s/2}$ gives the empirical value $s=1.23$.

\begin{figure}[h!]
\centering
  \includegraphics[width=.47\linewidth]{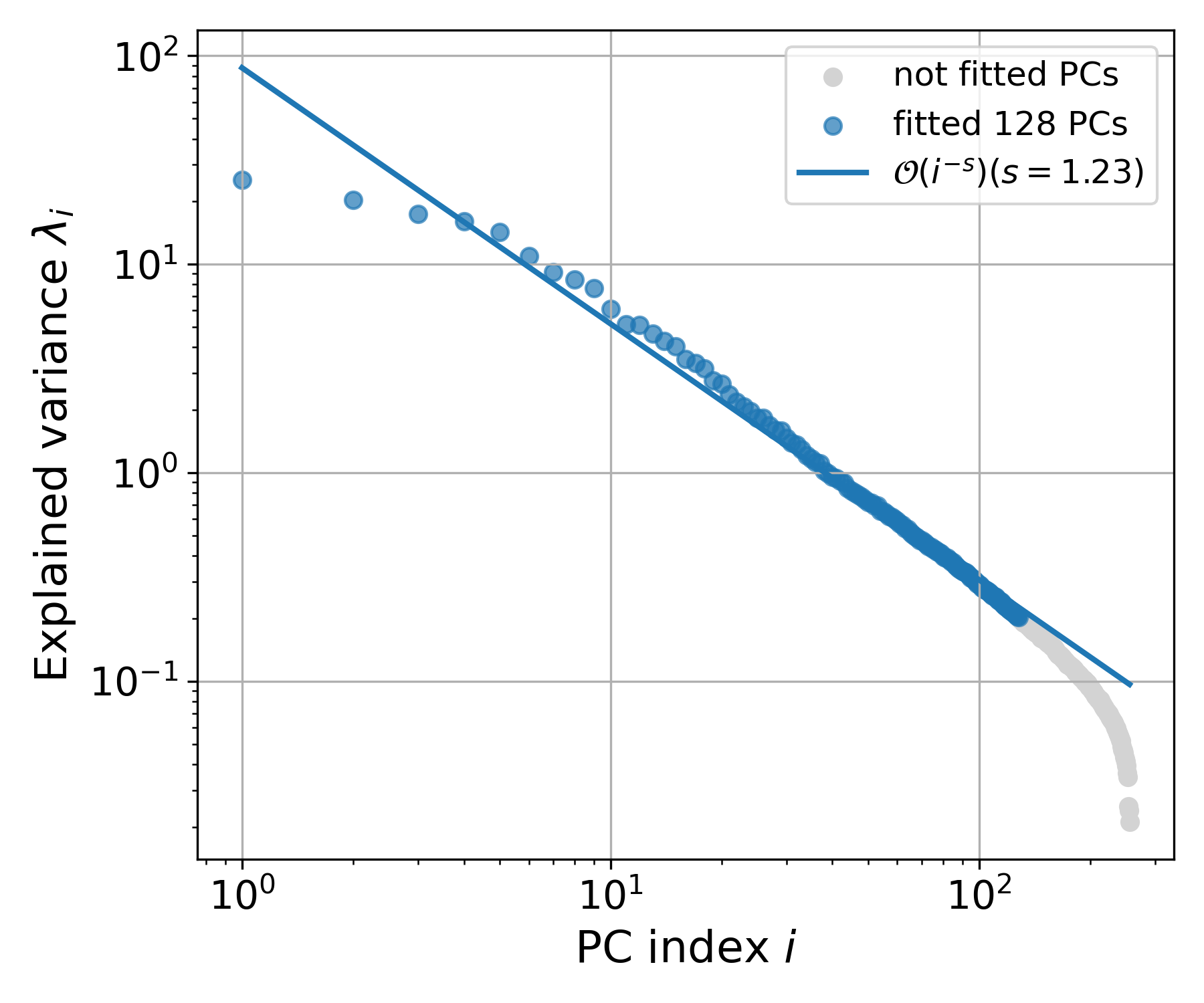} 
\caption{PCA featured matrix on the whole MNIST training dataset: \small\textit{}}
\label{fig:mlpmmh_pca_decay}
\end{figure}

We then numerically justify the PCA hierarchy in level $l$ with Gaussian prior decays in bias and variance as required in Assumption \ref{ass:rate_hyper} for the convergence of MLPMMH, see Figure \ref{fig:mlpmmh_regularity_prior}. It shows the decay of the prior predictive increments in bias and variance. For a NN-based experiment, the cost rate $\gamma=1$. From the Figure \ref{fig:mlpmmh_regularity_prior}, we can observe the variance decay rate $\beta=1.408$, which is close to the theoretical value $2s-1$. Also note that $\beta>1=\gamma$, hence satisfying all the conditions for the convergence of MLPMMH. 

\begin{figure}[h!]
\begin{subfigure}{.49\textwidth}
  \centering
  \includegraphics[width=.9\linewidth]{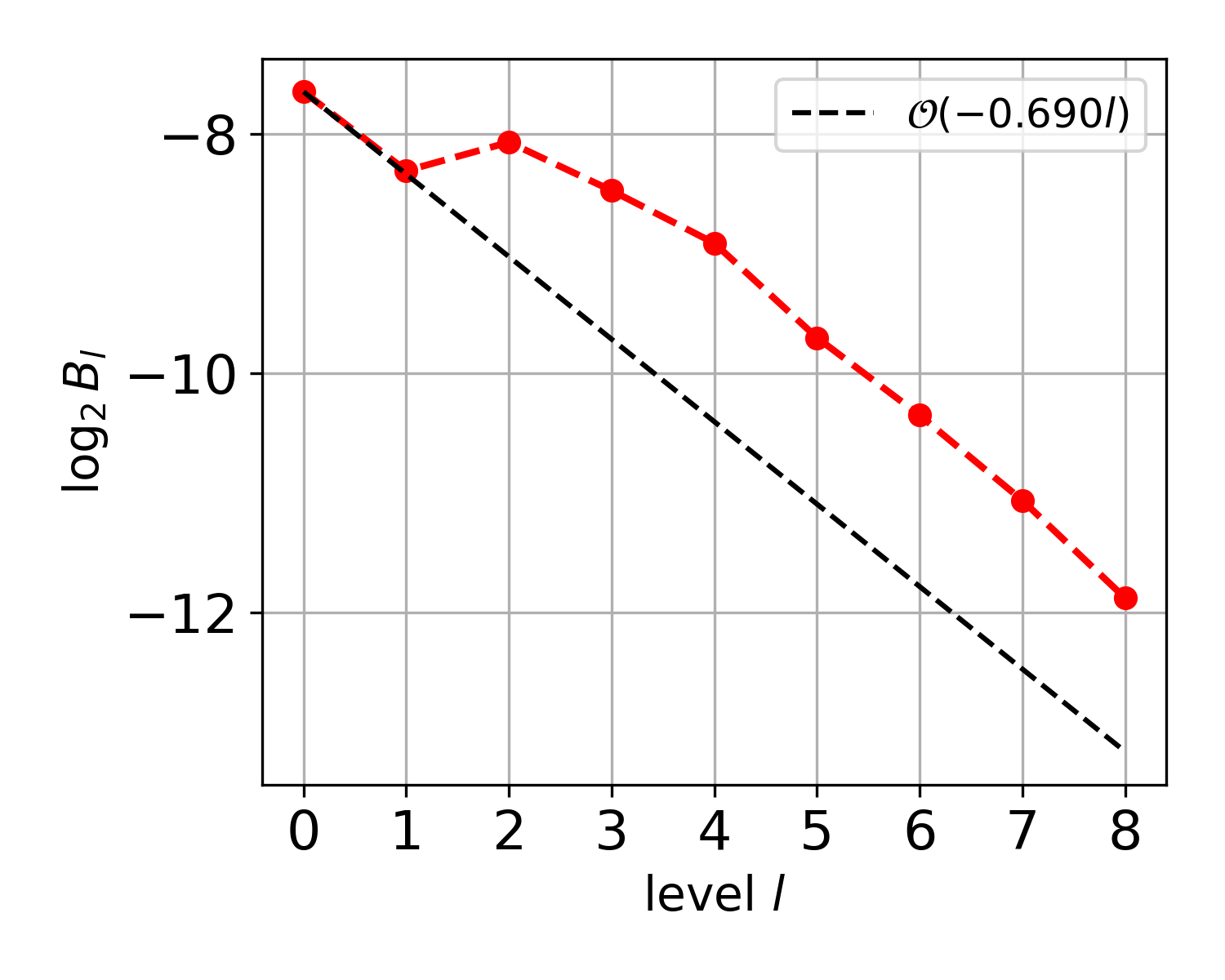}  
  \caption{}
\end{subfigure}
\begin{subfigure}{.49\textwidth}
  \centering
  \includegraphics[width=.9\linewidth]{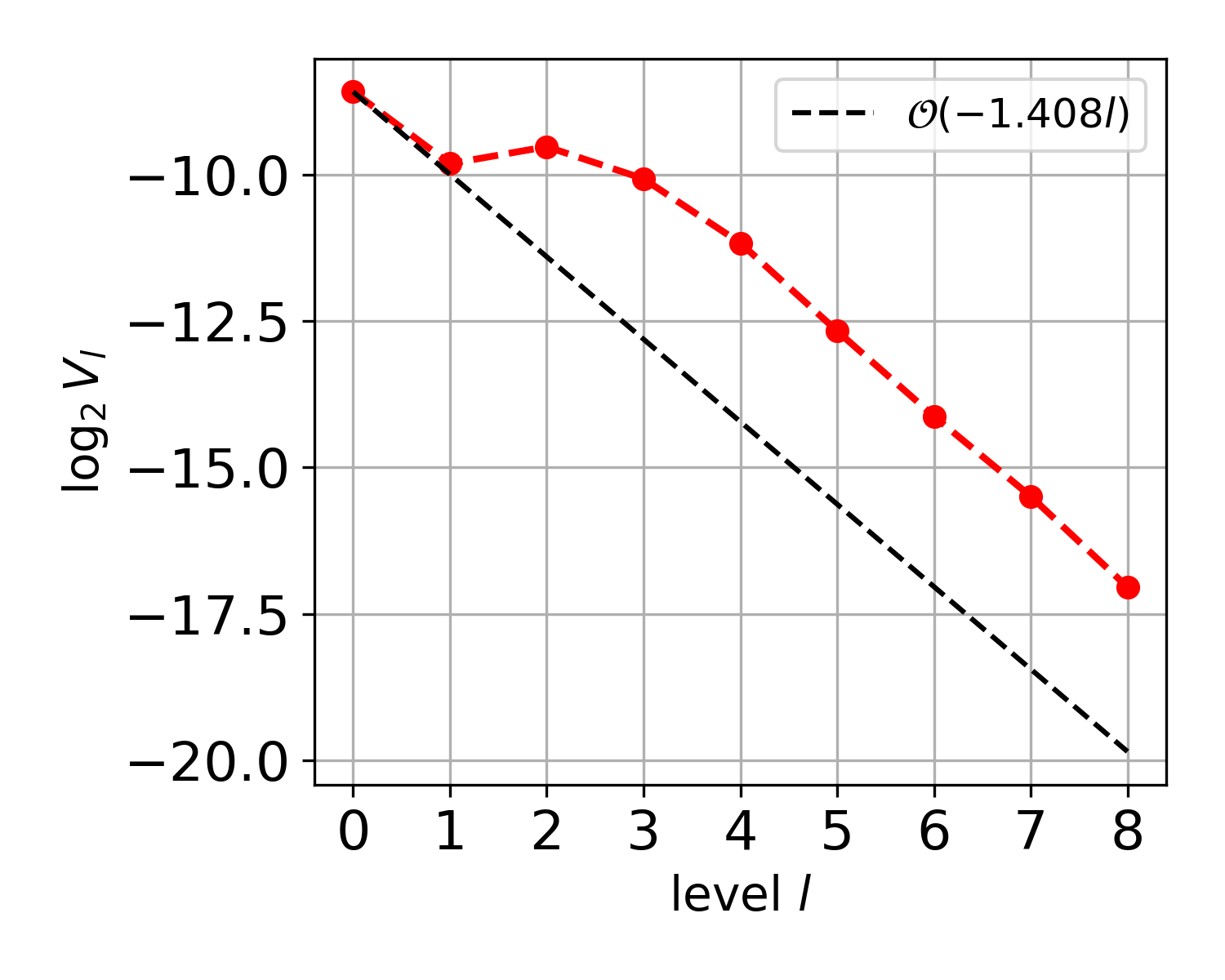}  
  \caption{}
\end{subfigure}
\caption{Regularity test for MLPMMH on MNIST dataset w.r.t prior: \small\textit{Empirical values for strong and weak convergence parameters under prior distribution $\pi_0$. All plots use $100$ realizations with $M_{l}=100$ samples in each index as the maximum level $L_{\max}=8$ and starting level as $L_{\text{min}}=0$ at each experiment. (a) Logarithm plot with base $2$ for mean of difference and reference line with slope $-0.690$; (b) Logarithm plot with base $2$ for variance of difference and reference line with slope $-1.408$.}}
\label{fig:mlpmmh_regularity_prior}
\end{figure}

\section{Practical settings in MLSMC methods}

\subsection{Multilevel sequential Monte Carlo estimators}
\label{sec:mlsmc_practical_set}

First, the empirical value of $L_{\text{min}}$ is determined by the regularity results, where the $B$ and $V$ needs to decay followed the expected rate in Assumption \ref{ass:rate_hyper} with fixed $\mu$ for each level $L_{\text{min}}<l\leq L_{\text{max}}$. Suppose we know the parameter values and the constants, we can compute the empirical values for $L_{\text{max}}$ and $M_l$ if we want the MSE to be smaller than $\varepsilon^2$, i.e. the variance and the bias terms are all smaller than $\frac{\varepsilon^2}{2}$, using the complexity theorem of MLSMC (Theorem \ref{thm:main_mlsmc}):
\begin{itemize}
    \item The value of $L_{\text{max}}$: \\
    If $|\bbE[\varphi_l-\varphi]| \leq c_1h_l^{\alpha}$ and $h_l=h^l$, we have the bias term as 
    \begin{equation}
        \bbE[\varphi_L - \varphi] = \sum_{l=L+1}^{\infty} \bbE[\varphi_l - \varphi_{l-1}] = \frac{\bbE[\varphi_L - \varphi_{L-1}]}{h^{-\alpha}-1}.
    \end{equation}
    In order to ensure $\bbE[\varphi_L - \varphi]^2 \leq \frac{\varepsilon^2}{2}$, we need
    \begin{equation}
    \begin{aligned}
        & \frac{\bbE[\varphi_L - \varphi_{L-1}]^2}{(h^{-\alpha}-1)^2} \leq \frac{c_1^2 h^{2 \alpha}_{L}}{(h^{-\alpha}-1)^2} \leq \frac{\varepsilon^2}{2} \\
        \Rightarrow & L \geq \frac{1}{\alpha}  \log_{h^{-1}}\Big( \frac{\sqrt{2}c_1}{\varepsilon (h^{-\alpha}-1)} \Big). \label{l_max}
    \end{aligned}
    \end{equation}
    We then take $L_{\text{max}} = \Big\lceil \frac{1}{\alpha}  \log_{h^{-1}}\Big( \frac{\sqrt{2}c_1}{\varepsilon (h^{-\alpha}-1)} \Big) \Big\rceil$.
    \item The value of $M_l$:\\
    From the observed $L_{\text{min}}$ and estimated $L_{\text{max}}$ above, we can estimate the number of samples at each level, which is given by
    \begin{equation}
        M_{l} = \Big\lceil 2\varepsilon^{-2} \sqrt{V_l/C_l} \Big( \sum_{l=L_{\text{min}}}^{L_{\text{max}}} \sqrt{V_l C_l} \Big) \Big\rceil, \label{m_l}
    \end{equation}
    to ensure the variance term is smaller than $\frac{\varepsilon^2}{2}$, where $\lceil \cdot \rceil $ denotes the smallest integer greater than the number inside.
\end{itemize}

When the Assumption \ref{ass:rate_hyper} is not satisfied, it leads to a sub-canonical case, stated in Remark \ref{rem:sub_canonical_mlsmc}. Moreover, since each level is independent, we can parallelize the algorithm for all the levels to accelerate the sampler.

\begin{rem}[Sub-canonical regime in MLSMC]\label{rem:sub_canonical_mlsmc}
In the sub-canonical case, where the variance decays more slowly than the cost grows, i.e. $\beta \leq \gamma$, the decay rate violates in the Assumption \ref{ass:rate_hyper}. The MLSMC estimators incur an additional exponent in the $\varepsilon$–COST complexity. Specifically, for level $l$, if the bias decays like $\cO (2^{-\alpha l})$ and the cost grows like $O(2^{\gamma l})$, then to achieve $\mathrm{MSE}\leq \varepsilon^2$, one obtains
$$
\mathrm{COST} = \cO \bigl(\varepsilon^{-2 - (\gamma-\beta)/\alpha}\bigr),
$$
rather than the canonical $O(\varepsilon^{-2})$ when $\beta>\gamma$.
\end{rem}

\subsection{Randomized multilevel sequential Monte Carlo estimators}
\label{sec:rmlsmc_practical_set}

Again, the empirical value of $L_{\text{min}}$ is determined by the regularity results, where the $B$ and $V$ needs to decay followed the expected rate in Assumption \ref{ass:rate_hyper} with fixed $\mu$ for each level $L_{\text{min}}<l\leq L_{\text{max}}$. Proposition 3 in \cite{law2022multi} states the requirement of
\begin{equation}\label{eqn:thm_var}
    \sum_{l \in \bbP} \frac1{p_l} 2^{-l \beta} \leq C
\end{equation}
to let $\bbE\big[ \big(\sum_{l \in \bbP} \frac{N_l}{p_l N} F^{N_l}_{l}(\Delta \zeta) - f(\zeta)\big)^2 \big]$ be of $\cO(N^{-1})$, so that the convergence rate in Theorem \ref{thm:main_rmlsmc}. Additionally, we let the distribution follow from a simple constrained minimization of the expected cost for one sample
\begin{equation}\label{eqn:thm_cost}
    \sum_{l \in \bbP} p_l 2^{l \gamma} \leq C.
\end{equation}
Since $p_l>0$, one finds that
\begin{equation}
    p_l \propto 2^{-l (\gamma+\beta)/2}
\end{equation}
is one suitable choice for probabilities to satisfy the two conditions stated in (\ref{eqn:thm_var}) and (\ref{eqn:thm_cost}). Then define the probability distribution ${\bf p}$ on $\bbP$ with 
${\bf p}(l) =: p_{l} > 0$.
Consider drawing $N/N_{\rm min}$ i.i.d. samples
$l_i \sim {\bf p}$,
and define the number of samples $N_{l}=N_{\rm min}\#i\{l_{i}=l\}$. Like in the MLSMC, since each level is independent, we can parallel the algorithm for all the levels to accelerate this sampler as well. It is worth noting that when the Assumption \ref{ass:rate_hyper} is not satisfied, it leads to a sub-canonical case like in the MLSMC, stated in Remark \ref{rem:sub_canonical_rmlsmc}.

\begin{rem}[Sub-canonical regime in rMLSMC \citep{chada2021unbiased}]\label{rem:sub_canonical_rmlsmc}
Similarly, in the randomized MLSMC estimator, even with randomization that eliminates all bias, the fundamental rate‐limiting factor remains the mismatch between variance decay and cost growth. Let the sampling probabilities be
$$
p_{l} \propto 2^{-\beta l},
$$
which is the asymptotically optimal “borderline” choice, so that the estimator achieves a sub‐canonical rate with a logarithmic penalty:
$$
\mathrm{COST} = \cO \bigl(\varepsilon^{-2}\,\lvert\log\varepsilon\rvert^{q}\bigr),
$$
for any $q>2$.
\end{rem}

\section{Proofs of Theoretical results}
\label{app:proofs}

The proofs of the various results in the paper, along with restatements of the results and verification of the assumptions, are presented here.

\subsection{Proof relating to Theorem \ref{thm:mlpmmh}}
\label{app:main_thm_proof}

{\em Let us assume Assumptions \ref{ass:G1_hyper}, \ref{ass:K_hyper} and \ref{ass:rate_hyper} with $\beta >\gamma$. Then, for suitable $\varphi: \Theta \times \sS \rightarrow \bbR$, there exists a $C < +\infty$ such that for any $\varepsilon >0$ one can choose $(L=\{L_{\text{min}},L_{\text{max}}\},\{M_l\}_{l=L_{\text{min}}}^{L_{\text{max}}})$ such that
    \begin{equation*}
    \bbE[(\hat{\varphi}_{L}^{\mathrm{MLPMMH}} - \pi(\varphi))^2] \leq C \varepsilon^2,
    \end{equation*}
    and expected COST$(\hat{\varphi}_{L}^{\mathrm{MLPMMH}}) \leq C \varepsilon^{-2}$, i.e. the canonical rate. The estimator is defined in equation \eqref{eqn:mlpmmh}.
}

\begin{proof}
We first define 
$$\Delta \Pi_l(\varphi_l) := \Pi_l(\Delta \varphi_l) = \Pi_{l}(\varphi_l - \varphi_{l-1}),$$
where $\Pi_l$
is the exact coupling. Note that $$\Delta \Pi_l(\varphi_l) = \pi_l(\varphi_l) - \pi_{l-1}(\varphi_{l-1}).$$ 
Noting the independence between differences, we derive
    \begin{equation}
    \begin{aligned}
        \bbE[(\hat{\varphi}_{L}^{\mathrm{MLPMMH}} - \pi(\varphi))^2]&= \Big( \pi_{L}(\varphi_L) - \pi(\varphi)\Big)^2 + \bbE \Bigg[ \Big( \hat{\varphi}_{L}^{\mathrm{MLPMMH}} - \pi_{L}(\varphi_L) \Big)^2 \Bigg] \\
        &= \Big( \pi_{L}(\varphi_L) - \pi(\varphi)\Big)^2 + \bbE\Bigg[\Big( \sum_{l=0}^{L}  \Delta \Pi_{l}^{N,M_l}(\varphi_l) - \Delta \Pi_l(\varphi_l) \Big)^2 \Bigg]\\
        &= \Big( \pi_{L}(\varphi_L) - \pi(\varphi)\Big)^2 + \sum_{l=0}^{L} \bbE\Bigg[ \Big( \Delta \Pi_{l}^{N,M_l}(\varphi_l) - \Delta \Pi_l(\varphi_l) \Big)^2 \Bigg] \\
        &+ \sum_{q\neq l=0}^{L} \bbE\Bigg[ \Big( \Delta \Pi_{q}^{N,M_q}(\varphi_q) - \Delta \Pi_q(\varphi_q) \Big) \Big( \Delta \Pi_{l}^{N,M_l}(\varphi_l) - \Delta \Pi_l(\varphi_l) \Big) \Bigg].
    \label{eqn:mse1}
    \end{aligned}
    \end{equation}
In order for us to process, we need to consider three terms, i.e. the bias term, the variance term and the covariance term. We let $\beta = (2s-1)$. Let us consider first the bias term, where a direct application of Assumption \ref{ass:rate_hyper} gives
\begin{align}
    \Big(\pi_{L}(\varphi_L) - \pi(\varphi) \Big)^2 & = \Bigg( \int_{\Theta \times \sS} (\varphi_L(\theta_L,\mu) - \varphi(\theta,\mu) )\pi(d\theta,d\mu) \Bigg)^2 \nonumber \\
    & \leq C 2^{-2\alpha L} \leq C 2^{-\beta L}.
\end{align}

Next, we consider the Variance term, which by Theorem \ref{thm:var_term} we have 
$$\bbE\Bigg[ \Big( \Delta \Pi_{l}^{N,M_l}(\varphi_l) - \Delta \Pi_l(\varphi_l) \Big)^2 \Bigg] \leq \frac{C 2^{-\beta l}}{M_{l}}.$$

By the Cauchy-Schwarz inequality, we have for the covariance term 
\begin{equation*}
\begin{aligned}
   & \bbE\Bigg[ \Big( \Delta \Pi_{q}^{N,M_q}(\varphi_q) - \Delta \Pi_q(\varphi_q) \Big) \Big( \Delta \Pi_{l}^{N,M_l}(\varphi_l) - \Delta \Pi_l(\varphi_l) \Big) \Bigg] \\
   & \leq \sqrt{ \bbE\Bigg[ \Big( \Delta \Pi_{q}^{N,M_q}(\varphi_q) - \Delta \Pi_q(\varphi_q) \Big)^2 \Bigg] } \sqrt{ \bbE\Bigg[ \Big( \Delta \Pi_{l}^{N,M_l}(\varphi_l) - \Delta \Pi_l(\varphi_l) \Big)^2 \Bigg] }\\
   & \leq \sqrt{\frac{C^{\prime} 2^{-\beta q}}{M_{q}}} \sqrt{\frac{C^{\prime \prime} 2^{-\beta l}}{M_{l}}} \leq C \frac{2^{-\beta (q+l)/2}}{\sqrt{M_q M_l}}.
\end{aligned}
\end{equation*}

Then, combining the results above, we have
\begin{equation}
\begin{aligned}\label{eqn:mse2}
    \eqref{eqn:mse1} \leq C^{\prime} 2^{-\beta L} + C^{\prime \prime} \sum_{l=0}^{L} \frac{2^{-\beta l}}{M_l} + C^{\prime \prime \prime} \sum_{q \neq l=0}^{L} \frac{2^{-\beta (q+l)/2}}{\sqrt{M_q M_l}}.
\end{aligned}
\end{equation}

In order to ensure the MSE is $\cO(\varepsilon^2)$, then all three terms in \eqref{eqn:mse2} need to be $\cO(\varepsilon^2)$. Following from the first term in \eqref{eqn:mse2}, it will be satisfied to let $L \propto 2|\log(\varepsilon)|/\beta$. For the the second term in \eqref{eqn:mse2}, suppose we minimize total COST $=\sum_{l=0}^{L} 2^{-\gamma l} M_l$ subject to $\sum_{l=0}^{L} \frac{2^{-\beta l}}{M_l} = \cO(\varepsilon^2)$, we  solve the number of samples at each level $M_0,M_1,...,M_L$. This gives
\begin{equation}
    M_l \propto \varepsilon^{-2} K_L 2^{-(\beta+\gamma)l/2},
\end{equation}
where $K_L = \sum_{l=0}^{L} 2^{-(\beta-\gamma)l/2}$. This gives the total cost of $\cO(\varepsilon^{-2})$. Finally, we need to check that the third term is of $\cO(\varepsilon^2)$:
\begin{equation}
    \sum_{q \neq l=0}^{L} \frac{2^{-\beta (q+l)/2}}{\sqrt{M_q M_l}} = \sum_{q \neq l=0}^{L} 2^{-(\beta-\gamma) (q+l)/4} K_L \varepsilon^2.
\end{equation}
Since we assume $\beta > \gamma$, the quantity $\sum_{q \neq l=0}^{L} 2^{-(\beta-\gamma) (q+l)/4} K_L< \infty$ is finite, and finally 
\begin{equation*}
\sum_{q \neq l=0}^{L} \frac{2^{-\beta (q+l)/2}}{\sqrt{M_q M_l}} = O(\varepsilon^2).
\end{equation*}
Therefore, this completes the proof.

\end{proof}

\subsection{Proof relating to Theorem \ref{thm:var_term}}
\label{app:proof_var_term}

{\em Assume Assumptions \ref{ass:G1_hyper}, \ref{ass:rate_hyper}. For any $\varphi: \Theta \times \sS \rightarrow \bbR^{d}$, there exists a $C < +\infty$ such that
    \begin{equation}\label{eqn:var_term}
        \bbE\Bigg[ \Big( \Pi_{l}^{N,M_l}(\Delta \varphi_l) - \Pi_l(\Delta  \varphi_l) \Big)^2 \Bigg] \leq \frac{C 2^{-\beta l}}{M_l}.
    \end{equation}
}

\begin{proof}
Writing out \eqref{eqn:var_term}, we have
    \begin{equation*}
    \begin{aligned}
        \bbE \Bigg[ & \Bigg( \frac{\frac{1}{M_l}\sum_{k=1}^{M_l} \varphi_l(\theta_{k,l},\mu_{k,l}) \frac{f_{l}(\theta_{k,l}|\mu_{k,l})}{F_{l}(\theta_{k,l},\theta_{k,l-1}|\mu_{k,l})}}{\frac{1}{M_l}\sum_{k=1}^{M_l}\frac{f_{l}(\theta_{k,l}|\mu_{k,l})}{F_{l}(\theta_{k,l},\theta_{k,l-1}|\mu_{k,l})}} - \frac{\frac{1}{M_l}\sum_{k=1}^{M_l} \varphi_{l-1}(\theta_{k,l-1},\mu_{k,l}) \frac{f_{l-1}(\theta_{k,l-1}|\mu_{k,l})}{F_{l}(\theta_{k,l},\theta_{k,l-1}|\mu_{k,l})}}{\frac{1}{M_l}\sum_{k=1}^{M_l}\frac{f_{l-1}(\theta_{k,l-1}|\mu_{k,l})}{F_{l}(\theta_{k,l},\theta_{k,l-1}|\mu_{k,l})}}\\
        &- \bigg( \frac{f_l(\varphi_l(\theta_{l},\mu))}{Z_l} - \frac{f_{l-1}(\varphi_{l-1}(\theta_{l-1},\mu))}{Z_{l-1}}  \bigg) \Bigg)^2 \Bigg].
    \end{aligned}
\end{equation*}
Now we simplify the notation, let
\begin{align*}
 a&=\frac{1}{M_l}\sum_{k=1}^{M_l} \varphi_l(\theta_{k,l},\mu_{k,l}) \frac{f_{l}(\theta_{k,l}|\mu_{k,l})}{F_{l}(\theta_{k,l},\theta_{k,l-1}|\mu_{k,l})}, \\
A &= \frac{1}{M_l}\sum_{k=1}^{M_l}\frac{f_{l}(\theta_{k,l}|\mu_{k,l})}{F_{l}(\theta_{k,l},\theta_{k,l-1}|\mu_{k,l})}, \\
b &= \frac{1}{M_l}\sum_{k=1}^{M_l} \varphi_{l-1}(\theta_{k,l-1},\mu_{k,l}) \frac{f_{l-1}(\theta_{k,l-1}|\mu_{k,l})}{F_{l}(\theta_{k,l},\theta_{k,l-1}|\mu_{k,l})}, \\
B &= \frac{1}{M_l}\sum_{k=1}^{M_l}\frac{f_{l-1}(\theta_{k,l-1}|\mu_{k,l})}{F_{l}(\theta_{k,l},\theta_{k,l-1}|\mu_{k,l})}, \\
c &= f_l(\varphi_l(\theta_{l},\mu)), \quad
C = Z_{l}, \quad
d = f_{l-1}(\varphi_{l-1}(\theta_{l-1},\mu)), \quad
D = Z_{l-1}.
\end{align*}


By Lemma \ref{lem:separate_terms} and Jensen's inequality, we have
\begin{align}
    \Bigg( \frac{a}{A} - & \frac{b}{B} - \Bigg( \frac{c}{C} - \frac{d}{D} \Bigg) \Bigg)^2 = \Bigg( \frac{[a-b-(c-d)]}{A} - \frac{b[A-B-(C-D)]}{AB} + \frac{1}{AC}[C-A][c-d] \nonumber \\
        &- \frac{1}{AB}(b-d)(C-D) + \frac{d}{CBD}(B-D)(C-D) + \frac{d}{ACB}(A-C)(C-D) \Bigg)^2 \nonumber \\
    & \leq 6 \Bigg[ \Bigg( \frac{[a-b-(c-d)]}{A} \Bigg)^2 + \Bigg( \frac{b[A-B-(C-D)]}{AB} \Bigg)^2 + \Bigg( \frac{1}{AC}[C-A][c-d] \Bigg)^2  \nonumber \\
        &+ \Bigg( \frac{1}{AB}(b-d)(C-D) \Bigg)^2 + \Bigg( \frac{d}{CBD}(B-D)(C-D) \Bigg)^2 + \Bigg( \frac{d}{ACB}(A-C)(C-D) \Bigg)^2 \Bigg]. \label{eqn:six_terms}
\end{align}

Now, we need to prove that all six terms in \eqref{eqn:six_terms} are $\cO(\frac{2^{-\beta l}}{M_l})$. By Proposition \ref{prop:bound_term}, the first term is verified directly and the second term is also verified by letting $\varphi == 1$.

By Assumption \ref{ass:rate_hyper} (V) and (B) with $\alpha \leq \beta/2$, we have 
$$
\bbE[(b-d)^2] \leq C \frac{2^{-\beta l}}{M_l}, \quad \bbE[(A-C)^2] \leq C \frac{2^{-\beta l}}{M_l},
$$
with $\varphi == 1$ (the same for $\bbE[(B-D)^2] \leq C \frac{2^{-\beta l}}{M_l}$). Since both
$$
\frac{f_{l}(\theta_{k,l}|\mu_{k,l})}{F_{l}(\theta_{k,l},\theta_{k,l-1}|\mu_{k,l})} \leq 1, \quad \frac{f_{l-1}(\theta_{k,l-1}|\mu_{k,l})}{F_{l}(\theta_{k,l},\theta_{k,l-1}|\mu_{k,l})} \leq 1,
$$
by the construction of coupled distribution $F_{l}(\theta_{k,l},\theta_{k,l-1}|\mu_{k,l})$, both
$$
A = \frac{1}{M_l}\sum_{k=1}^{M_l}\frac{f_{l}(\theta_{k,l}|\mu_{k,l})}{F_{l}(\theta_{k,l},\theta_{k,l-1}|\mu_{k,l})}, \quad B = \frac{1}{M_l}\sum_{k=1}^{M_l}\frac{f_{l-1}(\theta_{k,l-1}|\mu_{k,l})}{F_{l}(\theta_{k,l},\theta_{k,l-1}|\mu_{k,l})},
$$
are bounded. Noting that $c,C,d,D$ are not random, and they are all bounded following Assumption \ref{ass:G1_hyper}. Hence, the last four terms are bounded by $\cO(\frac{2^{-\beta l}}{M_l})$.
    
\end{proof}

\subsection{Proof relating to Lemma \ref{lem:separate_terms}}
\label{lem:proof_separate_terms}

{\em Let $a,b,c,d,A,B,C,D \in \bbR$ with $A,B,C,D$ non-zero then;
    \begin{equation}
    \begin{aligned}\label{eqn:lem_right}
        \frac{a}{A} - & \frac{b}{B} - \Bigg( \frac{c}{C} - \frac{d}{D} \Bigg) = \frac{[a-b-(c-d)]}{A} - \frac{b[A-B-(C-D)]}{AB} + \frac{1}{AC}[C-A][c-d] \\
        &- \frac{1}{AB}(b-d)(C-D) + \frac{d}{CBD}(B-D)(C-D) + \frac{d}{ACB}(A-C)(C-D).
    \end{aligned}
    \end{equation}
}

\begin{proof}
    Expand the terms in the right-hand side of \eqref{eqn:lem_right},
    \begin{align*}
        & \Bigg( \frac{a}{A} - \frac{b}{A} - \frac{c}{A}+\frac{d}{A} \Bigg) - \Bigg( \frac{bA}{AB} - \frac{bB}{AB} - \frac{bC}{AB} + \frac{bD}{AB} \Bigg) + \Bigg( \frac{Cc}{AC} - \frac{Cd}{AC} - \frac{Ac}{AC} + \frac{Ad}{AC} \Bigg) \\
        & - \Bigg( \frac{bC}{AB} - \frac{bD}{AB} - \frac{dC}{AB} + \frac{dD}{AB} +\Bigg) + \Bigg( \frac{dBC}{CBD} - \frac{dBD}{CBD} - \frac{dDC}{CBD} + \frac{dD^2}{CBD} \Bigg) \\
        & + \Bigg( 
        \frac{dAC}{ACB} - \frac{dAD}{ACB} - \frac{dC^2}{ACB} + \frac{dCD}{ACB} \Bigg)\\
        & = \Bigg( \frac{a}{A} - \frac{b}{A} - \frac{c}{A}+\frac{d}{A} \Bigg) - \Bigg( \frac{b}{B} - \frac{b}{A} - \frac{bC}{AB} + \frac{bD}{AB} \Bigg) + \Bigg( \frac{c}{A} - \frac{d}{A} - \frac{c}{C} + \frac{d}{C} \Bigg) \\
        & - \Bigg( \frac{bC}{AB} - \frac{bD}{AB} - \frac{dC}{AB} + \frac{dD}{AB} \Bigg) + \Bigg( \frac{d}{D} - \frac{d}{C} - \frac{d}{B} + \frac{dD}{CB} \Bigg) + \Bigg( 
        \frac{d}{B} - \frac{dD}{CB} - \frac{dC}{AB} + \frac{dD}{AB} \Bigg)\\
        & = \frac{a}{A} - \frac{b}{B} - \frac{c}{C} + \frac{d}{D} = \frac{a}{A} - \frac{b}{B} - \Bigg( \frac{c}{C} - \frac{d}{D} \Bigg).
    \end{align*}
\end{proof}

\subsection{Proof relating to Proposition \ref{prop:bound_term}}
\label{app:proof_bound_term}

{\em Assume Assumptions \ref{ass:G1_hyper}, \ref{ass:K_hyper}, \ref{ass:rate_hyper}. Suppose that $\{ W^i \}_i$ is a Markov chain with the Markov kernel $K$, and $W^1 \sim \pi$. Then there exists a $\beta > 0$ such that for any $\varphi \in \cB_b(\Theta \times \sS) \cap \text{Lip}(\Theta \times \sS)$ there exists a $C>0$ such that
    \begin{equation*}
        \bbE \Bigg[ \Big( \frac{1}{M_l} \sum_{k=1}^{M_l} \tilde{\varphi}_{l}(V^{k}) - \Pi_{l}(\tilde{\varphi}_{l}) \Big)^2 \Bigg] \leq \frac{C 2^{-\beta l}}{M_l},
    \end{equation*}
    where $V^k = (\theta_{k,l}, \mu)$ is the relevant variables of $W^k$.
}

\begin{proof}
    This proof is similar as Proposition A.1 in \cite{jasra2018bayesian}, where we repeat a similar procedure.
    \begin{align*}
        \bbE \bigg[ \bigg( \frac{1}{M_l}\sum_{k=1}^{M_l} \tilde{\varphi}_{l}(V^k) - \Pi_{l}(\tilde{\varphi}_{l}]) \bigg)^2 \bigg] = \bbE \bigg[ \bigg( \frac{1}{M_l}\sum_{k=1}^{M_l}f(W^k) \bigg)^2\bigg],
    \end{align*}
    where $f(w) = \tilde{\varphi}_{l} \circ \psi(w) - \pi_l(\tilde{\varphi}_{l} \circ \psi) = \tilde{\varphi}_{l}(v) - \Pi_{l}(\tilde{\varphi}_{l})$. Then, we have
    \begin{align}
        \bbE \bigg[ \bigg( \frac{1}{M_l}\sum_{k=1}^{M_l}f(W^k) \bigg)^2\bigg] & = \frac{1}{M_{l}^2}\sum_{p,q=1}^{M_l} \langle f, K^{|p-q|} f \rangle \nonumber \\
        & = \frac{1}{M_l} \norm{f}_{2}^{2} + \frac{2}{M_l} \sum_{p=1}^{M_l-1} (M_l-p) \langle f, K^{p}f \rangle \nonumber \\
        & \leq \frac{1}{M_l} \norm{f}_{2}^{2} + \frac{2}{M_l^2} \sum_{p=1}^{M_l-1} (M_l-p) \norm{K^p}_{2} \norm{f}_{2}^{2} \nonumber \\
        & \leq \frac{1}{M_l} \norm{f}_{2}^{2} + \frac{2}{M_l} \sum_{p=1}^{\infty} \norm{K^p}_2 \norm{f}_{2}^{2}. \label{eqn:f_sum}
    \end{align}
    By Assumption \ref{ass:K_hyper}, we have 
    \begin{align}
        \eqref{eqn:f_sum} \leq \frac{1}{M_l} \norm{f}_{2}^{2} \big( 1+\frac{1-\xi}{\xi} \big) \leq \frac{2\norm{f}_{2}^{2}}{M_l \xi}.
    \end{align}
    Since $\norm{f}_{2}^{2} = \bbE[(\tilde{\varphi}_{l}(V^k) - \Pi(\tilde{\varphi}_{l}))^2]$ and Assumption \ref{ass:rate_hyper} (V), we have
    \begin{align*}
        \norm{f}_{2}^{2} = \bbE[(\tilde{\varphi}_{l}(V^k) - \Pi(\tilde{\varphi}_{l}))^2] \leq V_l \leq C 2^{-\beta l}.
    \end{align*}
\end{proof}

\subsection{Verification of Assumptions}
\label{app:verify_ass}

We now want to verify Assumption \ref{ass:K_hyper} and Assumption \ref{ass:rate_hyper}. Let us begin with Assumption \ref{ass:K_hyper}.

{\em There exists a $\xi \in (0,1)$ and $\nu \in \cP(\sW)$ such that, for each $w \in \sW$, $\varphi \in \cB_{b}(\sW) \cap \text{Lip}(\sW)$,
    \begin{equation*}
        \int_{\sW} \varphi(w^{\prime}) K(w,dw^{\prime}) \geq \xi \int_{W} \varphi(v) \nu(dv).
    \end{equation*}
}

\begin{proof}
The proposal kernel $\psi$ of PMMH is 
\begin{equation}
    \psi(w,dw^{*}) = P(\theta_{1:M,l}|\mu^{*}_{q,l})Q(d\mu^{*}_{q,l}|\mu_{k,l})\frac{F_{q,l}(\theta_{q,l},\theta_{q,l-1}|\mu^{*}_{q,l})}{\sum_{i=1}^{M}F_{i,l}(\theta_{i,l},\theta_{i,l-1}|\mu^{*}_{q,l})},
\end{equation}
where $w = (\theta_{k,l},\mu_{k,l})$, $w^{*} = (\theta^{*}_{q,l},\mu^{*}_{q,l})$ and 
$$
P(\theta_{1:M,l}|\mu) = \prod_{k=1}^{M} \bar{\pi}(\theta_{k,l}) \frac{F_{k,l}(\theta_{k,l},\theta_{k,l-1}|\mu)}{\sum_{i=1}^{M}F_{i,l}(\theta_{i,l},\theta_{i,l-1}|\mu)}.
$$
The transition kernel $K$ is
\begin{align*}
    K(w,dw^{*}) = \psi(w,dw^{*}) \alpha(w,w^{*}) + \delta_{w}(dw^{*}) Q(w)
\end{align*}
where the acceptance probability $\alpha(w,w^{*})$ is 
\begin{align*}
    \alpha(w,w^{*}) = \min \Bigg\{ 1, \frac{\pi_0(\mu^{*}_{q,l})Q(\mu_{k,l}|\mu^{*}_{q,l}) p_{k,l}^{N}(\mu_{k,l}^{*})}{\pi_0(\mu^{*}_{q,l})Q(\mu^{*}_{q,l}|\mu_{k,l}) p_{k,l}^{N}(\mu_{k,l})} \Bigg\},
\end{align*}
and the rejection probability $Q(w)$ is 
\begin{align*}
    Q(w) = 1- \int_{w^{*}} \psi(w,dw^{*}) \alpha(w,w^{*}).
\end{align*}

Assume $\pi: \sS \rightarrow \bbR_{+}$ and $Q:\sS^2 \rightarrow \bbR_{+}$ are continuous and strictly positive and Assumption \ref{ass:G1_hyper} are satisfied, we have 
\begin{align*}
    C_1 := \inf_{w,w^{*}} \alpha(w,w^{*}) >0,
\end{align*}
\begin{align*}
    \inf_{w} \psi(w,dw^{*}) &\geq \big\{ \min_{\mu_{k,l},\mu^{*}_{q,l}} Q(\mu^{*}_{q,l}|\mu_{k,l}) \big\} 1_{\sS} (\mu^{*}_{q,l}) d\mu^{*}_{q,l} P(\theta_{1:M,l}|\mu^{*}_{q,l})\frac{F_{q,l}(\theta_{q,l},\theta_{q,l-1}|\mu^{*}_{q,l})}{\sum_{i=1}^{M}F_{i,l}(\theta_{i,l},\theta_{i,l-1}|\mu^{*}_{q,l})} \\&=: C_2 \psi(dw^{*}),
\end{align*}
where $C_2 = \text{Leb}(\sS) \min_{\mu_{k,l},\mu^{*}_{q,l}} Q(\mu^{*}_{q,l}|\mu_{k,l}]) >0$ with a probability measure $\psi(dw^{*})$. Thus, we have
\begin{align*}
    K(w,dw^{*}) \geq C_1 C_2 \psi(dw^{*}).
\end{align*}
Hence the Assumption \ref{ass:K_hyper} is satisfied.
\end{proof}

Next, we verify the bounds for the (B) and (V) in Assumption \ref{ass:rate_hyper}.

{\em For any $\varphi \in \cB_{b}(\Theta_{l} \times \Theta_{l-1} \times S) \cap \text{Lip} (\Theta_{l} \times \Theta_{l-1} \times \sS)$ there exist $c_1, c_2, c_3, \beta, \alpha, \gamma >0$ such that the following holds, under the joint posterior $\Pi_{l}$,
\begin{itemize}
\item[(B)] $\bbE\Big[\big|\varphi_{l}(\theta_l,\mu) \frac{f_{l}(\theta_l|\mu)}{F_{l}(\theta_l,\theta_{l-1}|\mu)} - f(\varphi(\theta,\mu)) \big| \Big] =: B_{l} \leq c_1  2^{-\alpha l}$;
\item[(V)] $ \bbE \Big[ \big|\varphi_{l} (\theta_l,\mu) \frac{f_{l}(\theta_l|\mu)}{F_{l}(\theta_l,\theta_{l-1}|\mu)} - \varphi_{l-1}(\theta_{l-1},\mu) \frac{f_{l-1}(\theta_{l-1}|\mu)}{F_{l}(\theta_l,\theta_{l-1}|\mu)} \big|^2 \Big] =: V_{l} \leq c_2 2^{-\beta l}$; 
\item[(C)] ${\rm COST}(\Pi_{l}(\varphi_l), \Pi_{l-1}(\varphi_{l-1})) =: C_{l} \leq c_3  2^{\gamma l}$.
\end{itemize}
}

\begin{proof}
\begin{align*}
    B_l &= \int_{\Theta \times \sS} \Bigg| \frac{f_{l}(\theta_l|\mu)}{F_{l}(\theta_l,\theta_{l-1}|\mu)} \varphi_{l}(\theta_l,\mu) - f(\theta|\mu)\varphi(\theta,\mu) \Bigg| \pi(d\theta,d\mu) \\
    & = \int_{\Theta \times \sS} \Bigg| \frac{f_{l}(\theta_l|\mu)}{F_{l}(\theta_l,\theta_{l-1}|\mu)} \varphi_{l}(\theta_l,\mu) - f(\theta|\mu)\varphi(\theta,\mu) \Bigg| \cL(\theta|\mu) \bar{\pi}(d\theta) \pi_0(d\mu) \times \int_{\Theta \times \sS}  \cL(\theta|\mu) \bar{\pi}(d\theta) \pi_0(d\mu)\\
    &= \int_{\Theta \times \sS} \Bigg| \frac{f_{l}(\theta_l|\mu)}{F_{l}(\theta_l,\theta_{l-1}|\mu)} \varphi_{l}(\theta_l,\mu) - \frac{f_{l}(\theta_l|\mu)}{F_{l}(\theta_l,\theta_{l-1}|\mu)} \varphi(\theta,\mu) + \frac{f_{l}(\theta_l|\mu)}{F_{l}(\theta_l,\theta_{l-1}|\mu)} \varphi(\theta,\mu) \\
    &- f(\theta|\mu)\varphi(\theta,\mu) \Bigg| \cL(\theta|\mu) \bar{\pi}(d\theta) \pi_0(d\mu) \times \int_{\Theta \times \sS}  \cL(\theta|\mu) \bar{\pi}(d\theta) \pi_0(d\mu).
    \end{align*}
we define $C^{\prime} = \max_{\theta \in \Theta}\cL(\theta|\mu)$ and $C^{\prime} = \int_{\Theta \times \sS}  \cL(\theta|\mu) \bar{\pi}(d\theta) \pi_0(d\mu)$ which results
 \begin{align*}
       B_l &\leq C^{\prime}C^{\prime \prime} \int_{\Theta \times \sS} \Bigg| \frac{f_{l}(\theta_l|\mu)}{F_{l}(\theta_l,\theta_{l-1}|\mu)} \varphi_{l}(\theta_l,\mu) - \frac{f_{l}(\theta_l|\mu)}{F_{l}(\theta_l,\theta_{l-1}|\mu)} \varphi(\theta,\mu) \Bigg| \bar{\pi}(d\theta) \pi_0(d\mu) \\
    &+ C^{\prime} \int_{\Theta \times \sS} \Bigg| \frac{f_{l}(\theta_l|\mu)}{F_{l}(\theta_l,\theta_{l-1}|\mu)} \varphi(\theta,\mu) - f(\theta|\mu)\varphi(\theta,\mu) \Bigg| \cL(\theta|\mu) \bar{\pi}(d\theta) \pi_0(d\mu) \\
    &\leq C^{\prime}C^{\prime \prime} \int_{\Theta \times \sS} \Bigg| \frac{f_{l}(\theta_l|\mu)}{F_{l}(\theta_l,\theta_{l-1}|\mu)} \Bigg| | \varphi_{l}(\theta_l,\mu) - \varphi(\theta,\mu)| \bar{\pi}(d\theta) \pi_0(d\mu) \\
    &+ C^{\prime} \int_{\Theta \times \sS} \Bigg| \frac{f_{l}(\theta_l|\mu)}{F_{l}(\theta_l,\theta_{l-1}|\mu)}  - f(\theta|\mu) \Bigg| |\varphi(\theta,\mu)| \cL(\theta|\mu) \bar{\pi}(d\theta) \pi_0(d\mu).
    \end{align*}
     Let us define $C^{\prime \prime \prime} = \max_{\theta \in \Theta, \mu \in \sS} |\varphi(\theta,\mu)|$ and $\tilde{C}$ to be a collection of various constants. Furthermore since 
     $$
     \Bigg| \frac{f_{l}(\theta_l|\mu)}{F_{l}(\theta_l,\theta_{l-1}|\mu)} \Bigg| \leq 1, \quad \pi\bigg(\frac{f_{l}(\theta_{l}|\mu)}{F_{l}(\theta_l,\theta_{l-1}|\mu)}\bigg) = \pi(f_{l}(\theta|\mu)),
     $$
     we have
    \begin{align*}
    B_l &\leq C^{\prime} C^{\prime \prime} \int_{\Theta \times \sS} | \varphi_{l}(\theta_l,\mu) - \varphi(\theta,\mu) | \bar{\pi}(d\theta) \pi_0(d\mu) \\&+ C^{\prime} C^{\prime \prime \prime} \int_{\Theta \times \sS} \Bigg| \frac{f_{l}(\theta_l|\mu)}{F_{l}(\theta_l,\theta_{l-1}|\mu)}  - f(\theta|\mu) \Bigg| \cL(\theta|\mu) \bar{\pi}(d\theta) \pi_0(d\mu)\\
    &\leq \tilde{C} \Bigg(\int_{\Theta \times \sS} | \varphi_{l}(\theta_l,\mu) - \varphi(\theta,\mu) | \bar{\pi}(d\theta) \pi_0(d\mu) +  | f_{l}(\theta_{l}|\mu) - f(\theta|\mu) | \cL(\theta|\mu) \bar{\pi}(d\theta) \pi_0(d\mu) \Bigg)\\
    &\leq \tilde{C} \Bigg( \int_{\Theta \times \sS} | \varphi_{l}(\theta_l,\mu) - \varphi(\theta,\mu) | \bar{\pi}(d\theta) \pi_0(d\mu) + | \varphi_{l}(\theta_{l}|\mu) - \varphi(\theta|\mu) | \bar{\pi}(d\theta) \pi_0(d\mu) \Bigg) \\
    &{\text{Applying Proposition \ref{prop:tnn} on it and letting $\alpha \leq (2s-1)/2$}}\\
    &\leq C 2^{-\alpha l},
\end{align*}
where we have assumed that $f_l(\theta|\mu)$ is Lipschitz w.r.t $\varphi_{l}(\theta,\mu)$ with Lipschitz constant  $G$,  $f_{\infty} = f$  and $C^{\prime} = \max_{\theta \in \Theta}\cL(\theta|\mu)$. Now, for the variance term, which we denote as $V_l$, we proceed as follows:

\begin{align*}
   V_l & = \int_{\Theta^2 \times \sS} \Bigg( \frac{f_{l}(\theta_l|\mu)}{F_{l}(\theta_l,\theta_{l-1}|\mu)} \varphi_{l}(\theta_l,\mu) - \frac{f_{l-1}(\theta_{l-1}|\mu)}{F_{l}(\theta_l,\theta_{l-1}|\mu)} \varphi_{l-1}(\theta_{l-1},\mu) \Bigg)^2 \times \\
   &  {\bf L}_{l}(\theta_{l},\theta_{l-1}|\mu)\bar{\Pi}_l(d(\theta_{l},\theta_{l-1})) \pi_0(d\mu) \Bigg( \int_{\Theta^2 \times \sS} {\bf L}_{l}(\theta_{l},\theta_{l-1}|\mu)\bar{\Pi}_l(d(\theta_{l},\theta_{l-1})) \pi_0(d\mu) \bigg)^{-1} \\
    & \leq C^{\prime} \int_{\Theta^2} 
    \Bigg( f_{l}(\theta_l|\mu) \varphi_{l}(\theta_l,\mu) - f_{l-1}(\theta_{l-1}|\mu) \varphi_{l-1}(\theta_{l-1},\mu) \Bigg)^2 \bar{\Pi}_l(d(\theta_{l},\theta_{l-1})) \pi_0(d\mu)
    \end{align*}
    where we define 
    $$
    C^{\prime} = \inf_{\theta \in \Theta^{2}} \frac{{\bf L}_{l}(\theta_l,\theta_{l-1}|\mu)}{F_{l}(\theta_l,\theta_{l-1}|\mu)^2} \times \Bigg( \int_{\Theta^2 \times \sS} {\bf L}_{l}(\theta_{l},\theta_{l-1}|\mu)\bar{\Pi}_l(d(\theta_{l},\theta_{l-1})) \pi_0(d\mu) \bigg)^{-1},
    $$
    \begin{align*}
    V_l& = C^{\prime} \int_{\Theta \times \sS} (f_{l}(\theta|\mu) \varphi_l(\theta,\mu) - f_{l-1}(\theta|\mu) \varphi_{l-1}(\theta,\mu))^2 \bar{\pi}_l(d\theta) \pi_0(d\mu)\\
    & = C^{\prime} \int_{\Theta \times \sS} \Big(f_{l}(\theta|\mu) \varphi_l(\theta,\mu) \\&- f_{l}(\theta|\mu)\varphi_{l-1}(\theta,\mu) + f_{l}(\theta|\mu)\varphi_{l-1}(\theta,\mu) - f_{l-1}(\theta|\mu) \varphi_{l-1}(\theta,\mu)\Big)^2 \bar{\pi}_l(d\theta) \pi_0(d\mu) \\
    & \leq C^{\prime}  \int_{\Theta \times \sS} (f_{l}(\theta|\mu) \varphi_l(\theta,\mu) - f_{l}(\theta|\mu)\varphi_{l-1}(\theta,\mu))^2 \bar{\pi}_l(d\theta) \pi_0(d\mu) \\
    &+ C^{\prime} \int_{\Theta \times \sS} ( f_{l}(\theta|\mu)\varphi_{l-1}(\theta,\mu) - f_{l-1}(\theta|\mu) \varphi_{l-1}(\theta,\mu))^2 \bar{\pi}_l(d\theta) \pi_0(d\mu)\\
    & \leq C^{\prime} \int_{\Theta \times \sS} (f_{l}(\theta|\mu))^2(\varphi_l(\theta,\mu) - \varphi_{l-1}(\theta,\mu))^2 \bar{\pi}_l(d\theta) \pi_0(d\mu) \\
    & + C^{\prime} \int_{\Theta \times \sS} (\varphi_{l-1}(\theta,\mu))^2 (f_{l}(\theta|\mu) - f_{l-1}(\theta|\mu))^2 \bar{\pi}_l(d\theta) \pi_0(d\mu)
    \end{align*}
    Since $f_{l}(\theta|\mu) \leq F_{l}(\theta_l,\theta_{l-1}|\mu) \leq C^{\prime}$, and we define $C^{\prime\prime} = \max_{\theta \in \Theta, \mu \in \sS} \varphi_{l-1}(\theta,\mu)$
    \begin{align*}
    V_l& \leq \frac{1}{C^{\prime}} \int_{\Theta \times \sS} (\varphi_l(\theta,\mu) - \varphi_{l-1}(\theta,\mu))^2 \bar{\pi}_l(d\theta) \pi_0(d\mu) \\&+ C^{\prime}(C^{\prime\prime})^2 \int_{\Theta \times \sS} (f_{l}(\theta|\mu) - f_{l-1}(\theta|\mu))^2 \bar{\pi}_l(d\theta) \pi_0(d\mu)\\
    & \text{Assume that $f_l(\theta|\mu)$ is Lipschitz w.r.t $\varphi_l(\theta,\mu)$ with Lipschitz constant $G$}\\
    & \leq \tilde{C} \Bigg( \int_{\Theta \times \sS} (\varphi_l(\theta,\mu) - \varphi_{l-1}(\theta))^2 \bar{\pi}_l(d\theta) \pi_0(d\mu) (\varphi_{l}(\theta,\mu) - \varphi_{l-1}(\theta,\mu))^2 \bar{\pi}_l(d\theta) \pi_0(d\mu) \Bigg)\\
    & \text{Applying the Proposition~\ref{prop:tnn} 
    on it and letting $\beta \leq 2s-1$}\\
    & \leq C 2^{-\beta l}. \\
\end{align*}
\end{proof}

\section{Future research directions}

 For future work, we list several potential projects to consider below.
 \begin{itemize}
      \item The examples given in this work were to firstly verify the theory, and also the provide a justification. As we have shown this, it would be of interest to consider more challenging datasets.       intensive and high-dimensional examples. Such examples could include bigger datasets such as IMDB or the CIFAR datasets. If dimensionality is a problem, one could exploit PCA.
 \item A useful direction would be to consider the complexity in terms of the depth of a neural network. We have only done this, in terms of the width, motivated by the original paper \cite{sell2020dimension}. Doing so would provide a fuller understanding of its impact algorithmically.
      \item Related to the above point, a natural question is to ask how it would compare to generic state-of-the-art MCMC methods like HMC. This has been answered, in comparison with parallel-SMC, but not in the context of our multilevel trace-class neural network setup. To do so, one could run SMC samplers with \texttt{Hamiltorch} mutations, and thus provide a comparison. This goes beyond this work, and will be done in a future paper.
     \item A final direction would be consider the use of such multilevel-trace class priors in the context of non-fully Bayesian methods. An example of this would be unadjusted Langevin algorithms, which have proven to be successful, which include recent multilevel algorithms \cite{ububu,smsubu}. This would involve a comparison between SMC/MCMC methodologies, as tested in this work, and ULA-algorithms.
 \end{itemize}
\end{document}